\documentclass{article}

    \PassOptionsToPackage{numbers, compress}{natbib}

\usepackage[main, final]{neurips_2026}

\usepackage[utf8]{inputenc} 
\usepackage[T1]{fontenc}    
\usepackage{hyperref}       
\usepackage{url}            
\usepackage{booktabs}       
\usepackage{nicefrac}       
\usepackage{microtype}      
\usepackage{xcolor}         

\usepackage{amsmath, amssymb, amsthm}
\usepackage{amsfonts}
\usepackage{mathtools}
\usepackage{dsfont}
\usepackage{graphicx}
\usepackage{enumitem}
\usepackage{nicematrix}
\usepackage{caption, subcaption}
\usepackage{algorithm, algpseudocode}
\usepackage{cleveref}
\usepackage{thmtools}
\usepackage{tabularx}

\usepackage{tikz}
\usetikzlibrary{positioning,chains,fit,shapes,calc,arrows.meta}
\tikzset{
  envstate/.style={circle,draw,thin,inner sep=1.4pt,
    minimum size=5mm,font=\scriptsize},
  envjunction/.style={circle,fill,inner sep=1.15pt},
  envleaf/.style={circle,draw,thin,inner sep=0pt,minimum size=2.2mm},
  envedge/.style={-Latex,thin},
  envpath/.style={-Latex,thin,densely dashed},
  envtext/.style={font=\scriptsize,align=center}
}

\usepackage[normalem]{ulem}

\usepackage[toc,page,header]{appendix}
\usepackage{minitoc}

\usepackage{thm-restate}

\usepackage{wrapfig}

\def\VAR{\mathrm{Var}}
\def\Regret{\mathrm{Reg}}

\theoremstyle{plain}

\newtheorem{lemma}{Lemma}
\newtheorem{claim}{Claim}[section]
\newtheorem{corollary}[lemma]{Corollary}
\theoremstyle{definition}

\newtheorem{assumption}{Assumption}

\theoremstyle{plain}
\newtheorem{remark}{Remark}
\newtheorem{example}{Example}

\theoremstyle{plain}
\newtheorem{theorem-rst}[theorem]{Theorem}
\newtheorem{lemma-rst}[lemma]{Lemma}
\newtheorem{proposition-rst}[theorem]{Proposition}
\newtheorem{assumption-rst}[assumption]{Assumption}
\newtheorem{claim-rst}[claim]{Claim}
\newtheorem{corollary-rst}[lemma]{Corollary}

\newenvironment{proofsketch}{%
  \proof}{\endproof}
  
\usepackage[textsize=tiny]{todonotes}

\DeclarePairedDelimiter\br{(}{)}
\DeclarePairedDelimiter\brs{[}{]}
\DeclarePairedDelimiter\brc{\{}{\}}

\DeclarePairedDelimiter\abs{\lvert}{\rvert}
\DeclarePairedDelimiter\norm{\lVert}{\rVert}

\DeclarePairedDelimiter\ceil{\lceil}{\rceil}

\DeclareMathOperator*{\argmax}{arg\,max}

\DeclareMathOperator*{\Ber}{Ber}

\newcommand{\ns}{\;\!\!}

\newcommand{\E}{\mathbb{E}}
\newcommand{\R}{\mathbb{R}}
\newcommand{\G}{\mathbb{G}}
\newcommand{\N}{\mathbb{N}}
\newcommand{\D}{\mathcal{D}}
\newcommand{\F}{\mathcal{F}}
\newcommand{\Kcal}{\mathcal{K}}
\newcommand{\X}{\mathcal{X}}
\newcommand{\Mcal}{\mathcal{M}}

\newcommand{\Acal}{\mathcal{A}}
\newcommand{\Ocal}{\mathcal{O}}
\newcommand{\Pcal}{\mathcal{P}}
\newcommand{\Scal}{\mathcal{S}}
\newcommand{\Rcal}{\mathcal{R}}

\newcommand{\Ical}{\mathcal{I}}

\newcommand{\Olog}{\tilde{\Ocal}}

\newcommand{\bR}{\boldsymbol{R}}
\newcommand{\bs}{\boldsymbol{s}}

\newcommand{\Brm}{\mathrm{B}}
\newcommand{\MB}{\Mcal^{\mathrm{B}}}

\newcommand{\Rla}{\mathfrak{R}}
\newcommand{\sla}{\mathfrak{s}}
\newcommand{\Jla}{\mathfrak{J}}
\newcommand{\Qla}{\mathfrak{Q}}
\newcommand{\Vla}{\mathfrak{V}}

\newcommand{\Sla}{\mathfrak{S}}
\newcommand{\ts}{\mathfrak{t}}

\newcommand{\Ind}[1]{\mathds{1}\brc*{#1}}

\usepackage{empheq}

\usepackage{selectp}

\title{Reinforcement Learning with Multi-Step Lookahead Information Via Adaptive Batching}

\author{%
  Nadav Merlis \\
  Technion -- Israel Institute of Technology\\
  \texttt{nmerlis@technion.ac.il}
}

\begin{document}
\doparttoc
\faketableofcontents

\maketitle

\begin{abstract}
  We study tabular finite-horizon reinforcement learning problems with multiple steps of lookahead information. Before acting, the learner observes $\ell$ steps of future transition and reward realizations: the exact state the agent would reach and the rewards it would collect under any possible course of action. While it has been shown that such information can drastically boost the value, finding the optimal policy is NP-hard, and it is common to apply one of two tractable heuristics: processing the lookahead in chunks of predefined sizes ('fixed batching policies'), and model predictive control. We first illustrate the problems with these two approaches and propose utilizing the lookahead in adaptive (state-dependent) batches; we refer to such policies as adaptive batching policies (ABPs). We derive the optimal Bellman equations for these strategies and design an optimistic regret-minimizing algorithm that enables learning the optimal ABP when interacting with unknown environments. Our regret bounds are order-optimal up to a potential factor of the lookahead horizon $\ell$, which can usually be considered a small constant.
\end{abstract}

\section{Introduction}
In reinforcement learning (RL), an agent sequentially interacts with potentially unknown environments, gaining rewards or penalties for its interaction \citep{sutton2018reinforcement}. The instantaneous rewards are influenced by the agent's action, as well as by the state of the environment, which in turn changes due to the agent's behavior. As such, to maximize the cumulative rewards, agents should not only maximize their immediate gain, but also take into consideration how their actions affect the environment. In the standard model, the agent relies on the current state of the environment to pick an action, and then observes its outcome: the reward it earns and the new state of the environment.

We study problems in which agents obtain additional valuable information before acting, in the form of \emph{lookahead} information. That is, agents get to observe the \emph{realized outcomes} for all their potential actions $\ell$ steps into the future: both where they would transition to and what rewards they would collect. For example, in navigation, traffic information predicts with high certainty which roads are available and how long it would take to traverse them; in packing problems, the agent often observes the next few items to be packed; in resource management, nearby supply and demand can be accurately estimated; and so on. This type of information has been shown to greatly boost the potential reward that the agent may collect: assuming $A$ actions and interaction length $H$, the value can increase by a factor of at most $\approx A^{\ell}$ for reward information \citep{merlis2024value}, and of $\approx A^{\Omega(H)}$ with even one step of transition information \citep{merlis2024reinforcement}. However, correctly utilizing this information is tricky due to its exponential nature. Specifically, while one can efficiently plan and learn with $1$-step lookahead information \citep{merlis2024reinforcement}, it has been proven in discounted settings that exact optimal planning with $2$-step lookahead is NP-hard \citep{pla2026hardness}, and the best known approximation algorithms are exponential in $\ell$ \citep{pla2026near}.

In this work, we propose a family of tractable policies for tabular finite-horizon RL problems with $\ell$-step lookahead information that we term adaptive batching policies (ABPs). Specifically, at each step, these policies decide how many steps to look into the future, depending also on the current state of the environment, and use this information to its end before asking for new information. Such policies clearly improve agents that process information in predefined batches -- which we show can catastrophically fail even in extremely simple environments -- but are still tractable to optimize via dynamic programming; this is in contrast to model predictive controllers, for which we prove that dynamic programming schemes can be exponentially suboptimal. Specifically, we derive the optimal Bellman equations for ABPs and explain how to efficiently find the optimal policy. We then show that the optimal ABP can also be learned while interacting with an unknown environment in episodes and prove near-tight regret bounds (up to factors of $\ell$, which is usually a small constant).

\subsection{Related Work}
We first emphasize the important distinction between obtaining lookahead information and performing lookahead planning. In control, lookahead often refers to obtaining additional knowledge about the future behavior of the system in the near future: the exact/approximate dynamics, the noise realization in the system, and more. This is the type of lookahead we study in this paper, and we refer to it as \emph{lookahead information}. This information can also be viewed as a \emph{prediction}. For example, in linear control with quadratic costs, and when lookahead reveals zero-mean additive i.i.d. noise, it is well known that model predictive control \citep[MPC,][]{camacho2007model} strategies are optimal \citep[e.g., ][]{yu2020power}. In particular, previous works have studied the optimality of MPC-like strategies in the presence of lookahead information in various settings, including linear and nonlinear control systems, with stochastic or adversarial state perturbations, inexact information, and more \citep{yu2020power, li2019online,zhang2021regret, lin2021perturbation,lin2022bounded}. 

In contrast, in the context of RL and/or Markov decision processes (MDPs), the term lookahead usually refers to propagating the model for multiple steps and/or sampling trajectories of a certain length as part of a planning or value-learning procedure \citep{tamar2017learning,efroni2020online,moerland2020think,rosenberg2023planning,chung2023thinker}. Specifically, this procedure does not incorporate any new information about the future of the system in the current interaction, but aims to provide more efficient means to calculate the optimal Markov policy. When working with lookahead information, it may be useful to adapt such approaches and perform multi-step planning as a subroutine, but this line of work studies settings that are fundamentally different from ours.

In tabular RL with lookahead information, \citep{merlis2024value} analyzed how the optimal value increases due to multiple steps of future reward information, as a function of the lookahead range and properties of the environment. However, they did not propose any algorithm for interaction. 
\citet{merlis2024reinforcement} then studied problems in which agents observe either the immediate rewards or transitions before acting (one-step reward/transition lookahead). They derived the optimal Bellman equations and proposed learning algorithms with order-optimal regret bounds. We show even better rates with combined one-step reward and transition lookahead information and extend both the planning and learning beyond one-step lookahead. 

For multi-step lookahead, \citet{pla2026hardness} prove that even with two steps of transition lookahead, exact planning is NP-hard in discounted/average tabular MDPs (and also provide a polynomial-time algorithm for one-step transition lookahead). Moreover, the best known results for approximate planning with multi-step lookahead are exponential in the number of lookahead steps \citep{pla2026near}. 
Consequently, we cannot aim to find the optimal lookahead policy, but rather limit ourselves to adaptive batching policies, which can be efficiently calculated. In \citep{zhang2025predictive}, the authors study the performance of MPC strategies in nonstationary environments where the agent obtains noisy multi-step predictions on the expected rewards and transition kernels. They show that given a sufficient lookahead horizon and mixing, MPC will converge to the optimal agent. We study a different setting (stochastic environments where we observe future realizations) in different regimes (no mixing assumptions), so the results are incomparable. In particular, we demonstrate how MPC can critically fail in our setting. In addition, \citep{lu2025reinforcement} studies an offline learning problem in which, given access to a generative model, the goal is to estimate the optimal policy with multi-step noisy transition lookahead information. While there are similarities in the underlying model, we study orthogonal aspects of the problem. For instance, they allow noisy transition lookahead observations, while we study perfect predictions, but also include reward lookahead. Moreover, they assume that information arrives at predefined timesteps $1,\ell+1,2\ell+1,\dots$; we explain why this can cause extremely suboptimal behavior and instead let the agent adaptively decide its lookahead range. Finally, both the interaction protocol and objectives are very different. We focus on learning during the interaction and regret minimization, balancing exploration and exploitation. In contrast, \citep{lu2025reinforcement} assumes access to samples of an equal (predefined) size from all state-action pairs and lookahead observations, and attempts to output a near-optimal policy using this data.

We end this section by mentioning a small selection from the numerous works on regret analysis in tabular MDPs \citep{jaksch2010near,azar2017minimax,jin2018q,dann2019policy,zanette2019tighter,efroni2019tight,efroni2021confidence,simchowitz2019non,zhang2021reinforcement,zhang2024settling,merlis2024reinforcement}. Specifically, we use the monotone bonus properties from \citep{zhang2021reinforcement,zhang2024settling,merlis2024reinforcement} and some of the techniques in \citep{efroni2021confidence,merlis2024reinforcement}.

\section{Setting}
An episodic tabular Markov decision process (MDP) is defined by the tuple $\Mcal=\br*{\Scal,\Acal,\Rcal,P,H}$, where $\Scal$ and $\Acal$ are the state and action spaces (of respective sizes $S,A$), $\Rcal$ is the reward distribution, $P$ is the transition kernel and $H$ is the interaction horizon. At step $h$, given that the agent is in state $s_h=s$ and plays an action $a_h=a$, the agent obtains a reward $R_h\sim \Rcal_h(s,a)$, bounded in $[0,1]$, and transitions to the state $s_{h+1}=s'$ with probability $P_h(s'\vert s,a)$. The rewards and transitions are assumed to be mutually independent between timesteps, but can be arbitrarily correlated between states and actions at a given step $h$. A policy $\pi$ prescribes which action the agent takes during the interaction. A deterministic Markov policy $\pi\in\Pi_{det}$ maps each state $s$ and step $h$ to an action $\pi_h(s)$. Policies are measured by their value -- their expected cumulative rewards
$V^\pi_1(s) \!=\! \E\brs*{\sum_{t=1}^H R_t\vert s_1=s, \pi}$, where the expectation is w.r.t. both the environment and the policy.

\textbf{Lookahead information.} $\ell$-step lookahead information consists of all realized rewards and transition information accessible in $\ell$ steps of interaction. Formally, given a deterministic policy $\phi\in\Pi_{det}$,\footnote{We emphasize the distinction between a global policy for the entire interaction and a local policy for the duration of the lookahead by denoting the former with $\pi$ and the latter with $\phi$.} we define a trajectory realization starting at $s_h=s$ and ending at $t$ when following $a_h=\phi_h(s_h)$ as 
\begin{align}
    \tau_{h:t}^\phi(s)=\brc*{s_h,a_h,R_h,\dots s_{t-1},a_{t-1},R_{t-1},s_t \vert s_h=s,\phi}. \label{eq: trajectory}
\end{align}
The $\ell$-step lookahead information contains the \emph{realized} trajectories from step $h$ to step $h+\ell$ for all possible policies, that is, $I_{h,\ell}(s) = \brc*{\tau_{h:h+\ell}^\phi(s)}_{\phi\in\Pi_{det}}$. We remark that we implicitly assume consistency between realized trajectories: if, in two different realized trajectories, an agent visits a state $s_t$ and plays $a_t$, then it will necessarily observe the same reward $R_t$ and next state $s_{t+1}$ in both trajectories. We use $\Ical_{h,\ell}(s)$ to denote the distribution of lookahead information sampled starting at step $h$ and state $s$, independently from all history up to step $h$. When clear for the context, we omit the lookahead range $\ell$ and/or the state $s$ and, e.g., write $I_h$ or $\Ical_h(s)$. Near the end of the episode, when we have fewer than $\ell$ remaining steps, we denote the effective lookahead range by $\ell_h=\min\brc*{\ell,H-h+1}$.

We also denote the reward and state at step $t$ given lookahead information at $s_h=s$ and a policy $\phi\in\Pi_{det}$ as $\Rla_{t\vert h}(s,\phi,I)$ and $\sla_{t\vert h}(s,\phi,I)$. Specifically for one-step lookahead, the policy is a fixed action, so we can also write $\Rla_{h\vert h}(s,a,I)$ and $\sla_{h+1\vert h}(s,a,I)$. Finally, for brevity, when the lookahead information $I$ is clear from the context, we sometimes omit it and write, for example, $\Rla_{t\vert h}(s,\phi)$. We describe in depth one potential model for lookahead information generation in \Cref{appendix: data generation} (though we recommend reading it only after the formal introduction of ABPs in \Cref{section:ABPs}). 

\subsection{Lookahead Policies}

An $\ell$-step lookahead policy observes lookahead information at every timestep and decides which action to take based on both the current state $s$ and all future trajectories of length $\ell$, that is, $a_h=\pi_h(s_h, I_{h,\ell}(s_h))$. Importantly, $I$ contains an exponential amount of information, and so even with $2$ steps of lookahead, finding the optimal policy is NP-hard \citep{pla2026hardness}. 
One core reason for this hardness is the temporal correlations that lookahead information induces. Specifically, the arrival to state $s_h=s$ may be affected by previous lookahead information that overlaps with $I_{h,\ell}(s)$. 
\begin{example}
    Assume that starting from $s_{init}$, we can move to either $s_1$ or $s_2$. In $s_1$, there is a Bernoulli reward $R(s_1)\sim \Ber(p)$, but assume that a $2$-step lookahead agent $\pi^{L}$ decides to go to $s_1$ only if $R(s_1)=1$ (and otherwise goes to $s_2$). Even though for any non-lookahead policy $\pi$, it holds that $\E[R(s_1)\vert s_1,\pi]=p$, when playing the lookahead policy $\pi^L$, we have $\E[R(s_1)\vert s_1,\pi^L]=1$. Thus, the distribution of the lookahead information from $s_1$, $\Ical_{2,2}(s_1)$ is `distorted' by the policy $\pi^L$.
\end{example}
Consequently, given $s_h$, lookahead information is not sampled `independently' from the distributions $\Rcal,P$, but also has a complex dependence on $s_h$. One common way to decorrelate the lookahead from the state is to process the information in \emph{batches}. For example, in \citep{lu2025reinforcement}, $\ell$-step transition information is observed only at discrete time points $1, \ell+1, 2\ell+1,\dots$, and until each information chunk is exhausted, no new lookahead information is gained. Then, the state in which new lookahead information is gained $s_{n\ell+1}$ only depends on information ending at step $n\ell+1$, and is not affected in any way by the information in $I_{n\ell+1,\ell}(s_{n\ell+1})$. 
Therefore, when following such policies, we can assume that the lookahead information is sampled independently from the environment model. More generally, one might decide to process information at predefined intervals of size $B\le \ell$, so that at time steps $h\in\brc*{nB+1,\dots,(n+1)B}$, the policy only depends on information at step $nB+1$: $a_h=\pi_h(s_{nB+1},I_{nB+1,B}(s_{nB+1}))$. We refer to such policies as \emph{fixed batching policies} and call $B$ the batching horizon or the lookahead range. Unfortunately, there are environments in which all such policies are exponentially suboptimal.

\begin{figure}[t]
  \centering
  \begin{subfigure}[t]{0.48\linewidth}
    \centering
    \resizebox{\linewidth}{!}{%
      \begin{tikzpicture}[x=0.80cm,y=0.75cm]
        \path[use as bounding box] (-0.35,-0.78) rectangle (7.35,5.12);
        \node[envstate] (s0) at (3.5,4.85) {$s_1$};
        \node[envstate] (sr) at (3.5,3.85) {$s_r$};
        \draw[envedge] (s0) -- node[right,envtext] {} (sr);
        \node[envjunction] (u1) at (1.75,3.00) {};
        \node[envjunction] (u2) at (5.25,3.00) {};
        \draw[envedge] (sr) -- (u1);
        \draw[envedge] (sr) -- (u2);
        \node[envjunction] (v1) at (0.90,2.02) {};
        \node[envjunction] (v2) at (2.65,2.02) {};
        \node[envjunction] (v3) at (4.35,2.02) {};
        \node[envjunction] (v4) at (6.10,2.02) {};
        \draw[envedge] (u1) -- (v1);
        \draw[envedge] (u1) -- (v2);
        \draw[envedge] (u2) -- (v3);
        \draw[envedge] (u2) -- (v4);
        \foreach \x/\parent [count=\i] in
          {0.45/v1,1.35/v1,2.20/v2,3.10/v2,
           3.90/v3,4.80/v3,5.65/v4,6.55/v4} {
          \node[envleaf] (leaf\i) at (\x,1.02) {};
          \draw[envpath] (\parent) -- (leaf\i);
        }
        \node[envtext] at (3.5,3.04) {$\cdots$};
        \node[envtext] at (5.96,3.82) {tree depth $=\ell$};
        \node[envtext] at (3.5,0.43)
          {$A^{\ell-1}$ leaves $s_{\rm leaf}$, each with $A$ actions};
        \node[envtext] at (3.5,-0.06)
          {independent $\Ber(A^{-\ell})$ rewards};
      \end{tikzpicture}%
    }
    \caption{Counterexample of \Cref{claim: failure fixed batching}. Fixed batching strategies fail since the state $s_1$ prevents them from observing the reward at a good timing: for a batching horizon $B$, rewards will be observed at a distance of $\approx\min\brc*{B,\ell-B}\leq \ell/2$ from the leaves, which leads to an expected value of $\approx A^{-\ell/2}$.}
    \label{fig:fixed-batching-counterexample}
  \end{subfigure}\hfill
  \begin{subfigure}[t]{0.48\linewidth}
    \centering
    \resizebox{\linewidth}{!}{%
      \begin{tikzpicture}[x=0.80cm,y=0.75cm]
        \path[use as bounding box] (-0.35,-0.78) rectangle (7.35,5.12);
        \node[envstate] (s1) at (3.5,4.85) {$s_1$};
        \node[envstate] (tree) at (1.75,3.85) {$s_{\rm Tr}$};
        \node[envstate] (line) at (5.55,3.85) {$s_L$};
        \draw[envedge] (s1) -- node[above left,envtext] {$R_{\rm Tr}=0$} (tree);
        \draw[envedge] (s1) -- node[above right,envtext] {$R_L$} (line);
        \node[envjunction] (u1) at (0.90,3.00) {};
        \node[envjunction] (u2) at (2.60,3.00) {};
        \draw[envedge] (tree) -- (u1);
        \draw[envedge] (tree) -- (u2);
        \node[envjunction] (v1) at (0.45,2.02) {};
        \node[envjunction] (v2) at (1.35,2.02) {};
        \node[envjunction] (v3) at (2.15,2.02) {};
        \node[envjunction] (v4) at (3.05,2.02) {};
        \draw[envedge] (u1) -- (v1);
        \draw[envedge] (u1) -- (v2);
        \draw[envedge] (u2) -- (v3);
        \draw[envedge] (u2) -- (v4);
        \foreach \x/\parent [count=\i] in
          {0.23/v1,0.67/v1,1.13/v2,1.57/v2,
           1.93/v3,2.37/v3,2.83/v4,3.27/v4} {
          \node[envleaf] (leaf\i) at (\x,1.02) {};
          \draw[envpath] (\parent) -- (leaf\i);
        }
        \node[envtext] at (1.75,3.00) {$\cdots$};
        \node[envtext] at (4.00,3.17) {tree/line \\ depth $=\ell$};
        \node[envtext] at (1.75,0.32) {$A^{\ell-1}$ leaves $s^X$};
        \node[envstate] (mid) at (5.55,2.53) {};
        \node[envstate] (y) at (5.55,1.02) {$s^Y$};
        \draw[envedge] (line) -- (mid);
        \draw[envpath] (mid) -- (y);
        \node[envtext,fill=white,inner sep=1pt] at (5.55,1.75) {$\vdots$};
        \node[envtext] at (5.55,0.32) {one line endpoint};
      \end{tikzpicture}%
    }
    \caption{Counterexample of \Cref{claim: failure MPC informal}. 
    MPC fails since it can be fooled into always heading to the line even though the tree is exponentially better. This is done by creating two candidate reward distributions: Bernoulli and deterministic. Then, we place the distribution with the higher MPC-value at $s^Y$ and the other at $s^X$.}
    \label{fig:mpc-counterexample}
  \end{subfigure}
  \caption{The diagnostic environments in Claims \ref{claim: failure fixed batching} and \ref{claim: failure MPC informal}. Both environments have deterministic transitions, start at state $s_1$, and are of horizon $\ell+1$. Rewards are zero except at the last step (and at the first step in \Cref{fig:mpc-counterexample} as a tie-breaking mechanism). In both cases, observing Bernoulli rewards $\Ber(A^{-\ell})$ from the root of a full tree exponentially augments their value: to collect a unit reward, only one leaf needs to realize $R=1$, which happens w.p. $\geq1-1/e$. Observing rewards later in the tree, or not leveraging this exponential increase, is what makes fixed-batching and MPC fail.
  }
  \label{fig:diagnostic-environments}
\end{figure}

\begin{restatable}{claim-rst}{failureFixedBatch}
\label{claim: failure fixed batching}
    For any $\ell\ge2, A\ge2,H\ge \ell+1$, there exists an environment with $S\le 1+A^{\ell}$ and deterministic transitions s.t. for some fixed initial state, any fixed batching policy collects a value of at most $A^{-\ell/2+1}$ in expectation. In contrast, the optimal $\ell$-step lookahead value is larger than $1/2$.
\end{restatable}
The proof can be found in \Cref{appendix: failure fixed batching}, and a numerical illustration of the values is at \Cref{appendix:tables}. The environment that causes this behavior is extremely simple: an initial state that deterministically leads to a tree of depth $\ell$ with Bernoulli rewards at its leaves (see also \Cref{fig:fixed-batching-counterexample}). An optimal lookahead policy will go to the root of the tree and navigate to one of the few leaves whose reward is realized to be $R=1$ (if any exist). Yet, we show that for any choice of batching horizon $B$, we observe the rewards when already being deep inside the tree -- we lose most of the navigation power of the agent.

Another prominent candidate designed to handle lookahead information is model predictive control (MPC). The key principle in such controllers is to replace the expected rewards/transition kernel inside the lookahead range by the lookahead realization and plan until the end of the lookahead (with some appropriate value after the lookahead ends). Then, the MPC only plays the first action in the solution, obtains the updated lookahead information (with one additional step), and repeats the entire planning process to incorporate the new information. The most critical issue in these controllers is the value computation; we show that any backwards-induced value calculation necessarily fails exponentially in some environments; i.e., letting the value only depend on future events can be extremely suboptimal. Specifically, we require the following mild assumption.
\begin{assumption}[informal]
    \label{ass: backward informal}
    For all states that deterministically lead to a zero-rewarding absorbing state, the MPC values can only depend on the joint reward distribution of this state's actions and on the timestep.
\end{assumption}
This is a strict requirement for backward value calculations, but also extends to more general computation schemes. Under this assumption, we prove the following result:
\begin{claim}[informal]
\label{claim: failure MPC informal}
    For any $\ell\ge2,A\ge2,H\ge\ell+1$ and any MPC agent whose values satisfy \Cref{ass: backward informal}, there exists an environment with $S\leq A^\ell+\ell+1$ states and deterministic transitions in which the ratio between the value of the MPC agent and the optimal value is less than $A^{1-\ell/2}$. In contrast, there exist optimal MPC agents in all constructed environments.
\end{claim}
Most standard procedures for calculating MPC agents rely on backward induction (and follow our assumption), so they cannot be guaranteed to converge to a near-optimal MPC agent even in planning scenarios (where all distributions are known). 
Due to space limitations, we state the exact formulation of the controller, the assumption and the proof to \Cref{appendix: failure MPC}, and only provide intuition. We also refer to \Cref{appendix:tables} for a numerical illustration of the values in the counterexample.
\begin{proofsketch}
The environment we consider starts at a fixed initial state and allows the agent to choose one of two paths: the first leads to a chain of states (`a line') and the second to a full tree, both of depth $\ell$ (see \Cref{fig:mpc-counterexample}). Rewards are located both at the last state of the line and in all the leaves of the tree, and all these states lead to a terminal non-rewarding state. The agent gets to see the rewards only after deciding whether to traverse the tree or follow the line.

By \Cref{ass: backward informal}, the values at both the leaves of the tree and at the end of the line only depend on the reward distribution at these states. We propose two candidates for the reward distribution to be placed in these states.
The first candidate $s^D$ deterministically assigns a reward of $A^{-\ell/2}$ to all actions, while the second $s^B$ independently assigns a reward $\Ber(A^{-\ell})$ to all actions. Notably, if $s^B$-like states are located at all leaf states, the agent can actively navigate to realized unit rewards to collect a constant value in expectation, but at the end of the line, $s^B$ can yield a value of at most $A^{-\ell+1}$. However, the value calculation procedure will assign the same value for $s^B$ in both cases: we use this to trick the algorithm into always choosing the line, regardless of how it allocates the values, while ensuring that the value of the tree is always higher by a factor of $A^{\ell/2-1}$.   
\end{proofsketch}

\begin{wrapfigure}{R}{0.48\linewidth}
\begin{minipage}{1\linewidth}
\vspace{-.8cm}
\begin{algorithm}[H]
\caption{Adaptive Batching Policies (ABPs)}\label{alg:adaptive batching}
\begin{algorithmic}[1]
\State  \textbf{Input}: Initial state $s_1$
\State \textbf{Initialize}: $h=1$ (or $\ts^\pi_1=1, i=1$)
\While{$h\leq H$}
    \State Pick $B_h=B_h^{\pi}(s_h)\in \brc*{1,\dots,\ell_h}$ 
    \State Observe $I=I_{h,B_h}(s_h)$ and set 
    \Statex \hspace{.5cm} $\phi^h=\phi^h(s_h,I;\pi)$
    \For{$t=h,\dots,h+B_h-1$} 
    \State Play $a_t = \phi_{t}^h(s_t)$
    \State Collect $R_t=\Rla_{t\vert h}(s_h,\phi,I)$ and 
    \Statex \hspace{1cm} transition to $s_{t+1}=\sla_{t+1\vert h}(s_h,\phi,I)$
    \EndFor
    \State $h \!\leftarrow\! h\ns+\ns B_h$ (or $\ts^\pi_{i+1} \!\!\leftarrow\! \ts^\pi_i+B_{\ts^\pi_i},i \!\leftarrow\! i\ns+\ns1$)
\EndWhile
\end{algorithmic}
\end{algorithm}
 \end{minipage}
 \end{wrapfigure}

\section{Adaptive batching policies (ABPs)}  
\label{section:ABPs}

\Cref{claim: failure fixed batching} demonstrates an important failure point in fixed batching policies: there might be critical times or states in which the lookahead information is critical, while others in which it has little benefit. In the example, the critical state is the root of the tree, while lookahead in the initial state has no significance. Fixed batching cannot adapt to such timing or states, even in simple environments, and this limitation will be even more severe in complex stochastic environments (e.g., stochastic delays). To overcome this, we propose working with a wider class of \emph{adaptive batching policies} (ABPs). Specifically, we allow the lookahead range to depend on the state and timestep. Starting at $s_1$, we pick a batching horizon of $B_1=B_1(s_1)\leq \ell$ and obtain $B_1$-step lookahead information. Then, we play for $B_1$ steps and reach state $s_{B_1+1}$. From there, we pick a new batching horizon $B_{B_1+1}=B_{B_1+1}(s_{B_1+1})$, act for $B_{B_1+1}$ steps, and so forth.

More formally, an ABP $\pi$ is specified by two mappings: i) Batching horizons $B_h=B^\pi_h(s)\in [\ell_h]$ that specify the utilized lookahead range when starting a new batch in step $h$ and state $s$, and ii) Deterministic policies $\phi^h=\phi^h(s,I;\pi)\in\Pi_{det}$ that determine how the agent will act from step $h$ to step $h+B_h-1$ given the lookahead information.

Notably, when the transitions are stochastic, the starting times of new batches may become random; we denote by $\ts_i^\pi$ the (random) step in which a policy $\pi$ started its $i^{th}$ batch, and by $E^n_h$, the event in which some batch started at step $h$. In particular, we always start a new batch at the first step, and so $\ts^{\pi}_1=1$ and $E^n_1$ always holds. Then, in step $h=\ts^{\pi}_i$, the ABP picks $B_h=B_h^{\pi}(s_h)$ and observes $I_h=I_{h,B_h}(s_h)$ (and \emph{not} $I_{h,\ell}(s_h)$\footnote{We discuss the implications of this feedback in \Cref{section: learning}.}). The ABP plays $\phi^h$ for $B_h$ steps, ending in $s_{h+B_h}=\sla_{h+B_h\vert h}(s_h,\phi^h,I_h)$, and the next batch then starts at $\ts^{\pi}_{i+1}=\ts^{\pi}_i+B_h$ in state $s_{\ts^{\pi}_{i+1}}=s_{h+B_h}$. The process continues until $H$ steps have been played. For clarity, we also describe the interaction protocol in \Cref{alg:adaptive batching}. We denote the set of policies that abide by this protocol as $\Pi_B$.
\begin{remark}
    Since the lookahead contains perfect information on the transition realization, a deterministic policy is completely equivalent to pre-choosing a sequence of $B_h$ actions (and this is how transition lookahead is defined in \citep{pla2026hardness}). Nonetheless, we still define the problem with Markov policies to facilitate future extensions to imperfect lookahead information. 
\end{remark}
Going back to the example of \Cref{claim: failure fixed batching}, even if the delay from $s_1$ to the root of the tree $s_r$ is stochastic, an ABP could look one step ahead at a time by setting $B_h^\pi(s_1)=1$, waiting to reach $s_r$, and only then utilize the full lookahead ($B_h^\pi(s_r)=\ell$) to optimally collect the rewards at the leaves. Similarly, it is also easy to verify that an ABP can optimally solve the tree-and-line example in \Cref{claim: failure MPC informal}; the optimal solution will choose $B_1=1$ and then continue until the end with $B_2=\ell$.

\textbf{Values and performance criterion.} We define the value of an ABP $\pi$ upon starting a new batch at step $h$ and state $s$ to be 
$V^\pi_h(s) = \E\brs*{\sum_{t=h}^H R_t\vert s_h=s,E^n_h, \pi},$ 
and the optimal value of ABPs $V_1^{*}(s) = \max_{\pi\in\Pi_B}V_1^\pi(s)$. Importantly, since a new batch starts at step $h$, there is no overlap between lookahead information obtained before step $h$ and any event after step $h$ (except for $s_h$, on which we condition). Therefore, the value will be the same regardless of the path of the agent on its way to a new batch at $s_h=s$.

In learning scenarios, an agent interacts with the environment for $K$ episodes, aiming to compete with the best ABP. The initial state in episode $k$ is $s_1^k$, and we allow it to be arbitrarily chosen by an adversary. Denoting by $\pi^k$ the policy of the agent in episode $k$, we measure agents by their regret, defined as the cumulative difference between the value of an optimal ABP and their expected cumulative reward: $\Regret(K) = \sum_{k=1}^K\br*{V_1^{*}(s_1^k)-V_1^{\pi^k}(s_1^k)}.$

\subsection{Optimal Planning}
We start with the case where the model of the environment is perfectly known, with the goal of determining the optimal ABP and its value. 
In \Cref{prop: opt values}, we derive the dynamic programming equations that characterize the optimal value of ABPs and yield the optimal policies. Let $\Qla^*_h(s,B;I,V)$ be the optimal reward gained starting from step $h$ and state $s$, given $B$-step lookahead $I$ and assuming a value $V\in\R^{\Scal}$ at the end of the lookahead range
\begin{align}
    &\Qla^*_h(s,B;I,V) = \max_{\phi\in\Pi_{det}}\brc*{\sum_{t=h}^{h+B-1}\Rla_{t\vert h}(s,\phi,I) +V(\sla_{h+B\vert h}(s,\phi,I))}. \label{eq: Q-val def}
\end{align}
Then, the following holds:
\begin{restatable}{proposition-rst}{optimalValue}
    \label{prop: opt values}
    There exists an optimal adaptive batching policy $\pi^{*}$ that maximizes the value $V^{\pi}_1(s)$ simultaneously for all $s\in\Scal$. For any $h\in[H]$ and $s\in\Scal$, the optimal values are given by the dynamic programming equation
    \begin{align}
        & V^{*}_h(s) = \max_{B\in[\ell_h]}\E_{I\sim\Ical_{h,B}(s)}\brs*{\Qla^*_h(s,B;I,V^{*}_{h+B})}, \label{eq: DP val}
    \end{align}
    where $V^{*}_{H+1}(s) = 0$ for all $s\in\Scal$ and the expectation samples fresh lookahead information independently between timesteps. An optimal policy chooses lookahead ranges according to 
    \begin{align}
        B_h^*(s) \in \argmax_{B\in[\ell_h]}\E_{I\sim\Ical_{h,B}(s)}\brs*{\Qla^*_h(s,B;I,V^{*}_{h+B})}. \label{eq: DP range}
    \end{align}
    Finally, when starting a new batch at step $h$ and state $s$ of length $B_h^*=B_h^*(s)$ and with lookahead information $I=I_{h,B_h^*}(s)$, it is optimal to play the deterministic policy 
    \begin{align}
         \phi^{*,h} \in &\argmax_{\phi\in\Pi_{det}}\brc*{
         \sum_{b=0}^{B_h^*-1}\Rla_{h+b\vert h}(s,\phi,I) + V^{*}_{h+B_h^*}\br*{\sla_{h+B_h^*\vert h}(s,\phi,I)}}.  \label{eq: DP policy}
    \end{align}
\end{restatable}
The proof can be found in \Cref{appendix: augmented MDP}. Inspired by \citep{merlis2024reinforcement} in the case of one-step lookahead, we design an augmented Markov environment $\MB$ in which the state includes both the lookahead information and the number of steps until the current batch ends. The action space is adapted to also include the choice of the batching horizon, in steps where it is relevant. We explain why this construction is almost equivalent to ABPs in the original MDP, with one delicate change: Markov policies in $\MB$ may also depend on the number of past batches (while ABPs do not use this as input). We show that the optimal policy does not use this extra information and results in being an ABP. We also rely on this construction to derive the value of arbitrary ABPs (\Cref{corollary: ABP value} in the appendix).

\textbf{Efficient Value Computation.} We now briefly discuss how to efficiently solve the dynamic programming equations of ABPs. Starting from a step $h$ and a state $s$, let $\Jla_{h,B}^I(s,s')$ be the maximal cumulative reward after playing $B$ steps, given lookahead information $I$ and under the constraint that the environment arrives at state $s'$ on step $h+B$:
\begin{align}
    \Jla_{h,B}^I(s,s') = \max_{\phi\in\Pi_{det}: \sla_{h+B\vert h}(s,\phi,I)=s'}\brc*{\sum_{t=h}^{h+B-1}\Rla_{t\vert h}(s,\phi,I)}. \label{eq: efficient reward rep}
\end{align}
In addition, let $\Sla_{h,B}^I(s)$ be the set of reachable states after $B$ steps 
\begin{align}
    \Sla_{h,B}^I(s) = \brc*{s'\in\Scal: \exists\phi\in\Pi_{det},\ s.t.\ \sla_{h+B\vert h}(s,\phi,I)=s'}. \label{eq: reachable states}
\end{align}
Using standard dynamic programming arguments, we have 
\begin{align}
    &\Qla^*_h(s,B;I,V)
    = \max_{s'\in \Sla_{h,B}^I(s)}\brc*{\Jla_{h,B}^I(s,s') + V(s')}, \label{eq: J and S to Q}
\end{align}
and given lookahead information, both $\Jla$ and $\Sla$ can be updated recursively by\footnote{For brevity, we slightly abuse the use of notations $\sla,\Rla$ and input the lookahead information in step $h$; notice that it directly determines the relevant state/rewards in step $h+B$.}
\begin{align}
    &\Jla_{h,B+1}^I(s,s') = \max_{\substack{\tilde{s}\in\Sla_{h,B}^I(s), a\in\Acal: \\\sla_{h+B+1\vert h+B}(\tilde{s},a,I)=s'}}\brc*{\Jla_{h,B}^I(s,\tilde{s}) + \Rla_{h+B\vert h+B}(\tilde{s},a,I)}, \quad \text{and}\nonumber\\
    &\Sla_{h,B+1}^I(s) = \brc*{s'\in\Scal: \exists \tilde{s}\in\Sla_{h,B}^I(s), a\in\Acal,\,s.t.\enspace\sla_{h+B+1\vert h+B}(\tilde{s},a,I)=s'}, \label{eq: efficient DP J and S}
\end{align}
with the initialization $\Jla_{h,0}^I(s,s)=0$ and $\Sla_{h,0}^I(s)=\brc*{s}$. If $s'$ is not reachable from $s$, we arbitrarily define $\Jla_{h,B+1}^I(s,s')=0$. For any step $h$ state $s$ and lookahead information $I$, the computation complexity of these calculations is $\Ocal\br*{SA\ell }$, and the memory complexity is $\Ocal\br*{S\ell}$. Therefore, we can efficiently compute $\Qla^*_h(s,B;I,V^{*}_{h+B})$ in \Cref{eq: DP val,eq: DP range}, and only need to approximate the expectation w.r.t. the lookahead information. Luckily, since rewards are bounded, all quantities concentrate well, and one could sample $poly\br*{1/\epsilon,S,H,\log\frac{1}{\delta}}$ lookahead realizations to estimate the values up to an $\epsilon$-accuracy w.h.p. using this procedure. Finally, we remark that given $B_h$ and $I_h$, after calculating $\Jla$ and $\Sla$ and solving for the state $s_{h+B_h}$, the optimal policy $\phi^{*,h}$ can be calculated by backtracking from the final state to the initial state $s_h=s$.

\section{Learning the Optimal ABP}
\label{section: learning}
Finally, we address the problem of learning to interact with unknown environments. We study episodic settings and assume that the interaction lasts for $K$ episodes. Following the ABP protocol, each episode is divided into batches determined by the agent. Starting from an arbitrary state $s_1^k$, the agent picks the first batching horizon $B_1^k$ and observes $B_1^k$-step lookahead information $I_h^k$ containing both realized rewards and transitions. Then, it plays for $B_1^k$ steps according to a policy $\phi^{k,h}$ until reaching the starting point of the second batch $\ts^k_2=\ts^k_1+B_1^k$ (with $\ts^k_1=1$), and so forth until $H$ steps were played. We emphasize that we assume much weaker feedback than the original $\ell$-step lookahead: first, the lookahead is only observed at the beginning of batches (and not in every step), and second, the observed lookahead range is $B_h^k\leq \ell$. This can be very beneficial when the lookahead information is costly -- whether the cost is computational or the agent must pay to acquire it.

We continue and define several useful empirical quantities. Let $\Kcal_h^{k}(s,B)\subseteq[k-1]$ be all episodes up to $k-1$ (inclusive) in which the agent started a new batch in state $s$ at step $h$ and asked to observe a lookahead range of $B$. The number of such episodes is denoted by $n_h^{k}(s,B) = \abs*{\Kcal_h^{k}(s,B)}$. Then, we denote the empirical expectation w.r.t. the collected lookahead data of an arbitrary function $f$ by
\begin{align*}
    \widehat{\E}_{h,s,B}^k[f(I)] = \frac{1}{n_h^{k}(s,B)\vee 1} \sum_{m\in\Kcal_h^{k}(s,B)}f(I_h^m(s)),
\end{align*}
where $x\vee y = \max\brc*{x,y}$. 
We remark that as long as computing $f(I)$ is tractable, calculating this expectation requires at most $k$ computations. We define the empirical variance of the Q-values to be 
\begin{align}
    &\widehat{\VAR}_h^k(s,B;V) = \widehat{\E}_{h,s,B}^k\brs*{\Qla^*_h(s,B;I,V)^2} - \br*{\widehat{\E}_{h,s,B}^k\brs*{\Qla^*_h(s,B;I,V)}}^2, \label{eq: empirical variance}
\end{align}
where $\Qla^*$ is defined in \cref{eq: Q-val def}. 

Using these notations, we are ready to present our algorithm, which we call Adaptive Lookahead Upper Confidence Bound (AL-UCB). We adapt the optimistic algorithms for RL, and, in particular, the one-step lookahead algorithms in \citep{merlis2024reinforcement}. In each episode, the algorithm recursively calculates optimistic values $\bar{V}^k_h(s)$ and utilizes them to determine the batching horizons $B_h^k$ and the policies inside batches $\phi^{k,h}$. The optimism is achieved via variance-based bonuses of the form
\begin{minipage}{\linewidth}
{\small\begin{align*}
     b^k_{h}(s,B) = \tilde{\Theta}\br*{\sqrt{\frac{\widehat{\VAR}_h^{k}(s,B;\bar{V}^k_{h+B}) }{n^{k}_{h}(s,B)}} + \frac{1}{n^{k}_{h}(s,B)}} 
\end{align*}}
\end{minipage}
for any $h\in[H],s\in\Scal$ and $B\in[\ell_h]$ (see \cref{eq: bonus} in the appendix for the precise expression). Then, given the bonuses, the optimistic Q-values are defined as
\begin{align}
    &\bar{Q}^k_h(s,B) = \min\brc*{\widehat{\E}_{h,s,B}^k\brs*{\Qla^*_h(s,B;I,\bar{V}^k_{h+B})} + b^k_{h}(s,B),H-h+1}. \label{eq: optimistic Q}
\end{align}
Notice that we truncate the values to ensure they stay in the domain $[0,H-h+1]$. 
The optimistic value is then defined as the maximum over all potential batching horizons $\bar{V}^k_h(s) = \max_{B\in[\ell_h]} \brc*{\bar{Q}^k_h(s,B)}$.

\begin{algorithm}[t]
\caption{Adaptive Lookahead Upper Confidence Bound (AL-UCB)} \label{alg: AL-UCB}
\begin{algorithmic}[1]
\For{$k=1,2,...$}
    \State  Initialize $\bar{V}^k_{H+1}(s)=0$ \textcolor{gray}{\# Planning } 
    \For{$h=H,H-1,..,1$ and $s\in\Scal$}   
        \State Calculate $\bar{Q}^k_h(s,B)$ via \cref{eq: optimistic Q} for all $B\in[\ell_h]$ and set $\bar{V}^k_h(s) = \max_{B\in[\ell_h]} \brc*{\bar{Q}^k_h(s,B)}$ 

    \EndFor
    \State Start from $h=1$ in state $s_1^k$  \textcolor{gray}{\# Acting }
    \While{$h\leq H$}
        \State Pick $B_h^k\in\argmax_{B\in[\ell_h]} \brc*{\bar{Q}^k_h(s_h^k,B)}$; observe $B_h^k$-step lookahead $I_h^k$ from state $s_h^k$ 
        \State Play for $B_h^k$ steps a policy $\phi^{k,h}$ as specified in \cref{eq: phi^k} and update $h\gets h+B_h^k$
    \EndWhile
    \State Update the empirical estimators and counts for all $s_h^k$ in which a new batch started and 
    \Statex \hspace{.5cm}  the respective batching horizons
\EndFor
\end{algorithmic}
\end{algorithm}

Having calculated the values, the algorithm repeatedly selects batching horizons that maximize the optimistic Q-values $B_h^k\in\argmax_{B\in[\ell_h]} \brc*{\bar{Q}^k_h(s,B)}$, and then, upon observing $I_h^k$, executes the greedy Markov policy that maximizes the reward given the lookahead and the optimistic value at the end of the lookahead 
\begin{align}
    &\phi^{k,h}\in\argmax_{\phi\in\Pi_{det}}\brc*{\sum_{t=h}^{h+B_h^k-1}\Rla_{t\vert h}(s_h^k,\phi,I_h^k)
    +\bar{V}^k_{h+B_h^k}(\sla_{h+B_h^k\vert h}(s_h^k,\phi,I_h^k))}. \label{eq: phi^k}
\end{align}
We remark that our discussion on efficient calculations in \Cref{section:ABPs} also applies here: one could use quantities similar to $\Jla,\Sla$ to efficiently find $\phi^{k,h}$ (see \Cref{appendix:computational complexity} for further details). We also note that while the bonuses affect the optimistic values $\bar{Q}^k_h(s,B),\bar{V}^k_{h+B_h^k}$, they do not directly appear in the policy $\phi^{k,h}$: this is because the lookahead information removes all uncertainty inside the batch. In other words, inside a batch, the algorithm is completely greedy w.r.t. the lookahead information and optimistic values.  
The full description of AL-UCB can be found in \Cref{alg: AL-UCB}. Our algorithm enjoys the following regret guarantees:
\begin{restatable}{theorem-rst}{regretALUCB}
\label{theorem: regret AL-UCB}
    For any $\delta\in(0,1)$, w.p. at least $1-\delta$ the regret of AL-UCB is upper bounded by 
            $\Regret(K) =\Ocal\br*{\sqrt{\!H^3SK\ell}\ln\frac{SH\ell K}{\delta} + H^3S^2\ell \ln^2\frac{SH\ell K}{\delta}}$.
\end{restatable}
One major challenge that also appeared in the one-step lookahead \citep{merlis2024reinforcement} is that lookahead exponentially increases the effective state-space. Specifically, for some empirical transition estimate $\hat{P}_h^k$, previous works often require bounding a (lower-order) term of the form $\br*{\hat{P}_h^k-P_h}^T\br*{\bar{V}^k_{h+1}-V^*_{h+1}}$, and one standard approach yields a bound of the form (see, e.g., \citep{efroni2021confidence}):
\begin{align}
    \sum_{s'}\hat{P}_h^k(s'\vert s,a)&\br*{\bar{V}_{h+1}^k(s')-V^*_{h+1}(s')} \nonumber\\
    &\lesssim \br*{1+\frac{1}{H}}\sum_{s'}P_h(s'\vert s,a)\br*{\bar{V}_{h+1}^k(s')-V^*_{h+1}(s')} + \Olog\br*{\frac{SH}{n_h^k(s,a)}}.
    \label{eq:problematic ineq}
\end{align}
However, this bound depends on $S$, and when generalizing this to incorporate lookahead by augmenting it to the state space, the relevant distribution is the lookahead distribution $\Ical_{h,B}(s)$. Thus, instead of $S$, one would get the support of this distribution, which has an exponential nature due to its transition elements and continuous components due to reward lookahead. With one-step reward lookahead, \citet{merlis2024reinforcement} tackled this issue by using uniform concentration on the value, suffering a factor of $\sqrt{A}$ (which did not affect the final rate due to the transition estimation terms). In our case, the `effective action' would be to choose the state at the end of the lookahead, which would lead to a regret degradation of $\sqrt{S}$. On the other hand, with transition lookahead, \citet{merlis2024reinforcement} reformulated the probability space using a special list structure of the optimal policy, so that it becomes polynomial. This structure does not hold in our case (even with one-step reward and transition lookahead).

Our way of solving this is to prove uniform concentration on the problematic error term using Freedman's inequality, alongside covering arguments that only lead to a polynomial degradation in lower-order terms. Specifically, we show that w.h.p.,
\begin{align*}
    \widehat{\E}_{h,s,B}^k&\brs*{\Qla^*_h(s,B;I,\bar{V}^k_{h+B})-\Qla^*_h(s,B;I,V^*_{h+B})}\\
    &\leq \!\br*{\!1+\frac{1}{2H}\!}\E_{I\sim\Ical}\brs*{\Qla^*_h(s,B;I,\bar{V}^k_{h+B})-\Qla^*_h(s,B;I,V^*_{h+B})} 
    + \Ocal\br*{\frac{H^2SL^k_{\delta}}{n^{k}_h(s,B)}},
\end{align*}
which is the equivalent of \cref{eq:problematic ineq} in our setting; notice that this inequality depends on the size of the original state space $S$, and not that of an augmented space. The aforementioned event, with the exact constants, is formally defined in \Cref{appendix: concentration event} as $E^{Err}(k)$, and the general concentration claim used to bound its probability is proved in \Cref{lemma: uniform max concentration}. This result allows us to propagate errors from the beginning of a batch at step $h$ to the beginning of the next batch at step $h+B$, and the multiplicative degradation due to this error propagation is at most $\br*{1+\frac{1}{2H}}^H=\Ocal(1)$. The full proof includes other technical challenges related to adapting to the batch structure and can be found in \Cref{appendix: regret}. We remark that our approach can be applied to the one-step transition lookahead to greatly simplify the proof in \citep{merlis2024reinforcement}, obviating the need for the list-based manipulations.

\textbf{Comparison to one-step lookahead.} The minimax optimal regret rates with either one-step reward or transition lookahead information are $\Olog\br*{\sqrt{H^3SAK}}$ and $\Olog\br*{\sqrt{H^2SK}(\sqrt{H}+\sqrt{A})}$, respectively \citep{merlis2024reinforcement}. Limiting our results to $\ell=1$ (where ABPs are equivalent to a general lookahead policy), we show an improved regret rate of $\Olog\br*{\sqrt{H^3SK}}$, removing all dependence on the size of the action space. This is not surprising: in any visited state, we observe all possible actions before acting. Nonetheless, as previously explained, it is not clear whether the proof in \citep{merlis2024reinforcement} can be adapted to achieve this rate. As for lower bounds, following a similar reduction to a bandit-like setting via trees \citep{domingues2021episodic}, this rate is tight. Specifically, taking the example in \citep{domingues2021episodic} and removing all actions but one in all the leaves would remove all benefits of one-step lookahead information, so that lookahead algorithms get the same information as vanilla RL algorithms. Removing these actions reduces the standard RL lower bound by a factor of $\sqrt{A}$: to a value of $\Omega\br*{\sqrt{H^3SK}}$.

\textbf{Tightness for $\ell>1$.} For multi-step lookahead, it is unclear whether the bounds are tight in the maximal lookahead $\ell$. One simple way to reduce the feedback from $\ell$-step lookahead to standard RL is to add delays: force the environment to progress only in steps where  $h=n\ell$ and freeze it (with no reward) otherwise. This can be achieved by duplicating each state $\ell$ times. Equivalently, we can transform an $\ell$-step lookahead environment with $\ell$ delays to a one-step lookahead environment with $S/\ell$ states. Based on this principle (and specifically, duplicating the leaves in the standard tree-like construction in \citep{domingues2021episodic}), the immediate candidate for a lower bound is $\Omega\br*{\sqrt{{H^3SK}/{\ell}}}$: different by a factor of $\ell$ from our upper bound. For most applications, it is reasonable to treat $\ell$ as an absolute constant (and, in particular, $\ell\ll H$), and so this gap is not major. Moreover, part of this gap is very likely due to our feedback model: the fact that at batch starts, we observe $B_h$-step lookahead and not $\ell_h$-step lookahead. In fact, were we to observe $\ell_h$-step lookahead at batch starting points, we could update the observations simultaneously for any $B\in[\ell_h]$, and all dependencies in $n^{k}_h(s,B)$ would be replaced by $n^{k}_h(s)$. In particular, this would lead to a regret upper bound of $\Olog\br*{\sqrt{H^3SK}}$ for any $\ell$ (see \Cref{appendix: reacher feedback} for more details). Similarly, looking at fixed batching policies and forcing $B_h$ to be independent of $s$ (i.e., to be any fixed function of $h$), we would also get a predefined batching horizon for any timestep. Therefore, we could maintain a single count $n^{k}_h(s)$. Then, following the same analysis would immediately lead to a regret of $\Olog\br*{\sqrt{H^3SK}}$ compared to the best policy with the same fixed batching profile. 

We hypothesize that for large enough $\ell$, the regret should diminish, since the feedback becomes increasingly similar to full information, but leave this for future work.

\section{Conclusions and Future Work}
In this work, we tackled the problem of planning and learning with multiple steps of lookahead information. We demonstrated the failure of fixed batching policies and the exponential suboptimality of common value calculations for MPC schemes, and proposed focusing on adaptive batching policies: both improving fixed batching policies while retaining the ability to compute the optimal strategy. We derived the optimal Bellman equations for ABPs and showed how to learn the optimal ABP when interacting with unknown tabular episodic environments. Our regret bounds are near-optimal up to factors that depend on the lookahead horizon $\ell$, which is typically treated as a small constant.

More broadly, the aim of the work is to expand the family of tractable strategies with multi-step lookahead and show how to conduct learning in this scenario. However, there is still much to be done to fully characterize the tractable regimes in this setting. In particular, though it has been proven that exactly computing the optimal lookahead policy is hard \citep{pla2026hardness} and there are exponential-in-$\ell$ approximation schemes \citep{pla2026near} for similar settings, it remains open whether there exist polynomial planning schemes with provable approximation guarantees. Establishing such results falls beyond the scope of our work.  
Alternatively, after showing the failure of backwards-induced MPC schemes, it would be interesting to study whether there are tractable global approaches for computing better values. We conjecture that calculating the optimal values is NP-hard, since they may depend delicately on the entire structure of the environment.

Other directions involve closing the regret gap or studying algorithmic alternatives that may be better suited for practical deployment. For example, instead of solving a planning problem inside each batch, one could use rollouts; the model-based optimistic value calculations could be replaced by model-free approaches; and so forth. After pursuing these directions, one could adapt the resulting algorithms to deep RL and apply them to some of the many applications in which lookahead information is available. 
Finally, our work assumes complete and perfect $\ell$-step information. In practice, information may be distorted (especially further into the future), and some future information might be missing. We believe our analysis can be naturally combined with that of \citep{lu2025reinforcement} to incorporate imperfect future predictions, but leave this for future study.

\section*{Acknowledgments}
This research was supported by Israel Science Foundation research grant (ISF's No. 4118/25) and the
Maimonides Fund's Future Scientists Center.

\bibliographystyle{plainnat}
\bibliography{references}
\clearpage

\addcontentsline{toc}{section}{Appendix}
\appendix
\setlength{\leftmargini}{2em}

\part{}


{\hypersetup{linkcolor=black}
\parttoc}

\clearpage
\section{Data Generation Process}
\label{appendix: data generation}
We first specify in detail how lookahead information is generated. While there are other equivalent ways to generate it, this generation process will simplify the analysis and, in particular, will trivially imply equivalence between obtaining lookahead information and interacting with an augmented environment described in \Cref{appendix: augmented MDP}.

The key observation on the adaptive batching policy space is that the lookahead information $I_{h,B_h}(s_h)$ is only observed at the beginning of a batch, after choosing a new batch length $B_h$, and even if obtained, additional lookahead information is not used in the middle of the batch. Therefore, for ABPs, we can assume w.l.o.g. that the lookahead information $I_{h,B_h}(s_h)$  is independently generated only \emph{after reaching state $s_h$ and picking a lookahead range $B_h$}.

Formally, for any $h\in[H]$ and $B\in\brc*{1,\dots,\ell_h}$, denote by $\bs'_{h,B}\in \Scal^{SA\ell }$, the $B$-step transition lookahead information starting from step $h$. Specifically, for any $h\leq t\leq h+B-1$, any reachable $s\in\Scal$ and any $a\in\Acal$, upon playing $a$ at step $t$ and state $s$, the agent would transition to $s_{t+1}=\bs_{h,B}'(s,a,t)$. To keep the dimension requirements, for $t\in\brc*{h+B,\dots,h+\ell-1}$, we pad with $\bs_{h,B}'(s,a,t)=\emptyset$ (for some arbitrary unreachable state $\emptyset$), and also define $\bs'_{h,0}=\emptyset^{\otimes SA\ell }$. We similarly denote reward lookahead by $\bR_{h,B}\in \R^{SA\ell }$, so that upon playing action $a$ at step $t\in\brc*{h,\dots,h+B-1}$ and state $s$, the agent would collect a reward $R_{t}=\bR_{h,B}(s,a,t)$, and set $\bR_{h,B}(s,a,t)=0$ for $t>h+B-1$. We again denote $\bR_{h,0}=0^{\otimes SA\ell }$. 
    
    \paragraph{Lookahead Sampling Procedure.} We further specify the lookahead sampling procedure to ensure equivalence between  $I_{h,B_h}(s_h)$ and the observations $\bR_{h,B}, \bs'_{h,B}$ sampled from $s_h=s$. The procedure consists of two stages.
    \begin{enumerate}
        \item \emph{Sampling}. Given $h,B$ for all $t\in\brc*{h,\dots,h+B-1}$, sample  $\bR_{h,B}(s,a,t)$ and $\bs_{h,B}'(s,a,t)$ for all $s\in\Scal,a\in\Acal$. The samples can be correlated between states and actions (if the reward and/or transitions are correlated), but are sampled independently between different values of $t$. This sampling procedure gives lookahead information $\bR_{h,B},\bs_{h,B}'$ also on unreachable states.
        \item \emph{Pruning}. To make the information equivalent to the lookahead information $I_{h,B_h}(s_h)$, we need to remove any information on unreachable states; we do it recursively. First, for $t=h$, the only reachable state is $s_h$. Therefore, for any $s\ne s_h$, we set $\bR_{h,B}(s,a,h)=0$ and $\bs_{h,B}'(s,a,h)=\emptyset$. The set of reachable states in step $h+1$ is 
        $$\Sla_{h,1}(s_h,\bs_{h,B}') = \brc*{s'\in\Scal: \exists a\in\Acal\enspace s.t.\enspace \bs_{h,B}'(s_h,a,h)=s'}.$$
        Then, for any step $h+b$ (with $b\in[B-1]$), if $s\notin \Sla_{h,b}(s_h,\bs_{h,B}')$, we set $\bR_{h,B}(s,a,h+b)=0$ and $\bs_{h,B}'(s,a,h+b)=\emptyset$ for all $a\in\Acal$. The reachable states in step $h+b+1$ will now be         
        $$\Sla_{h,b+1}(s_h,\bs_{h,B}') = \brc*{s'\in\Scal: \exists s\in\Sla_{h,b}(s_h,\bs_{h,B}'),a\in\Acal\enspace s.t.\enspace \bs_{h,B}'(s,a,h+b)=s'}.$$
    \end{enumerate}
    We write $\bs'_{h,B}\sim \Pcal_{h}^B(s)$ and  $\bR_{h,B}\sim \Rcal_{h}^B(s)$ to represent the independent sampling from the aforementioned distribution starting at $s_h=s$. By construction, the information left after pruning contains exactly all rewards and transitions along all possible trajectories/policies: the same information as in $I_{h,B_h}(s_h)$. With slight abuse of notation, given a deterministic policy $\phi$, we again use the notations $\Rla_{h+b\vert h}(s,\phi,\bR,\bs')$ and $\sla_{h+b\vert h}(s,\phi,\bR,\bs')$ to denote the rewards and states reached from state $s$ by playing policy $\phi$ on the lookahead $\bR,\bs'$. When the batching horizon $B$ is clear from the context, we sometimes write $\Ical_h(s)$.
    \begin{remark}
        The pruning step is not strictly needed for planning or when analyzing the value of policies: the pruned states can never be visited, so they cannot affect the optimal policy. Moreover, every policy that uses their information can be converted to a policy of the same value that does not use this information by sampling the information on the pruned states directly from the MDP distribution and acting as if this is the realization. Therefore, we prune the rewards mainly to make the equivalence between the lookahead information and the augmented MDP clear. On the other hand, in learning situations, observing information in unreachable states would effectively convert the bandit feedback to full information, thus making the problem much easier; we do not analyze this feedback in our paper.
    \end{remark}
\clearpage

\section{Description of an Equivalent Augmented MDP}
\label{appendix: augmented MDP}
    Having reformulated the lookahead information in a vectorized form in \Cref{appendix: data generation}, we now present an equivalent environment $\Mcal^{\Brm}$, in which this information is augmented into the state. This augmentation renders the environment Markov and enables us to analyze it using standard planning tools. We call $\Mcal$ the original MDP and $\MB$ the batched MDP. The different components of the batched MDP $\Mcal^{\Brm}$ are defined as follows:
    \begin{itemize}
        \item \textbf{State Space:} the state is represented by a tuple $(s,\bR,\bs',h, B)$, where $s\in\Scal$ is the current state of the agent, $\bR\in\R^{SA\ell }$ and $\bs'\in\Scal^{SA\ell }$ encode the lookahead information, $h\in[H]$ is the time step at the original MDP, and $B\in\brc*{0,\dots, \ell_h}$ is the remaining steps until the end of a batch. 
        \item \textbf{Action Space:} the agent picks an action $(a,B')\in \Acal\times [\ell_h]$, where the first component represents a real interaction and the second represents a batching horizon. 
        \item \textbf{Initial state:} Assuming the original MDP starts at state $s$, the batched environment starts at state 
            \begin{align*}
                (s_{1},\bR_{1},\bs'_{1},h_{1}, B_{1}) \gets \br*{s,0^{\otimes SA\ell },\emptyset^{\otimes SA\ell },1, 0}.
            \end{align*}
        \item \textbf{Interaction Protocol:} We denote the time index in $\Mcal^{\Brm}$ by $t$, which will advance at a slower pace than the time index of the original MDP $h$. The interaction protocol is divided into the following scenarios:
        \begin{itemize}
            \item If $h_t\leq H$ and $B_t=0$, we are about to start a new batch at timestep $h_t$. The agent picks the batch size using $B'_t$ (and the action $a_t$ has no effect). The reward is zero, and the agent transitions to the state 
            \begin{align*}
                (s_{t+1},\bR_{t+1},\bs'_{t+1},h_{t+1}, B_{t+1}) \gets (s_t,\bR_{h_t,B'_t},\bs'_{h_t,B'_t},h_t, B'_t),
            \end{align*}
            where both $\bR_{h_t,B'_t}\sim \Rcal_{h_t}^{B_t'}(s_t)$ and $\bs'_{h_t,B'_t}\sim\Pcal_{h_t}^{B_t'}(s_t)$ are lookahead information sampled according to the aforementioned sampling procedure. In other words, the state does not change, but the lookahead information is revealed, and the lookahead range is registered in the state.
            \item If $h_t\leq H$ and $B_t>1$, we are at the middle of a batch. Then, upon playing $a_t$, the agent will collect a reward $R_t=\bR_t(s_t,a_t,h_t)$ and transition to $s_{t+1}=\bs'_t(s_t,a_t,h_t)$ with one less step remaining in the batch:
            \begin{align*}
                (s_{t+1},\bR_{t+1},\bs'_{t+1},h_{t+1}, B_{t+1}) \gets (\bs'_t(s_t,a_t,h_t),\bR_t,\bs'_t,h_t+1, B_t-1).
            \end{align*}
            We remark that the structure of $\bs'_{t}$ implicitly encodes the initial lookahead range (by checking what is the largest time index for which it has a nonempty element). Therefore, alongside $h_t$ and $B_t$ (which encodes how many steps are left in a batch), the state encodes when the batch has started and at which step of the batch the agent is found.
            \item If $h_t\leq H$ and $B_t=1$, we are at the last step of a batch. Then, upon playing $a_t$, the agent will collect a reward $R_t=\bR_t(s_t,a_t,h_t)$, transition to $s_{t+1}=\bs'_t(s_t,a_t,h_t)$ and reset the lookahead information:
            \begin{align*}
            (s_{t+1},\bR_{t+1},\bs'_{t+1},h_{t+1}, B_{t+1}) \gets (\bs'_t(s_t,a_t,h_t),0^{\otimes SA\ell },\emptyset^{\otimes SA\ell},h_t+1, 0),
            \end{align*}
            In other words, at the end of every batch, the state will always be of the form $(s,0^{\otimes SA\ell },\emptyset^{\otimes SA\ell },h, 0)$.
            \item If $h_t>H$, we completed $H$ steps at the original MDP, and so w.l.o.g., we assume that the batched MDP transitions to a non-rewarding absorbing state.
            \item Notice that an action $B'_t$ only has impact when starting a new batch, and it is directly registered to the state. Therefore, with a slight abuse of notation, we use $B_t$ to indicate both the action and the registered state component.
        \end{itemize}
    \end{itemize}
    The horizon of the batched MDP is set to $2H$ to ensure that it passes the $H^{th}$ step in the original MDP. This is since for every step in which $B_t=0$ (and $h_t$ does not change), there is at least one following state with $B_t>0$ where $h_t$ progresses: it holds that $h_t\ge\ceil*{t/2}$. In particular, $h_{2H+1}>H$ for all policies, and if $B_t>0$ at some step in which $h_t\leq H$, then $t+B_t\leq 2H+1$. Similarly, at the beginning of batches (before the lookahead range is picked), necessarily $t\leq 2h-1$.
    
    At any point, a policy in the batched MDP either decides on a lookahead range based on the timestep $h$, state $s_t$, and the augmented time index $t$, or decides on an action based on $s,h,t$ and the lookahead information at the beginning of a batch. Thus, the only difference between such policies and ABPs in $\Mcal$ is the index $t$. On a closer look, the only case where $t$ advances but $h$ does not is at the beginning of a batch: the difference $t-h$ is the number of batches started up to step $h$ (up to $1$ if we are about to start a new batch). In the case where the policy in $\MB$ is Markov and deterministic, the state $s_t$ can be fully deduced from the state at the beginning of the batch and the lookahead information, similarly to ABPs. 

    In other words, a deterministic Markov policy in $\Mcal^{\Brm}$ is fully equivalent to policies we term \emph{extended adaptive batching policies}: policies that can select lookahead ranges and/or policies also depending on the number of past batches. To better illustrate the equivalence, we describe the interaction protocol of extended ABPs in $\Mcal$ and deterministic Markov policies in $\Mcal^{\Brm}$ in \Cref{alg:extended adaptive batching} and \Cref{alg:augmented policy}, respectively. Consequently, the values of the two MDPs coincide. 

    \begin{figure}[t]
    \begin{minipage}{0.52\linewidth}
    \begin{algorithm}[H]
    \caption{Extended Adaptive Batching Policies \vphantom{$\MB$}}\label{alg:extended adaptive batching}
    \begin{algorithmic}[1]
    \State  \textbf{Input}: Initial state $s_1$
    \State \textbf{Initialize}: $h=1, n=0$ \textcolor{gray}{$n=$ \#  past batches }
    \While{$h\leq H$}
        \State Pick {\small$B_h=B_h(s_h,n)\in \brc*{1,\dots,\ell_h}$} 
        \State Update $n \leftarrow n+1$
        \State Observe $B_h$-step lookahead $I$
        \For{$b=0,\dots,B_h-1$} 
        \State Play $a_{h+b} = \phi_{h+b}^{h,n}(s_h,I)$
        \State Collect reward $R_{h+b}$ and transition
        \Statex \hspace{1cm}  to $s_{h+b+1}$ in accordance with the 
        \Statex \hspace{1cm}  lookahead $I$
        \EndFor
        \State $h \leftarrow h+B_h$
    \EndWhile
    \end{algorithmic}
    \end{algorithm}
    \end{minipage}
    \hfill
    \begin{minipage}{0.4975\linewidth}
    \begin{algorithm}[H]
    \caption{Deterministic Markov Policy in $\Mcal^{\Brm}$}\label{alg:augmented policy}
    \begin{algorithmic}[1]
    \State  \textbf{Input}: Initial state $s_1$
    \State \textbf{Initialize}: $h=t=1$
    \While{$h\leq H$}
        \State Pick {\small$B_t=B_t(s_h,h)\in \brc*{1,\dots,\ell_h}$} 
        \State Update $t \leftarrow t+1$
        \State Observe $B_t$-step lookahead $\bR,\bs'$
        \For{$b=0,\dots,B_t-1$} 
        \State Play $a_{t+b} = \phi_{h+b}^{h,t}(s_h,\bR,\bs')$
        \State Collect reward $R_{h+b}$ and transition 
        \Statex \hspace{1cm}  to $s_{h+b+1}$ in accordance with the 
        \Statex \hspace{1cm}  lookahead $\bR,\bs'$
        \EndFor
        \State $h \leftarrow h+B_t$, $t \leftarrow t+B_t$
    \EndWhile
    \end{algorithmic}
    \end{algorithm}
    \end{minipage}
    \end{figure}

    \paragraph{Relation to previous augmented constructions. } We remark that augmenting the state space with the lookahead observations in tabular MDPs was presented in similar contexts first in one-step lookahead \citep{merlis2024reinforcement}, and then also in multi-step lookahead \citep{lu2025reinforcement,pla2026hardness}. Yet, our construction directly embeds the structure of the ABPs into the environment and requires additional care due to the extra control over the lookahead range. In particular, we must account for the fact that the timesteps at which batches start can be random, and for the implicit dependence on the number of previous batches.
    
    \paragraph{Values.} With some abuse of notation, we use $\pi$ to denote an extended batching policy in $\Mcal$ and its corresponding policy in $\Mcal^{\Brm}$. We let $V^\pi_t(s,\bR,\bs',h,B\vert \MB)$ be the value of the policy in $\MB$. Specifically, when a batch ends, we denote the value by 
    $$V^\pi_t(s,0^{\otimes SA\ell },\emptyset^{\otimes SA\ell},h, 0\vert \MB) \triangleq V^\pi_t(s,h\vert \MB) = \E_{\MB}\brs*{\sum_{i=t}^{2H}R_i\vert s_t=s,h_t=h,B_t=0,\pi}
    $$
    For an ABP, where the policy does not depend on the number of previous batches, we keep the notation from the main paper and denote the equivalent value in $\Mcal$ upon starting a new batch at step $h$ and state $s$ by
    \begin{align*}
        V^\pi_h(s) = \E_{\Mcal}\brs*{\sum_{i=h}^{H}R_i\vert s_h=s, E^n_h, \pi}.
    \end{align*}
    Recall that $E^n_h$ is the event that the ABP $\pi$ starts a new batch at timestep $h$, and that, by the data generation process, the expectation freshly samples all future information from the model independently between timesteps. 
    In particular, for the described data generation process, both environments coincide at the beginning of a batch under any ABP $\pi$, and it holds that 
    \begin{align}
        & V^\pi_h(s) = V^\pi_t(s,h\vert \MB) , \qquad \forall s\in\Scal,h\in[H], t\in[2h-1]
    \end{align}
    where the constraint to $t\in[2h-1]$ guarantees that there are enough remaining steps for the episode to end.
    
    Finally, we use $\pi^*$ to denote an optimal policy in $\MB$. Importantly, since any ABP in $\Mcal$ corresponds to a policy in $\MB$, the optimal value in $\MB$ upper bounds the value of the optimal ABP; thus, by proving that the optimal policy in $\MB$ corresponds to an ABP, we will be able to derive the optimal ABP and its value. We will denote this value by $V^*_h(s)$ (and will show that it indeed does not depend on $t-h$).

    \begin{lemma}
    \label{lemma: value in augmented MDP}
    \hspace{20cm}
        \begin{enumerate}
            \item Let $\pi$ be a deterministic Markov Policy interacting in $\Mcal^{\Brm}$, as specified in \Cref{alg:augmented policy}. Then, the value of the policy follows the Bellman equations
        {\small\begin{align*}
            V^\pi_t(s,h\vert \MB) = \E_{\substack{\bR\sim \Rcal_{h}^{B_t}(s),\\ \bs'\sim \Pcal_{h}^{B_t}(s)}}\brs*{\sum_{b=0}^{B_t-1}\Rla_{h+b\vert h}(s,\phi^{h,t},\bR,\bs') + V^{\pi}_{t+B_t+1}\br*{\sla_{h+B_t\vert h}(s,\phi^{h,t},\bR,\bs'),h+B_t}\vert \pi,s_t=s}
        \end{align*}}
        for any $h\leq H$, $t\leq 2h-1$ and $s\in\Scal$, with the initialization $V^\pi_t(s,h\vert \MB) =0$ when $h>H$ or $t>2H$. 
        \item For any $h\leq H$, $t\leq 2h-1$ and $s\in\Scal$, the optimal value is independent of $t$; we denote it by $V^{*}_h(s\vert \MB)$, and it can be calculated via the optimal Bellman equations
        {\small\begin{align*}
            V^{*}_h(s\vert \MB) = \max_{B\in[\ell_h]}\E_{\substack{\bR\sim \Rcal_{h}^{B}(s),\\ \bs'\sim \Pcal_{h}^{B}(s)}}\brs*{\max_{\phi\in\Pi_{det}}\brc*{\sum_{b=0}^{B-1}\Rla_{h+b\vert h}(s,\phi,\bR,\bs') + V^{*}_{h+B}\br*{\sla_{h+B\vert h}(s,\phi,\bR,\bs')}}\vert s_t=s}
        \end{align*}}
        \end{enumerate}
    \end{lemma}
\begin{proof}
    Before starting the proof, we again remark that the only case in which $t$ advances in $\MB$ and $h$ does not is a step in which a lookahead range is chosen. After each such step, a policy will be executed for at least $1$ step, advancing both $t$ and $h$. In particular, starting from $h=t=1$, all consecutive states must have that $t\leq 2h$, and at a point before picking the lookahead range of a new batch, it holds that $t\leq 2h-1$. Thus, without loss of generality, we assume throughout the proof that this always holds, which also implies that the policy necessarily plays an action at $h=H$ before the interaction ends, regardless of how it chooses the following lookahead ranges. 

    Recall that we denote the values at the batched MDPs of a policy $\pi$ by $V^{\pi}_t(s,\bR,\bs',h, B\vert\MB)$. To ease notations, we also implicitly assume w.l.o.g. that the lookahead information $\bR,\bs'$ is always of a valid form (can be realized in the interaction with $\MB$). We now use the dynamic programming equations for finite-horizon MDPs \citep{puterman2014markov} to compute the values of $\pi$ and the optimal values. 

    \paragraph{Initialization.} We initialize the value to be zero for any $t>2H$ (after the interaction ends). Moreover, by the dynamics model of $\MB$, the reward is zero starting from any state with $h>H$. Therefore, it holds that 
    \begin{align*}
        & V^\pi_{t}(s,\bR,\bs',h, B\vert\MB) = 0,  & \forall h>H,\; \mathrm{ or, }\; t>2H
    \end{align*}

    \paragraph{Bellman equations: arbitrary policies. } Let $h\leq H$. We follow the standard dynamic programming procedure \citep{puterman2014markov}, relying on the dynamics of $\MB$. In the middle of a batch ($B_t\ge1$), the dynamics are deterministic (due to the already-observed lookahead information), so the equations take the form
    \begin{align*}
        &V^\pi_{t}(s_t,\bR_t,\bs_t',h, B_t\vert\MB) = \bR_t(s_t,a_t,h) + V^\pi_{t+1}(\bs_t(s_t,a_t,h),\bR_{t},\bs_{t}',h+1, B_t-1\vert\MB), & B_t>1, \\
        &V^\pi_{t}(s_t,\bR_t,\bs_t',h, B_t\vert\MB) = \bR_t(s_t,a_t,h) + V^\pi_{t+1}(\bs_t(s_t,a_t,h),0^{\otimes SA\ell },\emptyset^{\otimes SA\ell},h+1, 0\vert\MB), & B_t=1.
    \end{align*}
    Notice that since the policy is deterministic, so is the sequence of state-actions inside the batch. Next, we apply this recursion $B_t$ times until the end of the batch. Using $\phi^{h,t}$ to describe the deterministic policy induced by $\pi$ in the middle of a batch, we get 
    \begin{align}
        V^\pi_{t}(s_t,\bR_t,\bs_t',h, B_t\vert\MB) 
        &= \sum_{b=0}^{B_t-1}\bR_{t+b}(s_{t+b},a_{t+b},h+b) + V^\pi_{t+B_t}(s_{t+B_t},0^{\otimes SA\ell },\emptyset^{\otimes SA\ell},h+B_t, 0\vert\MB) \nonumber\\
        & =  \sum_{b=0}^{B_t-1}\Rla_{h+b\vert h}(s_t,\phi^{h,t},\bR_t,\bs'_t) + V^\pi_{t+B_t}(\sla_{h+B_t\vert h}(s_t,\phi^{h,t},\bR_t,\bs'_t),h+B_t\vert\MB). \label{eq: dp pi inside batch}
    \end{align}
    We now turn our focus to states where a batch is about to begin. In this case, the dynamics dictate that the state $s_t$ remains the same without gaining a reward, $B_t$-step lookahead information is sampled, and the lookahead range $B_t$ is registered in the state. Thus, the Bellman equation is
    \begin{align*}
        V^\pi_{t}(s,h\vert\MB) 
        &= V^\pi_{t}(s,0^{\otimes SA\ell },\emptyset^{\otimes SA\ell},h, 0\vert\MB)
        = \E_{\substack{\bR\sim \Rcal_{h}^{B_t}(s),\\ \bs'\sim \Pcal_{h}^{B_t}(s)}}\brs*{V^\pi_{t+1}(s,\bR,\bs',h, B_t\vert\MB)}\\
        & =  \E_{\substack{\bR\sim \Rcal_{h}^{B_t}(s),\\ \bs'\sim \Pcal_{h}^{B_t}(s)}}\brs*{\sum_{b=0}^{B_t-1}\Rla_{h+b\vert h}(s,\phi^{h,t},\bR,\bs') + V^\pi_{t+B_t+1}(\sla_{h+B_t\vert h}(s,\phi^{h,t},\bR,\bs'),h+B_t\vert\MB)},
    \end{align*}
    where the last equality is by \cref{eq: dp pi inside batch}. This concludes the first part of the claim.

\paragraph{Bellman equations: optimal value. } We repeat the same process, but this time apply the optimal Bellman operator. Let $h\leq H$. In the middle of a batch ($B\ge1$), the optimal Bellman equations are
    \begin{align*}
        &V^{*}_{t}(s_t,\bR_t,\bs_t',h, B_t\vert\MB) = \max_{a_t\in\Acal}\brc*{\bR_t(s_t,a_t,h) + V^*_{t+1}(\bs_t'(s_t,a_t,h),\bR_{t},\bs_{t}',h+1, B_t-1\vert\MB)}, & B_t>1,\\
        &V^{*}_{t}(s_t,\bR_t,\bs_t',h, B_t\vert\MB) = \max_{a_t\in\Acal}\brc*{\bR_t(s_t,a_t,h) + V^*_{t+1}(\bs_t'(s_t,a_t,h),0^{\otimes SA\ell },\emptyset^{\otimes SA\ell},h+1, 0\vert\MB)},& B_t=1.
    \end{align*}
    Next, we again apply this recursion $B_t$ times until the end of the batch, which yields
    \begin{align*}
        V^{*}_{t}(s_t,\bR_t,\bs_t',h, B_t\vert\MB) 
        &= \max_{a_t,\dots,a_{t+B-1}\in\Acal}\brc*{
        \sum_{b=0}^{B_t-1}\bR_{t+b}(s_{t+b},a_{t+b},h+b) + V^*_{t+B_t}(s_{t+B_t},h+B_t\vert\MB) },
    \end{align*}
    where the states implicitly depend on the sequence of actions and the lookahead information. Since the dynamics are deterministic given lookahead information, there is a full equivalence between an arbitrary sequence of actions and a deterministic Markov policy that effectively executes this sequence. Therefore, we can write the previous relation using $\phi,\Rla$ and $\sla$ as
    \begin{align}
        V^{*}_{t}(s_t,\bR_t,\bs_t',h, B_t\vert\MB) 
        &= \max_{\phi\in\Pi_{det}}\brc*{\sum_{b=0}^{B_t-1}\Rla_{h+b\vert h}(s_t,\phi,\bR_t,\bs'_t) + V^*_{t+B_t}(\sla_{h+B_t\vert h}(s_t,\phi,\bR_t,\bs'_t),h+B_t\vert\MB)}.
        \label{eq: dp opt inside batch}
    \end{align}
    Finally, as in the evaluation of fixed policies, we turn to states in which a new batch starts. Since the action $a$ in these states is irrelevant and the valid lookahead range is $B\in\brs*{\ell_h}$, the optimal Bellman equation can be written as 
    \begin{align*}
        V^*_{t}(s,h\vert\MB) 
        &= V^*_{t}(s,0^{\otimes SA\ell },\emptyset^{\otimes SA\ell},h, 0\vert\MB)
        = \max_{B\in\brs*{\ell_h}}\E_{\substack{\bR\sim \Rcal_{h}^{B}(s),\\ \bs'\sim \Pcal_{h}^{B}(s)}}\brs*{V^*_{t+1}(s,\bR,\bs',h, B\vert\MB)}\\
        & =  \max_{B\in\brs*{\ell_h}}\E_{\substack{\bR\sim \Rcal_{h}^{B}(s),\\ \bs'\sim \Pcal_{h}^{B}(s)}}\brs*{\max_{\phi\in\Pi_{det}}\brc*{\sum_{b=0}^{B-1}\Rla_{h+b\vert h}(s,\phi,\bR,\bs') + V^*_{t+B+1}(\sla_{h+B\vert h}(s,\phi,\bR,\bs'),h+B\vert\MB)}},
    \end{align*}
    where the last equality substitutes \cref{eq: dp opt inside batch}. All that is left is to prove that the optimal value only depends on $h$ and not on $t$; we denote this value by $V^*_{h}(s\vert\MB)$. This is the result of a simple induction. First, the claim holds trivially for $h=H+1$ and all $t$, where all values are zero. Now assume that the claim holds for any $h'\in\brc*{h+1,\dots H+1}$,  $s\in\Scal$ and $t\leq 2h'-1$. In particular, for any $t\leq 2h-1$ and $B\in \brs*{\ell_h}$, we have that $h+B \in\brc*{h+1,\dots H+1}$ and 
    $$t+B+1 \leq 2h+B = 2(h+B)-B \leq 2(h+B)-1,$$
    so by the induction assumption, $V^*_{t+B+1}(s,h+B\vert\MB) = V^*_{h+B}(s\vert\MB)$. Thus, by the dynamic programming equation, we have that 
    \begin{align*}
        V^*_{t}(s,h\vert\MB) 
        & =  \max_{B\in\brs*{\ell_h}}\E_{\substack{\bR\sim \Rcal_{h}^{B}(s),\\ \bs'\sim \Pcal_{h}^{B}(s)}}\brs*{\max_{\phi\in\Pi_{det}}\brc*{\sum_{b=0}^{B-1}\Rla_{h+b\vert h}(s,\phi,\bR,\bs') + V^*_{t+B+1}(\sla_{h+B\vert h}(s,\phi,\bR,\bs'),h+B\vert\MB)}}\\
        & = \max_{B\in\brs*{\ell_h}}\E_{\substack{\bR\sim \Rcal_{h}^{B}(s),\\ \bs'\sim \Pcal_{h}^{B}(s)}}\brs*{\max_{\phi\in\Pi_{det}}\brc*{\sum_{b=0}^{B-1}\Rla_{h+b\vert h}(s,\phi,\bR,\bs') + V^*_{h+B}(\sla_{h+B\vert h}(s,\phi,\bR,\bs')\vert\MB)}}.
    \end{align*}
    Notably, the formula has no dependence on $t$ and so $V^*_{t}(s,h\vert\MB) \triangleq V^*_{h}(s\vert\MB)$ is independent of $t$.
\end{proof}

\clearpage
A direct implication of the lemma is a formula to evaluate the values of an ABP at the beginning of batches.
\begin{corollary}
    \label{corollary: ABP value}
    Let $\pi$ be an ABP specified as in \Cref{alg:adaptive batching}. Then, when starting a new batch at step $h$ and state $s$, the value of the policy can be calculated via the Bellman equations
    \begin{align*}
        V^\pi_{h}(s) 
        & =  \E_{I\sim\Ical_{h,B_h}(s)}\brs*{\sum_{b=0}^{B_h-1}\Rla_{h+b\vert h}(s,\phi^{h},I) + V^\pi_{h+B_h}(\sla_{h+B_h\vert h}(s,\phi^{h},I))\vert \pi},
    \end{align*}
    where $V^\pi_{H+1}(s)=0$. Specifically, this value can be written as 
    \begin{align*}
        V^\pi_{h}(s) 
        & =  \E\brs*{\sum_{b=0}^{B_h-1}R_{h+b} + V^\pi_{h+B_h}(s_{h+B_h})\vert s_h=s, E_h^n, \pi},
    \end{align*}
    where the expectation is over the entire episode in $\Mcal$.
\end{corollary}
\begin{proof}
    For ABPs, the choice of the lookahead range only depends on the state and timestep, and the policy inside a batch does not depend on the number of past batches (in contrast to extended ABPs); we denote the policy inside a batch by $\phi^h$. Therefore, by \Cref{lemma: value in augmented MDP}, the value of the policy in the augmented MDP is 
    {\small\begin{align*}
        V^\pi_t(s,h\vert \MB) = \E_{\substack{\bR\sim \Rcal_{h}^{B_h}(s),\\ \bs'\sim \Pcal_{h}^{B_h}(s)}}\brs*{\sum_{b=0}^{B_h-1}\Rla_{h+b\vert h}(s,\phi^{h},\bR,\bs') + V^{\pi}_{t+B_h+1}\br*{\sla_{h+B_h\vert h}(s,\phi^{h},\bR,\bs'),h+B_h}\vert s_t=s,\pi,\MB } ,
    \end{align*}}
    initialized to zero. By the coupling between $\Mcal$ and $\MB$ (that is, the data generation process in \Cref{appendix: data generation} and the related construction of $\MB$), this value coincides with the one at $\Mcal$ at the beginning of any new batch
    \begin{align*}
        V^\pi_t(s,h\vert \Mcal) &= \E\brs*{\sum_{b=0}^{B_h-1}\Rla_{h+b\vert h}(s,\phi^{h},I) + V^{\pi}_{t+B_h+1}\br*{\sla_{h+B_h\vert h}(s,\phi^{h},I),h+B_h}\vert s_t=s, E^n_h,\pi,\Mcal } \\
        & = \E_{I\sim\Ical_{h,B_h}(s)}\brs*{\sum_{b=0}^{B_h-1}\Rla_{h+b\vert h}(s,\phi^{h},I) + V^{\pi}_{t+B_h+1}\br*{\sla_{h+B_h\vert h}(s,\phi^{h},I),h+B_h}\vert \pi} ,
    \end{align*}
    where the last equality holds since the sampling model generates fresh lookahead information at the beginning of each batch. Thus, to prove the corollary, we only need to prove that the values are independent of the index $t$. The proof follows the exact same induction as in the proof of \Cref{lemma: value in augmented MDP} for the optimal value, which we briefly remind. Clearly, when $h=H+1$, the initial $0$-value is the same for all $t$. Then, assuming the claim holds by induction for all $h'>h$ and $t\le 2h'-1$, the claim trivially holds due to the dynamic programming equation at step $h$ for all $t\le 2h-1$. Thus, by induction, for any $h$, if a batch starts at $t\le 2h-1$, the future expected cumulative reward in $\Mcal$ is 
    \begin{align*}
        V^\pi_h(s) = V^\pi_t(s,h\vert \MB)  = \E_{I\sim\Ical_{h,B_h}(s)}\brs*{\sum_{b=0}^{B_h-1}\Rla_{h+b\vert h}(s,\phi^{h},I) + V^{\pi}_{h+B_h}\br*{\sla_{h+B_h\vert h}(s,\phi^{h},I)}\vert \pi}.
    \end{align*}
    Noticing that the batching structure enforces that $t\le 2h-1$ under all ABPs removes the restriction on $t$ and concludes the proof of the first statement.

    For the second statement, we again rely on the independent sampling of the lookahead at batch starts and write
    \begin{align*}
        V^\pi_{h}(s) 
        & =  \E_{I\sim\Ical_{h,B_h}(s)}\brs*{\sum_{b=0}^{B_h-1}\Rla_{h+b\vert h}(s,\phi^{h},I) + V^\pi_{h+B_h}(\sla_{h+B_h\vert h}(s,\phi^{h},I))\vert \pi} \\
        & = \E\brs*{\sum_{b=0}^{B_h-1}\Rla_{h+b\vert h}(s,\phi^{h},I) + V^\pi_{h+B_h}(\sla_{h+B_h\vert h}(s,\phi^{h},I))\vert  s_h=s,E^n_h,\pi} \\
        & = \E\brs*{\sum_{b=0}^{B_h-1}R_{h+b} + V^\pi_{h+B_h}(s_{h+B_h})\vert s_h=s,E^n_h,\pi}.
    \end{align*}
    The last equality is due to the definition of the lookahead information -- starting from $s_h=s$, the policy $\phi^h$ and lookahead information $I$ fully dictate the future collected rewards $R_{h+b}$ and visited states $s_{h+b}$ throughout the lookahead range.
\end{proof}

This corollary allows us to directly prove \Cref{prop: opt values} by substitution.
\optimalValue*
\begin{proof}  
    Following the coupling between the environments $\Mcal$ and $\MB$, by substituting $\Qla^*_h(s,B;I,V_{h+B}^*)$ (see \Cref{eq: Q-val def}), $V^*_h(s)$ is equal to the optimal value in $\MB$, as stated in \Cref{lemma: value in augmented MDP}. Since $\MB$ is Markov, there exists a policy that simultaneously maximizes the value for all states and steps in the interaction. Moreover, by the construction of $\MB$, every ABP corresponds to a policy in $\MB$, so the value of ABPs is upper bounded by the optimal value in $\MB$. Thus, if we prove that an ABP achieves this value, it is necessarily an optimal ABP simultaneously for all initial states; doing so will conclude the proof. We now show that the policy $\pi^*$ described in the statement indeed has a value $V^*_h(s)$. 
    
    By \Cref{corollary: ABP value}, the value of $\pi^*$ follows the dynamic programming equations
    \begin{align}
        \label{eq: dp of opt policy substitution}
        V^{\pi^*}_{h}(s) 
        & =  \E_{I\sim\Ical_{h,B_h^*}(s)}\brs*{\sum_{b=0}^{B_h^*-1}\Rla_{h+b\vert h}(s,\phi^{*,h},I) + V^{\pi^*}_{h+B_h^*}(\sla_{h+B_h^*\vert h}(s,\phi^{*,h},I))}.
    \end{align}
    We now prove by induction that this value is equal to the value $V^{*}_h(s)$. The base of the induction trivially holds when $h=H+1$. Next, assume that the claim holds for any $h'\ge h+1$ and $s\in\Scal$. Then, substituting the induction hypothesis in \cref{eq: dp of opt policy substitution}, we get 
    \begin{align*}
        V^{\pi^*}_{h}(s) 
        & =  \E_{I\sim\Ical_{h,B_h^*}(s)}\brs*{\sum_{b=0}^{B_h^*-1}\Rla_{h+b\vert h}(s,\phi^{*,h},I) + V^{\pi^*}_{h+B_h^*}(\sla_{h+B_h^*\vert h}(s,\phi^{*,h},I))} \\
        & =  \E_{I\sim\Ical_{h,B_h^*}(s)}\brs*{\sum_{b=0}^{B_h^*-1}\Rla_{h+b\vert h}(s,\phi^{*,h},I) + V^{*}_{h+B_h^*}(\sla_{h+B_h^*\vert h}(s,\phi^{*,h},I))} \tag{by induction hypothesis}\\
        & = \E_{I\sim\Ical_{h,B_h^*}(s)}\brs*{\max_{\phi\in\Pi_{det}}\brc*{\sum_{b=0}^{B_h^*-1}\Rla_{h+b\vert h}(s,\phi,I) + V^{*}_{h+B_h^*}(\sla_{h+B_h^*\vert h}(s,\phi,I))}} \tag{by def. of $\phi^{*,h}$}\\
        & = \max_{B\in[\ell_h]}\E_{I\sim\Ical_{h,B}(s)}\brs*{\max_{\phi\in\Pi_{det}}\brc*{\sum_{b=0}^{B-1}\Rla_{h+b\vert h}(s,\phi,I) + V^{*}_{h+B}(\sla_{h+B\vert h}(s,\phi,I))}} \tag{by def. of $B_h^*$} \\
        & = V^{*}_{h}(s) .
    \end{align*}
    This proves the induction, showing that the value of $\pi^*$ at the beginning of new batches is indeed $V^*$, and also concludes the proof.
\end{proof}

\clearpage

\section{Additional Notations}
\label{appendix: additional notations}
In this section, we present additional notations that we require for the algorithm and proofs.

\textbf{Bonus.} We first explicitly present the bonus used in our algorithm and fix it to be
\begin{align}
    \label{eq: bonus}
    b^k_{h}(s,B) = \sqrt{\frac{8\widehat{\VAR}_h^{k}(s,B;\bar{V}^k_{h+B}) L^k_{\delta}}{n^{k}_{h}(s,B) \vee 1}} + \frac{11H L^k_{\delta}}{n^{k}_{h}(s,B) \vee 1},  
\end{align}
where $L^k_{\delta}=\ln\frac{18SH\ell k^3(k+1)}{\delta}\ge1$.

\textbf{Optimal $Q$-values and their variance.} We let $Q^*_h(s,B)$ be the value of the optimal batching policy that starts a new batch of length $B$ at time step $h$ and state $s$ and continues optimally. This value can be explicitly written as
\begin{align*}
    Q^*_h(s,B) = \E_{I\sim\Ical_{h,B}(s)}\brs*{\Qla^*_h(s,B;I,V^*_{h+B})}.
\end{align*}
In particular, notice that $V^*_h(s) = \max_{B\in[\ell_h]}\brc*{ Q^*_h(s,B)}$.

We denote by $\VAR^*_h(s,B)$, the variance of the optimal return starting from state $s$ at step $h$ with lookahead range $B$. The variance is calculated w.r.t. fresh samples of lookahead information
\begin{align}
     \VAR^*_h(s,B)
    &= \E_{I\sim\Ical_{h,B}(s)}\brs*{\Qla^*_h(s,B;I,V^*_{h+B})^2} - \br*{\E_{I\sim\Ical_{h,B}(s)}\brs*{\Qla^*_h(s,B;I,V^*_{h+B})}}^2. \label{eq: var opt value}
\end{align}
\textbf{$Q$-values of ABPs and their variance.} For any ABP $\pi$, recall that we use the notation $\phi^h$ to denote the Markov policy played given a batch starting at step $h$ and the lookahead range as $B_h=B_h^{\pi}(s)$. We now denote the value of $\pi$ given a batch starting at step $h$, state $s$ and lookahead information $I$ to be 
    \begin{align}
        \label{eq: Q-value ABP}
        \Vla^\pi_h(s,I)=\sum_{t=h}^{h+B_h-1}\Rla_{t\vert h}(s,\phi^h,I) +V^{\pi}_{h+B_h}(\sla_{h+B_h\vert h}(s,\phi^h,I)).
    \end{align}
    In particular, by \Cref{corollary: ABP value}, it holds that $\E_{I\sim\Ical_{h,B_h}(s)}\brs*{\Vla^\pi_h(s,I)}=V^\pi_h(s)$, and we denote the variance of this value by 
    \begin{align}
        \label{eq: Q-value variance ABP}
            \VAR^\pi_h(s) = \E_{I\sim\Ical_{h,B_h}(s)}\brs*{\Vla^\pi_h(s,I)^2} - \br*{V^\pi_h(s)}^2
            = \E_{I\sim\Ical_{h,B_h}(s)}\brs*{\br*{\Vla^\pi_h(s,I)-V^\pi_h(s)}^2}.
    \end{align}
    Notice that by the boundedness of the rewards (and consequently, of the values), it holds that $\Vla^\pi_h(s,I)\in[0,H]$ and $\VAR^\pi_h(s)\in[0,H^2]$. Moreover, since we push through the lookahead with the policy we play, assuming we start from a new batch in step $h$, we have that $\Rla_{t\vert h}(s,\phi^h,I)=R_t$ and $\sla_{t\vert h}(s,\phi^h,I)=s_t$, so we can write
    \begin{align*}
        & \Vla^\pi_h(s,I) = \sum_{t=h}^{h+B_h-1}R_t + V_{h+B_h}^{\pi}(s_{h+B_h}),\textrm{ and}\\
        &\VAR^\pi_h(s) = \E\brs*{\br*{\sum_{t=h}^{h+B_h-1}R_t + V_{h+B_h}^{\pi}(s_{h+B_h})-V_{h}^{\pi}(s)}^2\vert \pi,s_h=s,E^n_h} .
    \end{align*}
    where in both expressions, the dependence in the lookahead is implicit in the rewards and states.

Finally, for ease of notation, for any $x\in\R^d$, define the set 
\begin{align*}
    [x,C] = \brc*{y\in\R^d: y(i)\in[x(i),C], \forall i\in[d]}.
\end{align*}

\clearpage

\section{Good Events}

\subsection{The First Good Event -- Concentration}
\label{appendix: concentration event}
We define the following good events:
{\small
\begin{align*}
    &E^{opt}(k)=\brc*{\forall s,h,B:\ \abs*{Q^*_h(s,B) - \widehat{\E}_{h,s,B}^k\brs*{\Qla^*_h(s,B;I,V^*_{h+B})}} \leq \sqrt{\frac{2\VAR^*_h(s,B) L^k_{\delta}}{n^{k}_h(s,B) \vee 1}} + \frac{HL^k_{\delta}}{n^{k}_h(s,B) \vee 1}}\\
    &E^{\widehat{\VAR}}(k)=\brc*{\forall s,h,B:\ \abs*{ \sqrt{\widehat{\VAR}_h^{k}(s,B;V^*_{h+B})} -  \sqrt{\VAR^*_h(s,B)} } \leq 4H\sqrt{\frac{L^k_{\delta}}{n^{k}_h(s,B)\vee 1}} }\\
    &E^{Err}(k)=\Biggl\{\forall s,h,B\in[\ell_h], \forall V\in[V^*_{h+B},H]: \Biggr.\\
    & \qquad\qquad \qquad \widehat{\E}_{h,s,B}^k\brs*{\Qla^*_h(s,B;I,V)-\Qla^*_h(s,B;I,V^*_{h+B})}\leq \br*{1+\frac{1}{2H}}\E_{I\sim\Ical_{h,B}(s)}\brs*{\Qla^*_h(s,B;I,V)-\Qla^*_h(s,B;I,V^*_{h+B})} \\
    & \hspace{8.25cm} \Biggl.+ \frac{4H^2SL^k_{\delta}}{n^{k}_h(s,B)\vee 1}\Biggr\}
\end{align*}
}
where we again use $L^k_{\delta}=\ln\frac{18SH\ell k^3(k+1)}{\delta}$. The first good event is the intersection of all these events, i.e.,
$$\G_1 = \bigcap_{k\geq 1} E^{opt}(k) \bigcap_{k\geq 1} E^{\widehat{\VAR}}(k)\bigcap_{k\geq 1} E^{Err}(k),$$
and for which the following holds: 
\begin{lemma}[The first good event]\label{lemma: first good event}
It holds that $\Pr(\G_1)\geq 1-\delta/2$.
\end{lemma}
\begin{proof}
    We first remind the reader that when limiting ourselves to adaptive batching policies, upon ending a batch, the agent first decides on the next batching horizon and only then observes the lookahead information. Consequently, for such policies, the distribution of the lookahead information starting at state $s_h=s$ with lookahead range of $B$, $\Ical_{h,B}(s)$, is independent of any event up to timestep $h$ (including whether $s_h=s$). Moreover, no new lookahead information is observed until the start of a new batch. As such, we can assume w.l.o.g. that the lookahead information is generated only after the agent chooses the lookahead range $B$, and is generated simultaneously for the whole batch of $B$ steps.

    Specifically, we assume w.l.o.g. that before the interaction starts, we generate $I_{h,B}^n(s)$ for all $h\in[H],n\in[K]$, $s\in\Scal$ and $B\in[\ell_h]$. Then, upon the $i^{th}$ time the agent is starting a new batch of lookahead range $B$ at $s_h=s$, the agent observes the lookahead information $I_{h,B}^i(s)$. Under this assumption, if $n_h^{k}(s,B)=n\ge1$, then $\widehat{\E}_{h,s,B}^k[f(I)] = \frac{1}{n} \sum_{i=1}^nf(I_{h,B}^i(s)).$ 

    As a preliminary step, we analyze the domain of the values and Q-values. First, since all rewards are non-negative, and due to the clipping of $\bar{Q}^k_h(s,B)$, it holds that $\bar{V}^k_h(s)\in\brs*{0,H-h+1}$ for all $k,h$ and $s$.

    Next, we recursively prove that $V^*_h(s)$, $Q^*_h(s,B)$ and $\Qla^*_h(s,B;I,V^*_{h+B})$ are all bounded (a.s.) in $[0,H-h+1]$. The recursion initializes at $V^*_{H+1}(s)=0$, obeying this requirement. We continue and assume that the domain holds by induction for $V^*_{h'}(s)$ for all $h'>h$ and $s\in\Scal$. Since the reward is bounded a.s. in $[0,1]$, then for any $s\in\Scal,B\in[\ell_h]$ and valid lookahead information $I$,
    \begin{align*}
        \Qla^*_h(s,B;I,V^*_{h+B}) &= \max_{\phi\in\Pi_{det}}\brc*{\sum_{t=h}^{h+B-1}\Rla_{t\vert h}(s,\phi,I) +V^*_{h+B}(\sla_{h+B\vert h}(s,\phi,I))} \\
        &\leq B+(H-B-h+1) \leq H-h+1,
    \end{align*}
    and, similarly, $\Qla^*_h(s,B;I,V^*_{h+B})\geq 0$. 
    Consequently, it holds that 
    \begin{align*}
        Q^*_h(s,B) = \E\brs*{\Qla^*_h(s,B;I,V^*_{h+B})}\in[0,H-h+1], \quad \textrm{and,}\quad V^*_h(s) = \max_{B\in[\ell_h]}\brc*{ Q^*_h(s,B)}\in[0,H-h+1].
    \end{align*}
    Notably, the support of all values directly implies that all the events trivially hold when $n_h^{k}(s,B)=0$: the upper bound for this case is larger than the range of the random quantities. Thus, in the following, we analyze w.l.o.g. only the case of $n_h^{k}(s,B)\ge 1$.
    
    \textbf{Event $\bigcap_{k\geq 1} E^{opt}(k)$.} By the assumption that the lookahead was generated independently before the interaction, it holds that $Q^*_h(s,B) = \E\brs*{\Qla^*_h(s,B;I_{h,B}^i(s),V^*_{h+B})}$, and that different values of $n$ are independent. Moreover, the variance of $\Qla^*_h(s,B;I_{h,B}^i(s),V^*_{h+B})$ is $\VAR^*_h(s,B)$. Thus, by Bernstein's inequality, for any  $\delta'>0$, it holds w.p. $1-\delta'$ that 
        \begin{align}
         \abs*{Q^*_h(s,B) - \frac{1}{n}\sum_{i=1}^n\Qla^*_h(s,B;I_{h,B}^i(s),V^*_{h+B})} &\le \sqrt{\frac{2\VAR^*_h(s,B)\ln\frac{2}{\delta'} }{n}} +\frac{2(H-h+1)\ln\frac{2}{\delta'} }{3n}\nonumber\\
         &\leq \sqrt{\frac{2\VAR^*_h(s,B)\ln\frac{3}{\delta'} }{n}} +\frac{2H\ln\frac{3}{\delta'} }{3n} \label{eq: Bernstein}
     \end{align}
    Specifically, we get
    {\small\begin{align*}
        \Pr&\br*{\bar{E}^{opt}(k)}
         = \Pr\brc*{\exists s,h,B: \abs*{Q^*_h(s,B) - \widehat{\E}_{h,s,B}^k\brs*{\Qla^*_h(s,B;I,V^*_{h+B})}} > \sqrt{\frac{2\VAR^*_h(s,B) L^k_{\delta}}{n^{k}_h(s,B) \vee 1}} + \frac{HL^k_{\delta}}{n^{k}_h(s,B) \vee 1} } \\
        & \leq  \Pr\brc*{\exists s,h,B, \exists n\in[k]: \abs*{Q^*_h(s,B) - \widehat{\E}_{h,s,B}^k\brs*{\Qla^*_h(s,B;I,V^*_{h+B})}} > \sqrt{\frac{2\VAR^*_h(s,B) L^k_{\delta}}{n^{k}_h(s,B) \vee 1}} + \frac{HL^k_{\delta}}{n^{k}_h(s,B) \vee 1}, n^{k}_h(s)=n} \\
        & = \Pr\brc*{\exists s,h,B, \exists n\in[k]: \abs*{Q^*_h(s,B) - \frac{1}{n}\sum_{i=1}^n\Qla^*_h(s,B;I_{h,B}^i(s),V^*_{h+B})}  > \sqrt{\frac{2\VAR^*_h(s,B) L^k_{\delta}}{n}} + \frac{HL^k_{\delta}}{n}, n^{k}_h(s)=n} \\
        & \leq \sum_{s,h,B,n} \Pr\brc*{\abs*{Q^*_h(s,B) - \frac{1}{n}\sum_{i=1}^n\Qla^*_h(s,B;I_{h,B}^i(s),V^*_{h+B})}  > \sqrt{\frac{2\VAR^*_h(s,B) L^k_{\delta}}{n}} + \frac{HL^k_{\delta}}{n}} \\
        & \overset{(*)}\leq SH\ell k\frac{\delta}{6SH\ell k^3(k+1)} 
        \leq \frac{\delta}{6k(k+1)},
    \end{align*}}
    where $(*)$ is by \cref{eq: Bernstein}, since $L^k_{\delta}$ corresponds to the value $\delta'=\frac{\delta}{6SH\ell k^3(k+1)}$. Using the union bound to sum over all $k\ge1$ yields
    \begin{align*}
        \Pr\brc*{\bigcap_{k\geq 1} E^{opt}(k)} = 1- \Pr\brc*{\bigcup_{k\geq 1} \bar{E}^{opt}(k)} \geq 1-\sum_{k\ge1 }\Pr\brc*{\bar{E}^{opt}(k)}
        \geq 1- \sum_{k\ge1}\frac{\delta}{6k(k+1)} \geq 1-\frac{\delta}{6}.
    \end{align*}


    \textbf{Event $\bigcap_{k\geq 1} E^{\widehat{\VAR}}(k)$.} Again, by the independent initial lookahead generation, when $n_h^{k}(s,B)=n$, the variance $\widehat{\VAR}_h^{k}(s,B;V^*_{h+B})$ is calculated using the $n$ i.i.d. samples $\brc*{I_{h,B}^i(s)}_{i\in[n]}$. We now apply \Cref{lemma: variance concentration} on the variables $X_i=\Qla^*_h(s,B;I_{h,B}^i(s),V_{h+B}^*)$ (which are bounded in $[0,H]$), and as with $\bar{E}^{opt}(k)$, we take the union bound on $n\in[k]$ to account for all possible values of $n_h^{k}(s,B)$. Also taking the union bounds for all potential $s,h,B$, and noting that $L^k_{\delta}$ corresponds to setting $\delta'=\frac{\delta}{9SH\ell k^3(k+1)}$, we get
    \begin{align*}
        \Pr\br*{\bar{E}^{\widehat{\VAR}}(k)}
         &= \Pr\brc*{\exists s,h,B:\ \abs*{ \sqrt{\widehat{\VAR}_h^{k}(s,B;V^*_{h+B})} -  \sqrt{\VAR^*_h(s,B)} } > 4H\sqrt{\frac{L^k_{\delta}}{n^{k}_h(s,B)\vee 1}}} \\
         & \leq SH\ell k\frac{\delta}{9SH\ell k^3(k+1)} \leq \frac{\delta}{6k(k+1)}.
    \end{align*}
    Therefore, 
    $\Pr\br*{\bigcap_{k\geq 1}E^{\widehat{\VAR}}(k)}\geq 1-\sum_{k\ge1}\Pr\br*{\bar{E}^{\widehat{\VAR}}(k)}\geq  1-\frac{\delta}{6}.$
    
    \textbf{Event $\bigcap_{k\geq 1} E^{Err}(k)$.} For this event, we apply \Cref{lemma: uniform max concentration} for any $s,h,B$ w.r.t. the random vectors 
    $X_i(s')=\Jla_{h,B}^{I_{h,B}^i(s)}(s,s')$ and $Y_i=\Sla_{h,B}^{I_{h,B}^i(s)}(s)$; by the boundedness of the rewards $X_i\in[0,H]^S$, and clearly $Y_i\subseteq\brs*{S}$. Moreover, for this choice, we get by \cref{eq: J and S to Q} that 
    \begin{align*}
        f(X_i,Y_i,V) =\max_{s'\in Y_i}\brc*{X_i(s')+V(s')}
        =\Qla^*_h(s,B;I_{h,B}^i(s),V).
    \end{align*}
    In particular, fixing $V_0=V_{h+B}^*$, $\delta'=\frac{\delta}{6SH\ell k^2(k+1)}, \alpha=1/2H$ and using the union bound similarly to the events $E^{opt}(k),E^{\widehat{\VAR}}(k)$, \Cref{lemma: uniform max concentration} implies that $\Pr\brc*{\bar{E}^{Err}(k)}\leq \frac{\delta}{6k(k+1)}$. Finally, another union bound over all $k\ge1$ leads to the bound $\Pr\br*{\bigcap_{k\geq 1}E^{Err}(k)}\geq  1-\frac{\delta}{6}.$

    To conclude, we proved that each of the events $\bigcap_{k\geq 1} E^{opt}(k) \bigcap_{k\geq 1} E^{\widehat{\VAR}}(k)\bigcap_{k\geq 1} E^{Err}(k)$ hold w.p. at least $1-\frac{\delta}{6}$; thus, as their intersection, the event $\G_1$ holds w.p. at least $1-\frac{\delta}{2}$.
\end{proof}

\clearpage


\subsection{Optimism Under the First Good Event}
We now prove that under the event $\G_1$, the values that AL-UCB outputs are optimistic. 
\begin{lemma}[Optimism] \label{lemma: optimism of values}
Under the first good event $\G_1$, for all $k\in[K]$, $h\in [H]$, $s\in\Scal$, $B\in[\ell_h]$ and lookahead information $I$, it holds that $\Qla^*_h(s,B;I,V^*_{h+B})\leq \Qla^*_h(s,B;I,\bar{V}^k_{h+B})$; in particular, $V^*_h(s)\leq \bar{V}^k_{h}(s)$.
\end{lemma}\vspace{-.25cm}
\begin{proof}
    We prove this by backwards induction on the values $V^*_h(s)$ and $\bar{V}^k_{h}(s)$. First, by definition, $V^*_{H+1}(s)= \bar{V}^k_{H+1}(s)=0$ for all $s\in\Scal$, and so the base of the induction holds. Now, for any $k\in[K]$ and $h\in[H]$, assume by induction that $V^*_{t}(s)\leq \bar{V}^k_{t}(s)$ for all $t\in\brc*{h+1,\dots,H+1}$ and $s\in\Scal$; we will show that it implies both claims also for timestep $h$.

    For the first claim, by the induction hypothesis, for any $s\in\Scal$, $B\in[\ell_h]$ and lookahead information $I$, it holds that 
    \begin{align}
        \Qla^*_h(s,B;I,\bar{V}^k_{h+B})
        &= \max_{s'\in \Sla_{h,B}^I(s)}\brc*{\Jla_{h,B}^I(s,s') +  \bar{V}^k_{h+B}(s')} \nonumber\\
        &\geq \max_{s'\in \Sla_{h,B}^I(s)}\brc*{\Jla_{h,B}^I(s,s') +  V^*_{h+B}(s')} \tag{$\bar{V}^k_{h+B}(s')\geq V^*_{h+B}(s')$ since $B\ge 1$} \\
        & = \Qla^*_h(s,B;I,V^*_{h+B}) \label{eq: lookahead-conditioned optimism}.       
    \end{align}
    We now continue by proving that under the good event and induction hypothesis, $Q^*_h(s,B)\leq \bar{Q}^k_h(s,B)$ for all $s\in\Scal$ and $B\in[\ell_h]$. We first briefly discuss the domain of the $Q$-values. Since rewards are bounded in $[0,1]$, it holds that $Q^*_h(s,B)\in[0,H-h+1]$. As for the optimistic values, the rewards are non-negative, and therefore so are $\bar{V}^k_{h}(s)$ and $ \Qla^*_h(s,B;I,\bar{V}^k_{h+B})$ (almost surely for all potential lookahead realizations). Moreover, again by the upper bound on the rewards and the truncation of $\bar{V}^k_{h+B}$, it holds that $\Qla^*_h(s,B;I,\bar{V}^k_{h+B})\leq H-h+1$ almost surely. Thus, to summarize, we know that both $Q^*_h(s,B)$ and $\Qla^*_h(s,B;I,\bar{V}^k_{h+B})$ take values in $[0,H-h+1]$. In particular, whenever $\bar{Q}^k_h(s,B) = H-h+1$, the claim $Q^*_h(s,B)\leq \bar{Q}^k_h(s,B)$ trivially holds; therefore, in the following, we assume w.l.o.g. that $\bar{Q}^k_h(s,B) < H-h+1$. Notice that this also implies that $n^{k}_{h}(s,B)\ge1$: when $n^{k}_{h}(s,B)=0$, the bonus is large enough to ensure that $\bar{Q}^k_h(s,B) = H-h+1$. 
    
    Using the aforementioned \cref{eq: lookahead-conditioned optimism} (and the insight that $n^{k}_{h}(s,B)\ge1$) to apply \Cref{lemma: bonus monotonicity} with $\alpha=\sqrt{\frac{8L^k_{\delta}}{n^{k}_h(s,B) \vee 1}}$ and a value upper bound of $H$:\vspace{-.35cm}
    {\small
    \begin{align*}
        \bar{Q}^k_h(s,B) 
        &=  \widehat{\E}_{h,s,B}^k\brs*{\Qla^*_h(s,B;I,\bar{V}^k_{h+B})} + \sqrt{\frac{8\widehat{\VAR}_h^{k}(s,B;\bar{V}^k_{h+B}) L^k_{\delta}}{n^{k}_h(s,B) \vee 1}} + \frac{11H L^k_{\delta}}{n^{k}_{h}(s,B) \vee 1} \\
        & \geq \br*{\widehat{\E}_{h,s,B}^k\brs*{\Qla^*_h(s,B;I,\bar{V}^k_{h+B})} + \max\brc*{\sqrt{\frac{8\widehat{\VAR}_h^{k}(s,B;\bar{V}^k_{h+B}) L^k_{\delta}}{n^{k}_h(s,B) \vee 1}},\frac{8H L^k_{\delta}}{n^{k}_{h}(s,B) \vee 1}}} + \frac{3H L^k_{\delta}}{n^{k}_{h}(s,B) \vee 1} \\
        & \geq \br*{\widehat{\E}_{h,s,B}^k\brs*{\Qla^*_h(s,B;I,V^*_{h+B})} + \max\brc*{\sqrt{\frac{8\widehat{\VAR}_h^{k}(s,B;V^*_{h+B}) L^k_{\delta}}{n^{k}_h(s,B) \vee 1}},\frac{8H L^k_{\delta}}{n^{k}_{h}(s,B) \vee 1}}} + \frac{3H L^k_{\delta}}{n^{k}_{h}(s,B) \vee 1} \tag{\Cref{lemma: bonus monotonicity}} \\
        & \geq \br*{\widehat{\E}_{h,s,B}^k\brs*{\Qla^*_h(s,B;I,V^*_{h+B})} + \sqrt{\frac{2\widehat{\VAR}_h^{k}(s,B;V^*_{h+B}) L^k_{\delta}}{n^{k}_h(s,B) \vee 1}} +\frac{4H L^k_{\delta}}{n^{k}_{h}(s,B) \vee 1}} + \frac{3H L^k_{\delta}}{n^{k}_{h}(s,B) \vee 1}\tag{$\max\brc*{a,b}\ge \frac{a+b}{2}$} \\
        & \geq \widehat{\E}_{h,s,B}^k\brs*{\Qla^*_h(s,B;I,V^*_{h+B})} + \sqrt{\frac{2\VAR^*_h(s,B) L^k_{\delta}}{n^{k}_h(s,B) \vee 1}} + \frac{H L^k_{\delta}}{n^{k}_{h}(s,B) \vee 1} \tag{Under $E^{\widehat{\VAR}}$} \\
        &\geq Q^*_h(s,B) \tag{Under $E^{opt}$}.
    \end{align*}}
    This directly implies that 
    {\small\begin{align*}
        \bar{V}^k_{h}(s) 
        = \max_{B\in[\ell_h]}\brc*{\bar{Q}^k_h(s,B)}
        \geq \max_{B\in[\ell_h]}\brc*{ Q^*_h(s,B)}
        = V^*_{h}(s), 
    \end{align*}}
    completing the proof of the induction.
\end{proof}

\clearpage


\subsection{The Second Good Event -- Cumulative Concentration} \label{appendix: second good event}
Define the following events:
\begin{align*}
    &E^{\VAR}= \brc*{ K\geq 1:\  \sum_{k=1}^K \sum_{i=1}^H \VAR^{\pi^k}_{\ts_i^k}(s_{\ts_i^k}^k)\leq 2\sum_{k=1}^K \sum_{i=1}^H\E\brs*{\VAR^{\pi^k}_{\ts_i^k}(s_{\ts_i^k}^k) \vert s_1^k,\pi^k}  + 4H^3 \ln\frac{4K(K+1)}{\delta}},\\
    &E^{\mathrm{diff}}=\left\{\forall i\in[H], K\geq 1:\ \sum_{k=1}^K \E\brs*{\bar{V}^k_{\ts_{i+1}^k}(s_{\ts_{i+1}^k}^k) - V^{\pi^k}_{\ts_{i+1}^k}(s_{\ts_{i+1}^k}^k)\vert \ts_i^k, s_{\ts_i^k}^k,\pi^k}\right. \\
    &\hspace{14em}\leq \left.\br*{1+\frac{1}{2H}} \sum_{k=1}^K \br*{\bar{V}^k_{\ts_{i+1}^k}(s_{\ts_{i+1}^k}^k) - V^{\pi^k}_{\ts_{i+1}^k}(s_{\ts_{i+1}^k}^k)} + 18H^2 \ln\frac{4HK(K+1)}{\delta}\right\},
\end{align*}
where $\VAR^\pi_h(s)$ is defined at \Cref{eq: Q-value variance ABP}.

The second good event is the intersection of the events $\G_2 =E^{\mathrm{diff}} \cap  E^{\VAR}$.
We define the good event $\G=\G_1\cap\G_2$.
\begin{lemma}
    \label{lemma: good event}
    The good event $\G$ holds with a probability of at least $1-\delta$.
\end{lemma}
\begin{proof}
    \textbf{Event $E^{\VAR}$.} Let $\F_k$ be the natural filtration including all information on the interaction of the algorithm up to the end of episode $k$ (states, actions, batch starting times, lookahead ranges/information), as well as the initial state $s^{k+1}_1$. In particular, $\pi^k$ is $\F_{k-1}$-measurable and $\VAR^{\pi^k}_{\ts_i^k}(s_{\ts_i^k}^k)$ is $\F_k$-measurable. Moreover, since rewards are bounded a.s. in $[0,1]$, $\VAR^{\pi^k}_{\ts_i^k}(s_{\ts_i^k}^k)\in[0,H^2]$ and $\sum_{i=1}^H \VAR^{\pi^k}_{\ts_i^k}(s_{\ts_i^k}^k)$ is a.s. in $[0,H^3]$. Then, by \Cref{lemma: consequences of freedman's inequality},  for any $K\ge1$, it holds w.p. $1 - \frac{\delta}{4K(K+1)}$ that 
    \begin{align*}
        \sum_{k=1}^K \sum_{i=1}^H \VAR^{\pi^k}_{\ts_i^k}(s_{\ts_i^k}^k)\leq 2\sum_{k=1}^K \sum_{i=1}^H\E\brs*{\VAR^{\pi^k}_{\ts_i^k}(s_{\ts_i^k}^k) \vert \F_{k-1}}  + 4H^3 \ln\frac{4K(K+1)}{\delta}
    \end{align*}
    Taking the union bound for all $K\ge1$ (and recalling that $\sum_{n\ge1}\frac{1}{n(n+1)}=1$), the aforementioned event holds uniformly for all $K\ge1$ w.p. $1-\delta/4$. Finally, we note that given $\pi^k$ and $s^k_1$, the $k^{th}$ episode is independent of all events before its start, and since both $\pi^k$ and $s^k_1$ are $\F_{k-1}$-measurable,
    \begin{align*}
        \E\brs*{\VAR^{\pi^k}_{\ts_i^k}(s_{\ts_i^k}^k) \vert \F_{k-1}} = \E\brs*{\VAR^{\pi^k}_{\ts_i^k}(s_{\ts_i^k}^k) \vert s^k_1,\pi^k}.
    \end{align*}
    Substituting this, we get that $\Pr\br*{E^{\VAR}} \ge 1-\delta/4$.

    \textbf{Event $E^{\mathrm{diff}}$.} Let $i\in[H]$ and let $\F_{k,i}$ be the filtration that includes all information in $F_{k-1}$, as well as all information on the interaction in episode $k$ up to the beginning of the $i^{th}$ lookahead batch, including its starting time $\ts_i^k$ and state $s^k_{\ts_i^k}$. Specifically, $\pi^k$, $\bar{V}^k_h$ and $V^{\pi^k}_h$ are all $\F_{k,i}$-measurable and the value difference $\bar{V}^k_{\ts_{i+1}^k}(s_{\ts_{i+1}^k}^k) - V^{\pi^k}_{\ts_{i+1}^k}(s_{\ts_{i+1}^k}^k)$ is $\F_{k+1,i}$-measurable. 
    Moreover, the policy $\pi^k$ decides on the lookahead range only depending on the timestep and state, so we can directly calculate the beginning of the next batch given the starting point of the current one: $\ts_{i+1}^k$ is $\F_{k,i}$-measurable. Consequently, given $\F_{k,i}$, the only randomness in $\bar{V}^k_{\ts_{i+1}^k}(s_{\ts_{i+1}^k}^k) - V^{\pi^k}_{\ts_{i+1}^k}(s_{\ts_{i+1}^k}^k)$ is due to $s_{\ts_{i+1}^k}^k$, which in turn only depends on the lookahead generated inside the batch and the policy. This lookahead information only depends on the starting point of the batch, and thus 
    \begin{align}
         \E\brs*{\bar{V}^k_{\ts_{i+1}^k}(s_{\ts_{i+1}^k}^k) - V^{\pi^k}_{\ts_{i+1}^k}(s_{\ts_{i+1}^k}^k)\vert \F_{k,i}}
         = \E\brs*{\bar{V}^k_{\ts_{i+1}^k}(s_{\ts_{i+1}^k}^k) - V^{\pi^k}_{\ts_{i+1}^k}(s_{\ts_{i+1}^k}^k)\vert \ts_i^k, s_{\ts_i^k}^k,\pi^k}.\label{eq: val diff condition change}
    \end{align}
    Finally, define 
    \begin{align*}
        W_k=\Ind{\forall s\in\Scal,h\in[H]: \bar{V}^k_{h}(s) - V^{\pi^k}_{h}(s)\ge0},\qquad \textrm{and,} \qquad Y_{k,i} = W_k\br*{\bar{V}^k_{\ts_{i+1}^k}(s_{\ts_{i+1}^k}^k) - V^{\pi^k}_{\ts_{i+1}^k}(s_{\ts_{i+1}^k}^k)}.
    \end{align*}
    Notice that $W_k$ is $F_{k,i}$-measurable for all $i$. Furthermore, by \Cref{lemma: optimism of values}, when $\G_1$ holds, then $ \bar{V}^k_{h}(s) \ge V^{*}_{h}(s) \ge  V^{\pi^k}_{h}(s)$, and we have that $W_k=1$. 
    
    By the clipping of $\bar{Q}^k_h$, it holds that $Y_{k,i}\in[0,H]$ (and the same trivially holds in cases where $\ts_{i+1}^k=H+1$: in this case, both values are zero). We can thus apply \Cref{lemma: consequences of freedman's inequality} and get that for any $K\ge1$ and $i\in[H]$, w.p. $1-\frac{\delta}{4HK(K+1)}$,
    \begin{align*}
        \sum_{k=1}^K \E[Y_{k,i}\vert \F_{k,i}]\leq \br*{1+\frac{1}{2H}} \sum_{k=1}^K Y_{k,i} + 18H^2 \ln\frac{4HK(K+1)}{\delta}.
    \end{align*}
    Taking the union bound, the same holds uniformly for all $K\ge1$ and $i\in[H]$ w.p. at least $1-\delta/4$. We now relate this to the event $E^{\mathrm{diff}}$, when $\G_1$ holds.
    \begin{align*}
        \G_1 \cap E^{\mathrm{diff}}
        &= \G_1\cap \left\{\forall i\in[H], K\geq 1:\ \sum_{k=1}^K \E\brs*{\bar{V}^k_{\ts_{i+1}^k}(s_{\ts_{i+1}^k}^k) - V^{\pi^k}_{\ts_{i+1}^k}(s_{\ts_{i+1}^k}^k)\vert \ts_i^k, s_{\ts_i^k}^k,\pi^k}\right. \\
        &\hspace{5em}\leq \left.\br*{1+\frac{1}{2H}} \sum_{k=1}^K \br*{\bar{V}^k_{\ts_{i+1}^k}(s_{\ts_{i+1}^k}^k) - V^{\pi^k}_{\ts_{i+1}^k}(s_{\ts_{i+1}^k}^k)} + 18H^2 \ln\frac{4HK(K+1)}{\delta}\right\} \\
        & = \G_1\cap \left\{\forall i\in[H], K\geq 1:\ \sum_{k=1}^K \E\brs*{\bar{V}^k_{\ts_{i+1}^k}(s_{\ts_{i+1}^k}^k) - V^{\pi^k}_{\ts_{i+1}^k}(s_{\ts_{i+1}^k}^k)\vert \F_{k,i}}\right. \\
        &\hspace{5em}\leq \left.\br*{1+\frac{1}{2H}} \sum_{k=1}^K \br*{\bar{V}^k_{\ts_{i+1}^k}(s_{\ts_{i+1}^k}^k) - V^{\pi^k}_{\ts_{i+1}^k}(s_{\ts_{i+1}^k}^k)} + 18H^2 \ln\frac{4HK(K+1)}{\delta}\right\} \tag{by \cref{eq: val diff condition change}}\\
        & \overset{(*)}{=} \G_1\cap \left\{\forall i\in[H], K\geq 1:\ \sum_{k=1}^K \E\brs*{W_k\br*{\bar{V}^k_{\ts_{i+1}^k}(s_{\ts_{i+1}^k}^k) - V^{\pi^k}_{\ts_{i+1}^k}(s_{\ts_{i+1}^k}^k)}\vert \F_{k,i}}\right. \\
        &\hspace{5em}\leq \left.\br*{1+\frac{1}{2H}} \sum_{k=1}^K W_k\br*{\bar{V}^k_{\ts_{i+1}^k}(s_{\ts_{i+1}^k}^k) - V^{\pi^k}_{\ts_{i+1}^k}(s_{\ts_{i+1}^k}^k)} + 18H^2 \ln\frac{4HK(K+1)}{\delta}\right\}\\
        & = \G_1\cap \underbrace{\brc*{\forall i\in[H], K\geq 1: \sum_{k=1}^K \E[Y_{k,i}\vert \F_{k,i}]\leq \br*{1+\frac{1}{2H}} \sum_{k=1}^K Y_{k,i} + 18H^2 \ln\frac{4HK(K+1)}{\delta}}}_{\tilde{E}^{\mathrm{diff}}},
    \end{align*}
    where relation $(*)$ holds since $W_k$ is $\F_{k,i}$-measurable and equals $1$ under $\G_1$. In particular, we know that $\G_1$ holds w.p. $1-\delta/2$ (by \Cref{lemma: first good event}) and $\tilde{E}^{\mathrm{diff}}$ holds w.p. $1-\delta/4$, and so $\Pr\brc*{\overline{\G_1 \cap E^{\mathrm{diff}}}}\leq 3\delta/4$.

    \textbf{Event $\G$.} Combining both parts, we have
    \begin{align*}
        \Pr\brc*{\overline{\G}} 
        = \Pr\brc*{\overline{\G_1\cap E^{\mathrm{diff}}\cap E^{\VAR}}} 
        \leq \Pr\brc*{\overline{E^{\VAR}}}  + \Pr\brc*{\overline{\G_1\cap E^{\mathrm{diff}}}}
        \leq \delta/4 +3\delta/4 = \delta.
    \end{align*}
    
\end{proof}

\clearpage


\section{Regret Analysis}
\label{appendix: regret}
\regretALUCB*
\begin{proof}
    Throughout the proof, we assume that the good event $\G$ holds; therefore, by \Cref{lemma: good event}, the resulting bounds will hold w.p. at least $1-\delta$. 
    Fix $k\in[K], h\in[H]$ and $s\in\Scal$. We start by rewriting the value $V^{\pi^k}_{h}(s)$ using the notation $\Qla^*$. To ease notations, we now shorten $B_h^{\pi^k}(s)\to B_h$ and $\E_{I\sim\Ical_{h,B_h}(s)}\brs*{\cdot}\to\E_{I}\brs*{\cdot}$.
    \begin{align*}
        V^{\pi^k}_{h}(s)
        &= \E_{I}\brs*{\sum_{t=h}^{h+B_h-1}\Rla_{t\vert h}(s,\phi^k_h,I) + V^{\pi^k}_{h+B_h}(\sla_{h+B_h\vert h}(s,\phi^k_h,I))} \tag{By \Cref{corollary: ABP value}} \\
        & = \E_{I}\brs*{\sum_{t=h}^{h+B_h-1}\Rla_{t\vert h}(s,\phi^k_h,I) + \bar{V}^k_{h+B_h}(\sla_{h+B_h\vert h}(s,\phi^k_h,I))}\\
         &\quad + \E_{I}\brs*{V^{\pi^k}_{h+B_h}(\sla_{h+B_h\vert h}(s,\phi^k_h,I)) - \bar{V}^k_{h+B_h}(\sla_{h+B_h\vert h}(s,\phi^k_h,I))} \\
         & = \E_{I}\brs*{\max_{\phi\in\Pi_{det}}\brc*{\sum_{t=h}^{h+B_h-1}\Rla_{t\vert h}(s,\phi,I) + \bar{V}^k_{h+B_h}(\sla_{h+B_h\vert h}(s,\phi,I))}}\\
         &\quad + \E_{I}\brs*{V^{\pi^k}_{h+B_h}(\sla_{h+B_h\vert h}(s,\phi^k_h,I)) - \bar{V}^k_{h+B_h}(\sla_{h+B_h\vert h}(s,\phi^k_h,I))}  \tag{Def. of $\phi^k_h$  in \Cref{eq: phi^k}}\\
         & = \E_{I}\brs*{\Qla^*_h(s,B_h;I,\bar{V}^k_{h+B_h})}
         + \E_{I}\brs*{V^{\pi^k}_{h+B_h}(\sla_{h+B_h\vert h}(s,\phi^k_h,I)) - \bar{V}^k_{h+B_h}(\sla_{h+B_h\vert h}(s,\phi^k_h,I))}.
    \end{align*}
    In particular, by optimism (\Cref{lemma: optimism of values}), we have $V^*_{h}(s) - V^{\pi^k}_{h}(s)
        \leq \bar{V}^k_{h}(s) - V^{\pi^k}_{h}(s)$: we henceforth bound the difference between the optimistic and real value of $\pi^k$.
    {\small\begin{align}
        \bar{V}^k_{h}(s) - V^{\pi^k}_{h}(s) & = \bar{Q}^k_h(s,B_h) - V^{\pi^k}_{h}(s)\nonumber\\
        & \leq \widehat{\E}_{h,s,B_h}^k\brs*{\Qla^*_h(s,B_h;I,\bar{V}^k_{h+B_h})} + b^k_{h}(s,B_h) - \E_{I}\brs*{\Qla^*_h(s,B_h;I,\bar{V}^k_{h+B_h})} \nonumber\\
        & \quad + \E_{I}\brs*{\bar{V}^k_{h+B_h}(\sla_{h+B_h\vert h}(s,\phi^k_h,I)) - V^{\pi^k}_{h+B_h}(\sla_{h+B_h\vert h}(s,\phi^k_h,I))} \nonumber\\
        & \leq \underbrace{\widehat{\E}_{h,s,B_h}^k\brs*{\Qla^*_h(s,B_h;I,\bar{V}^k_{h+B_h}) - \Qla^*_h(s,B_h;I,V^*_{h+B_h})} - \E_{I}\brs*{\Qla^*_h(s,B_h;I,\bar{V}^k_{h+B_h}) - \Qla^*_h(s,B_h;I,V^*_{h+B_h})}}_{(i)} \nonumber\\ 
        &\quad + \underbrace{\widehat{\E}_{h,s,B_h}^k\brs*{\Qla^*_h(s,B_h;I,V^*_{h+B_h})} - \E_{I}\brs*{\Qla^*_h(s,B_h;I,V^*_{h+B_h})}}_{(ii)}  + b^k_{h}(s,B_h)\nonumber\\ 
        & \quad + \E_{I}\brs*{\bar{V}^k_{h+B_h}(\sla_{h+B_h\vert h}(s,\phi^k_h,I)) - V^{\pi^k}_{h+B_h}(\sla_{h+B_h\vert h}(s,\phi^k_h,I))} \label{eq: value-difference}
    \end{align} }
    Before continuing, we take a brief detour to focus on bounding the bonus.

    \paragraph{Bounding the bonus.} To bound $b^k_{h}(s,B_h)$, we have to bound the empirical variance. In the following, we aim to relate it to the variance of the value of $\pi^k$. We start with some additional notations. Recall that in \Cref{appendix: additional notations}, we denote the value of an ABP $\pi$ given lookahead $I$ by $\Vla^\pi_h(s,I)$ and its variance w.r.t. the lookahead as $\VAR^\pi_h(s)$. 
    Since both values and rewards are non-negative, it holds that $\Vla^\pi_h(s,I)\ge0$. Moreover, by the optimality of $V^*$, 
    \begin{align*}
        \Vla^\pi_h(s,I)
        &\leq\sum_{t=h}^{h+B^{\pi}_h-1}\Rla_{t\vert h}(s,\phi^h,I) +V^{*}_{h+B^{\pi}_h}(\sla_{h+B^{\pi}_h\vert h}(s,\phi^h,I))\\
        & \leq \max_{\phi}\brc*{\sum_{t=h}^{h+B^{\pi}_h-1}\Rla_{t\vert h}(s,\phi,I) +V^{*}_{h+B^{\pi}_h}(\sla_{h+B^{\pi}_h\vert h}(s,\phi,I))} \\
        & = \Qla_h^*(s,B^{\pi}_h;I,V^*_{h+B^{\pi}_h}),
    \end{align*}
    where the inequality holds for all lookahead information $I$. going back to $\pi^k$ and again abbreviate $B_h=B_h^{\pi^k}(s)$, by \Cref{lemma: optimism of values} and under the good event, 
    \begin{align}
        0\leq \underbrace{\Vla^{\pi^k}_h(s,I)}_{'V_1'}
        \leq \underbrace{\Qla_h^*(s,B_h;I,V^*_{h+B_h})}_{'V_2'}
        \leq \underbrace{\Qla^*_h(s,B_h;I,\bar{V}^k_{h+B_h})}_{'V_3'}\leq H.
        \label{eq: value ordering by optimism}
    \end{align}
    Therefore, we can apply \Cref{lemma: variance difference bound with different measures} w.r.t. both the empirical and real lookahead probability measures (that is, generate $I$ with the measures corresponding to $\widehat{\E}_{h,s,B}^k$ and $\E_{I}$). Specifically, we set $\alpha = 4\sqrt{8L^k_{\delta}}H$, $n=n^{k}_h(s,B_h)\vee 1$ and  $\beta=4H\sqrt{\frac{L^k_{\delta}}{n^{k}_h(s,B_h)\vee 1}}$, so that the variance difference assumption holds under the event $E^{\widehat{\VAR}}(k)$. For these values, \Cref{lemma: variance difference bound with different measures} claims that    
    {\small\begin{align*}
        \frac{\sqrt{8\widehat{\VAR}_h^{k}(s,B_h;\bar{V}^k_{h+B_h})L^k_{\delta} }}{\sqrt{n^{k}_h(s,B_h)\vee 1}}
        &\leq \frac{\sqrt{8\VAR^{\pi^k}_h(s)L^k_{\delta} }}{\sqrt{n^{k}_h(s,B_h)\vee 1}} + \frac{1}{4H}\widehat{\E}_{h,s,B_h}^k\brs*{\Qla^*_h(s,B_h;I,\bar{V}^k_{h+B_h})-\Qla_h^*(s,B_h;I,V^*_{h+B_h})}  \\
         &\quad+ \frac{1}{4H}\E_{I}\brs*{\Qla_h^*(s,B_h;I,V^*_{h+B_h})-\Vla^{\pi^k}_h(s,I)} 
         + \underbrace{\frac{16H^2L^k_{\delta}}{n^{k}_h(s,B_h)\vee 1}}_{\sqrt{8L^k_{\delta}}C\alpha/(2n)} + \underbrace{\frac{12H^2L^k_{\delta}}{n^{k}_h(s,B_h)\vee 1}}_{\sqrt{8L^k_{\delta}}\beta/\sqrt{n}}.
    \end{align*}}
    We remark that when $n^{k}_h(s,B_h)=0$, the empirical variance is zero and the above inequality holds trivially. Next, we further rely on the good event, and specifically, the event $E^{Err}(k)$; by the optimism of $\bar{V}^k$, we can bound the empirical expectation by the real expectation and get the bound
    {\small\begin{align*}
        \hspace{-.75cm}\frac{\sqrt{8\widehat{\VAR}_h^{k}(s,B_h;\bar{V}^k_{h+B_h})L^k_{\delta} }}{\sqrt{n^{k}_h(s,B_h)\vee 1}}
        &\leq \frac{\sqrt{8\VAR^{\pi^k}_h(s)L^k_{\delta} }}{\sqrt{n^{k}_h(s,B_h)\vee 1}}  
        + \frac{1}{4H}\br*{1+\frac{1}{2H}}\E_{I}\brs*{\Qla^*_h(s,B_h;I,\bar{V}^k_{h+B_h})-\Qla_h^*(s,B_h;I,V^*_{h+B_h})} + \frac{HSL^k_{\delta}}{n^{k}_h(s,B_h)\vee 1}  \\
         &\quad+ \frac{1}{4H}\E_{I}\brs*{\Qla_h^*(s,B_h;I,V^*_{h+B_h})-\Vla^{\pi^k}_h(s,I)} 
         + \frac{28H^2L^k_{\delta}}{n^{k}_h(s,B_h)\vee 1}\\
         &\overset{(*)}\leq \frac{\sqrt{8\VAR^{\pi^k}_h(s)L^k_{\delta} }}{\sqrt{n^{k}_h(s,B_h)\vee 1}}  
         + \frac{1}{4H}\br*{1+\frac{1}{2H}}\E_{I}\brs*{\Qla^*_h(s,B_h;I,\bar{V}^k_{h+B_h})-\Vla^{\pi^k}_h(s,I)}\\
         &\quad+ \frac{1}{4H}\E_{I}\brs*{\Qla^*_h(s,B_h;I,\bar{V}^k_{h+B_h})-\Vla^{\pi^k}_h(s,I)} 
         + \frac{29H^2SL^k_{\delta}}{n^{k}_h(s,B_h)\vee 1} \\
         & \leq \frac{\sqrt{8\VAR^{\pi^k}_h(s)L^k_{\delta} }}{\sqrt{n^{k}_h(s,B_h)\vee 1}}  
         + \frac{1}{2H}\br*{1+\frac{1}{2H}}\E_{I}\brs*{\Qla^*_h(s,B_h;I,\bar{V}^k_{h+B_h})-\Vla^{\pi^k}_h(s,I)}
         + \frac{29H^2SL^k_{\delta}}{n^{k}_h(s,B_h)\vee 1} 
    \end{align*}}
    In inequality $(*)$, we relied on \cref{eq: value ordering by optimism} to lower/upper bound the values so that the argument of both expectations will be the same, and combined the $1/n$ terms. In the last inequality, we also slightly increased the coefficient of the expectation -- this is valid since optimism ensures that the expectation is positive.

    Finally, we bound the value difference $\Qla^*_h(s,B_h;I,\bar{V}^k_{h+B_h})-\Vla^{\pi^k}_h(s,I)$. Let $\phi^{k,h}$ be the Markov policy played by $\pi^k$ inside a batch starting at step $h$ and state $s$ given lookahead information $I$. Then,
    {\small\begin{align}
        \Qla^*_h(s,B_h;I,\bar{V}^k_{h+B_h})&-\Vla^{\pi^k}_h(s,I) \nonumber\\
        &= \br*{\max_{\phi\in\Pi_{det}}\brc*{\sum_{t=h}^{h+B_h-1}\Rla_{t\vert h}(s,\phi,I) +\bar{V}^k_{h+B_h}(\sla_{h+B_h\vert h}(s,\phi,I))}} \nonumber\\
        &- \br*{\sum_{t=h}^{h+B_h-1}\Rla_{t\vert h}(s,\phi^{k,h},I) +V^{\pi^k}_{h+B_h}(\sla_{h+B_h\vert h}(s,\phi^{k,h},I))} \nonumber\\
        & = \br*{\sum_{t=h}^{h+B_h-1}\Rla_{t\vert h}(s,\phi^{k,h},I) +\bar{V}^k_{h+B_h}(\sla_{h+B_h\vert h}(s,\phi^{k,h},I))}\nonumber \\
        &- \br*{\sum_{t=h}^{h+B_h-1}\Rla_{t\vert h}(s,\phi^{k,h},I) +V^{\pi^k}_{h+B_h}(\sla_{h+B_h\vert h}(s,\phi^{k,h},I))} \tag{Def. of $\phi^{k,h}$} \\
        &= \bar{V}^k_{h+B_h}(\sla_{h+B_h\vert h}(s,\phi^{k,h},I)) - V^{\pi^k}_{h+B_h}(\sla_{h+B_h\vert h}(s,\phi^{k,h},I)). \label{eq: Q-diff to value diff}
    \end{align}}
    Plugging this into the variance bound and then into the bonus leads to the desired bound on the bonus:
    \begin{align*}
        &b^k_{h}(s,B_h) = \sqrt{\frac{8\widehat{\VAR}_h^{k}(s,B_h;\bar{V}^k_{h+B_h}) L^k_{\delta}}{n^{k}_{h}(s,B) \vee 1}} + \frac{11H L^k_{\delta}}{n^{k}_{h}(s,B) \vee 1} \\
        & \leq \frac{\sqrt{8\VAR^{\pi^k}_h(s)L^k_{\delta} }}{\sqrt{n^{k}_h(s,B_h)\vee 1}}  
             + \frac{1}{2H}\br*{1+\frac{1}{2H}}\E_{I}\brs*{\bar{V}^k_{h+B_h}(\sla_{h+B_h\vert h}(s,\phi^{k,h},I)) - V^{\pi}_{h+B_h}(\sla_{h+B_h\vert h}(s,\phi^{k,h},I))}
             + \frac{40H^2SL^k_{\delta}}{n^{k}_h(s,B_h)\vee 1}.
    \end{align*}

    \paragraph{Bounding term $(i)$.} Next, we bound term $(i)$, relying on similar concepts as the bonus bound. We again utilize the optimism and event $E^{Err}(k)$ to bound 
    {\small\begin{align*}
        (i) & = \widehat{\E}_{h,s,B_h}^k\brs*{\Qla^*_h(s,B_h;I,\bar{V}^k_{h+B_h}) - \Qla^*_h(s,B_h;I,V^*_{h+B_h})} - \E_{I}\brs*{\Qla^*_h(s,B_h;I,\bar{V}^k_{h+B_h}) - \Qla^*_h(s,B_h;I,V^*_{h+B_h})}\\
        &\leq \frac{1}{2H} \E_{I}\brs*{\Qla^*_h(s,B_h;I,\bar{V}^k_{h+B_h}) - \Qla^*_h(s,B_h;I,V^*_{h+B_h})} + \frac{4H^2SL^k_{\delta}}{n^{k}_h(s,B_h)\vee 1} \tag{Under $E^{Err}(k)$+optimism} \\
        &\leq \frac{1}{2H} \E_{I}\brs*{\Qla^*_h(s,B_h;I,\bar{V}^k_{h+B_h}) - \Vla^{\pi^k}_h(s,I)} + \frac{4H^2SL^k_{\delta}}{n^{k}_h(s,B_h)\vee 1} \tag{by \cref{eq: value ordering by optimism}} \\
        & = \frac{1}{2H} \E_{I}\brs*{\bar{V}^k_{h+B_h}(\sla_{h+B_h\vert h}(s,\phi^{k,h},I)) - V^{\pi^k}_{h+B_h}(\sla_{h+B_h\vert h}(s,\phi^{k,h},I))} + \frac{4H^2SL^k_{\delta}}{n^{k}_h(s,B_h)\vee 1} \tag{by \cref{eq: Q-diff to value diff}} \\
        &  \leq \frac{1}{2H}\br*{1+\frac{1}{2H}} \E_{I}\brs*{\bar{V}^k_{h+B_h}(\sla_{h+B_h\vert h}(s,\phi^{k,h},I)) - V^{\pi^k}_{h+B_h}(\sla_{h+B_h\vert h}(s,\phi^{k,h},I))}  + \frac{4H^2SL^k_{\delta}}{n^{k}_h(s,B_h)\vee 1}.
        \end{align*}}
        In the last inequality, we slightly increased the coefficient to simplify the final bound -- this is again valid because of the optimism, which ensures the positivity of the value difference. Adding this to the bound on the bonus concludes the proof.

    \paragraph{Bounding term $(ii)$.} Noting that $\E_{I}\brs*{\Qla^*_h(s,B_h;I,V^*_{h+B_h})} = Q^*_h(s,B_h)$, we can bound term $(ii)$ using event $E^{opt}(k)$,
    \begin{align*}
        (ii) 
        = \widehat{\E}_{h,s,B_h}^k\brs*{\Qla^*_h(s,B_h;I,V^*_{h+B_h})} - \E_{I}\brs*{\Qla^*_h(s,B_h;I,V^*_{h+B_h})}
        \leq \sqrt{\frac{2\VAR^*_h(s,B_h) L^k_{\delta}}{n^{k}_h(s,B_h) \vee 1}} + \frac{HL^k_{\delta}}{n^{k}_h(s,B_h) \vee 1}.
    \end{align*}
    By \Cref{lemma: variance difference bound} w.r.t. $\Qla^*_h(s,B_h;I,V^*_{h+B_h})$ and $\Vla^{\pi^k}_h(s,I)$, we can replace the variance of the optimal Q-value with the variance of the policy $\pi^k$. Specifically, choosing $\alpha =2\sqrt{2L^k_{\delta}}H$ and $n=n^{k}_h(s,B_h) \vee 1$, we get
    {\small\begin{align*}
        (ii) 
        &\leq \frac{\sqrt{2\VAR^{\pi^k}_h(s)L^k_{\delta} }}{\sqrt{n^{k}_h(s,B_h)\vee 1}}   
        + \frac{1}{2H}\E_{I}\brs*{\Qla^*_h(s,B_h;I,\bar{V}^k_{h+B_h})-\Vla^{\pi^k}_h(s,I)}
        + \frac{2H^2L^k_{\delta}}{n^{k}_h(s,B_h)\vee 1} \\
        & \leq \frac{\sqrt{2\VAR^{\pi^k}_h(s)L^k_{\delta} }}{\sqrt{n^{k}_h(s,B_h)\vee 1}}   
        + \frac{1}{2H} \E_{I}\brs*{\Qla^*_h(s,B_h;I,\bar{V}^k_{h+B_h}) - \Vla^{\pi^k}_h(s,I)} + \frac{2H^2L^k_{\delta}}{n^{k}_h(s,B_h)\vee 1}  \tag{by \cref{eq: value ordering by optimism}} \\
        & = \frac{\sqrt{2\VAR^{\pi^k}_h(s)L^k_{\delta} }}{\sqrt{n^{k}_h(s,B_h)\vee 1}}  + \frac{1}{2H} \E_{I}\brs*{\bar{V}^k_{h+B_h}(\sla_{h+B_h\vert h}(s,\phi^{k,h},I)) - V^{\pi^k}_{h+B_h}(\sla_{h+B_h\vert h}(s,\phi^{k,h},I))} + \frac{2H^2L^k_{\delta}}{n^{k}_h(s,B_h)\vee 1}  \tag{by \cref{eq: Q-diff to value diff}} \\
        &  \leq \frac{\sqrt{2\VAR^{\pi^k}_h(s)L^k_{\delta} }}{\sqrt{n^{k}_h(s,B_h)\vee 1}}  + \frac{1}{2H}\br*{1+\frac{1}{2H}} \E_{I}\brs*{\bar{V}^k_{h+B_h}(\sla_{h+B_h\vert h}(s,\phi^{k,h},I)) - V^{\pi^k}_{h+B_h}(\sla_{h+B_h\vert h}(s,\phi^{k,h},I))}  + \frac{2H^2L^k_{\delta}}{n^{k}_h(s,B_h)\vee 1} .
    \end{align*}}
    Substituting $(i),(ii)$ and $b^k_{h}(s,B_h)$ back into \cref{eq: value-difference} (and slightly increasing some constants to merge expressions), we get the bound 
    {\small\begin{align}
        \bar{V}^k_{h}(s) - V^{\pi^k}_{h}(s)  &\leq \frac{3\sqrt{2\VAR^{\pi^k}_h(s)L^k_{\delta}}}{\sqrt{n^{k}_h(s,B_h)\vee 1}}   + \br*{1+\frac{3}{2H}}\br*{1+\frac{1}{2H}} \E_{I}\brs*{\bar{V}^k_{h+B_h}(\sla_{h+B_h\vert h}(s,\phi^{k,h},I)) - V^{\pi^k}_{h+B_h}(\sla_{h+B_h\vert h}(s,\phi^{k,h},I))} \nonumber\\ 
         &\quad + \frac{46H^2SL^k_{\delta}}{n^{k}_h(s,B_h)\vee 1}.
        \label{eq: value-difference after calculation}
    \end{align} }
    The next step is to recursively apply this value difference formula between timesteps and episodes. Specifically, for $i\in[H]$, we apply the recursion with $h\gets\ts_i^k$; we remind that this is the step in which the $i^{th}$ batch started on episode $k$ (with $\ts_1^k=1$).  In particular, if $\ts_i^k\le H$, then $\ts_{i+1}^k = \ts_i^k+B_{\ts_i^k}^k$, and we use the convention that $\ts_i^k=H+1$ if there were fewer than $i$ batches. Notably, since we look at batch starts, new lookahead information is independently sampled from $\Ical$, so the expectation w.r.t. $\Ical$ coincides with the sampling of the lookahead during the interaction (given the batch starting time, state and policy). Then, it holds that \vspace{-1.75em}
    
    {\small
    \begin{align}
        &\hspace{-1cm}\sum_{k=1}^K \br*{\bar{V}^k_{\ts_i^k}(s_{\ts_i^k}^k) - V^{\pi^k}_{\ts_i^k}(s_{\ts_i^k}^k)} \nonumber\\
        &\hspace{-.5cm} \leq \sum_{k=1}^K \frac{3\sqrt{2\VAR^{\pi^k}_{\ts_i^k}(s_{\ts_i^k}^k)L^k_{\delta}}}{\sqrt{n^{k}_{\ts_i^k}(s_{\ts_i^k}^k,B_{\ts_i^k}^k)\vee 1}}
        + \br*{1+\frac{3}{2H}}\br*{1+\frac{1}{2H}} \sum_{k=1}^K \E\brs*{\bar{V}^k_{\ts_{i+1}^k}(\sla_{\ts_{i+1}^k\vert \ts_{i}^k}(s_{\ts_i^k}^k,\phi^{k,\ts_i^k},I_{\ts_i^k}^k)) - V^{\pi^k}_{\ts_{i+1}^k}(\sla_{\ts_{i+1}^k\vert \ts_{i}^k}(s_{\ts_i^k}^k,\phi^{k,\ts_i^k},I_{\ts_i^k}^k))\vert \ts_i^k, s_{\ts_i^k}^k,\pi^k} \nonumber\\ 
         &\hspace{-.5cm}\quad + \sum_{k=1}^K \frac{46H^2SL^k_{\delta}}{n^{k}_{\ts_i^k}(s_{\ts_i^k}^k,B_{\ts_i^k}^k)\vee 1} \nonumber\\
         &\hspace{-.5cm} \overset{(1)}{=} \sum_{k=1}^K \frac{3\sqrt{2\VAR^{\pi^k}_{\ts_i^k}(s_{\ts_i^k}^k)L^k_{\delta}}}{\sqrt{n^{k}_{\ts_i^k}(s_{\ts_i^k}^k,B_{\ts_i^k}^k)\vee 1}}
         + \br*{1+\frac{3}{2H}}\br*{1+\frac{1}{2H}} \sum_{k=1}^K \E\brs*{\bar{V}^k_{\ts_{i+1}^k}(s_{\ts_{i+1}^k}^k) - V^{\pi^k}_{\ts_{i+1}^k}(s_{\ts_{i+1}^k}^k)\vert \ts_i^k, s_{\ts_i^k}^k,\pi^k} 
         + \sum_{k=1}^K \frac{46H^2SL^k_{\delta}}{n^{k}_{\ts_i^k}(s_{\ts_i^k}^k,B_{\ts_i^k}^k)\vee 1} \nonumber\\
         &\hspace{-.5cm} \overset{(2)}{\leq} \sum_{k=1}^K \frac{3\sqrt{2\VAR^{\pi^k}_{\ts_i^k}(s_{\ts_i^k}^k)L^k_{\delta}}}{\sqrt{n^{k}_{\ts_i^k}(s_{\ts_i^k}^k,B_{\ts_i^k}^k)\vee 1}}
         + \br*{1+\frac{3}{2H}}\br*{1+\frac{1}{2H}}^2 \sum_{k=1}^K \br*{\bar{V}^k_{\ts_{i+1}^k}(s_{\ts_{i+1}^k}^k) - V^{\pi^k}_{\ts_{i+1}^k}(s_{\ts_{i+1}^k}^k)} + 90H^2 L^K_{\delta} + \sum_{k=1}^K \frac{46H^2SL^k_{\delta}}{n^{k}_{\ts_i^k}(s_{\ts_i^k}^k,B_{\ts_i^k}^k)\vee 1} ,
        \label{eq: cumulative value-difference}
    \end{align}
    }
    where $(1)$ is since $\sla_{\ts_{i+1}^k\vert \ts_{i}^k}(s_{\ts_i^k}^k,\phi^{k,\ts_i^k},I_{\ts_i^k}^k)$ points to the exact state the agent will reach in step $\ts_{i+1}^k$, and $(2)$ is under $E^{\mathrm{diff}}$ (with a slight increase to the log-term). 
    We now apply this recursion $H-1$ times, from $\ts_1^k=1$ up to $\ts_H^k$ (which is trivially $H+1$ if there is no such batch), and under the convention that $n_{H+1}^k(s,B)=\infty$. At the end of the recursion, both values equal zero, and so we get under $\G$ that
    {\small
    \begin{align}
        \hspace*{-2cm}\sum_{k=1}^K &\br*{\bar{V}^k_{1}(s_{1}^k) - V^{\pi^k}_{1}(s_{1}^k)} \nonumber\\
        & = \sum_{k=1}^K \br*{\bar{V}^k_{\ts_1^k}(s_{\ts_1^k}^k) - V^{\pi^k}_{\ts_1^k}(s_{\ts_1^k}^k)}\nonumber\\
        & \leq \br*{1+\frac{3}{2H}}^H\!\br*{1+\frac{1}{2H}}^{2H} \sum_{k=1}^K\sum_{i=1}^H \frac{3\sqrt{2\VAR^{\pi^k}_{\ts_i^k}(s_{\ts_i^k}^k)L^k_{\delta}}}{\sqrt{n^{k}_{\ts_i^k}(s_{\ts_i^k}^k,B_{\ts_i^k}^k)\vee 1}}
        + \br*{1+\frac{3}{2H}}^H\!\br*{1+\frac{1}{2H}}^{2H}\br*{\sum_{k=1}^K\sum_{i=1}^H \frac{46H^2SL^k_{\delta}}{n^{k}_{\ts_i^k}(s_{\ts_i^k}^k,B_{\ts_i^k}^k)\vee 1} +90H^3L^K_{\delta}}\nonumber\\
        & \leq e^3 \sum_{k=1}^K\sum_{i=1}^H \frac{3\sqrt{2\VAR^{\pi^k}_{\ts_i^k}(s_{\ts_i^k}^k)L^k_{\delta}}}{\sqrt{n^{k}_{\ts_i^k}(s_{\ts_i^k}^k,B_{\ts_i^k}^k)\vee 1}}
        + e^3\br*{\sum_{k=1}^K\sum_{i=1}^H \frac{46H^2SL^k_{\delta}}{n^{k}_{\ts_i^k}(s_{\ts_i^k}^k,B_{\ts_i^k}^k)\vee 1}+90H^3L^K_{\delta}} \label{eq: change point reacher feedback} \\
        &\leq 3e^3\br*{2\sqrt{2SH^3K\ell }L^K_{\delta} + 4\sqrt{S\ell}H^2 (L^K_{\delta})^{1.5}} 
        + e^3\br*{SH\ell \br*{2 + \ln(K)}}\br*{136H^2S}L^k_{\delta}\tag{by \Cref{lemma: sum value variance bound,lemma: count sum bounds}}\\
        & = \Ocal\br*{\sqrt{SH^3K\ell }L^K_{\delta} + H^3S^2\ell \br*{L^K_{\delta}}^2}.\nonumber
    \end{align}
    }
    Recalling that the good event holds w.p. $1-\delta$ (by \Cref{lemma: good event}) and using the optimistic bound $V^*_{h}(s) - V^{\pi^k}_{h}(s)\leq \bar{V}^k_{h}(s) - V^{\pi^k}_{h}(s)$ (by \Cref{lemma: optimism of values}) yields that w.p. $1-\delta$, for all $K\ge1$,
        \begin{align*}
            \Regret(K) 
            = \sum_{k=1}^K \br*{V^*_{1}(s_{1}^k) - V^{\pi^k}_{1}(s_{1}^k)}
            \leq \sum_{k=1}^K \br*{\bar{V}^k_{1}(s_{1}^k) - V^{\pi^k}_{1}(s_{1}^k)} 
            =\Ocal\br*{\sqrt{SH^3K\ell }L^K_{\delta} + H^3S^2\ell \br*{L^K_{\delta}}^2}.
        \end{align*}
        This concludes the proof.
\end{proof}

    \subsection{Regret When Lookahead Is Obtained for All Batching Horizons}
    \label{appendix: reacher feedback}
    We now briefly discuss the case in which we obtain full $\ell$-step lookahead at the beginning of each batch (but still compete against the best ABP). This additional information is extremely beneficial for learning: we can update all empirical estimators for all $B\in[\ell_h]$ (instead of just for the played $B_h^k$).

    Concretely, we run the exact same algorithm, but at any step $h$ in which a new batch started, we obtain lookahead $I_{h,B}^k$ for all $B\in[\ell_h]$. In particular, we may update $n_h^k(s_h^k,B)\gets n_h^k(s_h^k,B)+1$ for all $B\in[\ell_h]$, and so $n_h^k(s,B)=n_h^k(s)$. The algorithm does not use any of the additional observations inside an episode, so all of its properties are retained, and the same concentration bounds hold with the counts $n_h^k(s,B)=n_h^k(s)$, with at most minor changes to the union bounds. Thus, following the same derivation up to \cref{eq: change point reacher feedback}, and under the good event, it holds that
    \begin{align}
        \Regret(K) 
        &\leq e^3 \sum_{k=1}^K\sum_{i=1}^H \frac{3\sqrt{2\VAR^{\pi^k}_{\ts_i^k}(s_{\ts_i^k}^k)L^k_{\delta}}}{\sqrt{n^{k}_{\ts_i^k}(s_{\ts_i^k}^k,B_{\ts_i^k}^k)\vee 1}}
        + e^3\br*{\sum_{k=1}^K\sum_{i=1}^H \frac{46H^2SL^k_{\delta}}{n^{k}_{\ts_i^k}(s_{\ts_i^k}^k,B_{\ts_i^k}^k)\vee 1}+90H^2L^k_{\delta}}\nonumber\\
        &=e^3 \sum_{k=1}^K\sum_{i=1}^H \frac{3\sqrt{2\VAR^{\pi^k}_{\ts_i^k}(s_{\ts_i^k}^k)L^k_{\delta}}}{\sqrt{n^{k}_{\ts_i^k}(s_{\ts_i^k}^k)\vee 1}}
        + e^3\br*{\sum_{k=1}^K\sum_{i=1}^H \frac{46H^2SL^k_{\delta}}{n^{k}_{\ts_i^k}(s_{\ts_i^k}^k)\vee 1}+90H^2L^k_{\delta}}. \label{eq: regret reacher feedback}
    \end{align}
    Importantly, we can follow the standard derivation of the inverse count sums (e.g., \Cref{lemma: count sum bounds}) and get 
    \begin{align}
        \sum_{k=1}^K\sum_{i=1}^H \frac{1}{n^{k}_{\ts_i^k}(s_{\ts_i^k}^k)\vee 1} 
        &= \sum_{k=1}^K\sum_{h=1}^H\sum_{s\in\Scal} \frac{\Ind{\exists i\in[H]:\ts_i^k=h}\Ind{s_h^k=s}}{n^{k}_{h}(s)\vee 1} 
        =\sum_{h=1}^H\sum_{s\in\Scal}\sum_{m=0}^{n_h^{K}(s)} \frac{1}{m\vee 1} \nonumber\\
        & \leq\sum_{h=1}^H\sum_{s\in\Scal}\br*{2 + \ln\br*{n_h^{K}(s) \vee 1}}
        \leq \sum_{h=1}^H\sum_{s\in\Scal}\br*{2 + \ln\br*{K}}
        \leq SH \br*{2 + \ln(K)}\label{eq: count reacher feedback}.
    \end{align}
    This relation can also be used to adapt \Cref{lemma: sum value variance bound} and get
    \begin{align*}
       \sum_{k=1}^K\sum_{i=1}^H \frac{\sqrt{\VAR^{\pi^k}_{\ts_i^k}(s_{\ts_i^k}^k)}}{\sqrt{n^{k}_{\ts_i^k}(s_{\ts_i^k}^k)\vee 1}}
        &\leq  \sqrt{\sum_{k=1}^K\sum_{i=1}^H \frac{1}{n^{k}_{\ts_i^k}(s_{\ts_i^k}^k)\vee 1}}
        \sqrt{\sum_{k=1}^K\sum_{i=1}^H \VAR^{\pi^k}_{\ts_i^k}(s_{\ts_i^k}^k)}\tag{Cauchy Schwarz}\\
        &\leq \sqrt{SH \br*{2 + \ln(K)}}\sqrt{\sum_{k=1}^K\sum_{i=1}^H \VAR^{\pi^k}_{\ts_i^k}(s_{\ts_i^k}^k)}\tag{by \cref{eq: count reacher feedback}} \\
        & \leq \sqrt{2SH L^K_{\delta}}\sqrt{2\sum_{k=1}^K \sum_{i=1}^H\E\brs*{\VAR^{\pi^k}_{\ts_i^k}(s_{\ts_i^k}^k) \vert s_1^k,\pi^k}  + 4H^3 \ln\frac{4K(K+1)}{\delta}}\tag{under $E^{\VAR}$} \\
        & \leq \sqrt{2SH L^K_{\delta}}\sqrt{2KH^2  + 4H^3 \ln\frac{4K(K+1)}{\delta}} \tag{by \Cref{lemma: ltv} } \\
        & \leq 2\sqrt{SH^3K L^K_{\delta}} + \sqrt{8S}H^2 L^K_{\delta}.
    \end{align*}
    Substituting back into \cref{eq: regret reacher feedback} results in a regret bound of $\Olog\br*{\sqrt{H^3SK}}$. Following the same flow will similarly lead to this regret bound for fixed batching policies (with any predefined batching horizon).

\clearpage

\subsection{Lemmas for Bounding Variance Terms}

First, we prove a variant of the law of total variance (LTV) for ABPs. 
\begin{lemma}
    \label{lemma: ltv} 
    Let $\pi$ be an ABP, and let $\ts^{\pi}_i$ be the (random) step in which $\pi$ starts its $i^{th}$ batch. Finally, denote the variance of $\pi$ w.r.t. the immediate lookahead information by
\begin{align*}
    \VAR^\pi_h(s) = \E\brs*{\br*{\sum_{t=h}^{h+B_h-1}R_t + V_{h+B_h}^{\pi}(s_{h+B_h})-V_h^{\pi}(s)}^2\vert \pi,s_h=s,E^n_h} .
\end{align*}
    Then, the following holds:
\begin{align*}
    \E\brs*{\sum_{i=1}^H \VAR^\pi_{\ts^{\pi}_i}(s_{\ts^{\pi}_i})\vert\pi,s_1=s} 
    = \E\brs*{\br*{\sum_{t=1}^H R_t - V_1^{\pi}(s_1) }^2\vert\pi,s_1=s}
    \leq H^2.
\end{align*}
\end{lemma}
\begin{proof}
    We extend and adapt the proof of \citep{zanette2019tighter}. 
    We start by proving that for any $h\in[H]$, it holds that
    \begin{align}
        \E\brs*{\br*{\sum_{t=h}^H R_t - V_h^{\pi}(s_h) }^2\vert\pi,s_h, E_h^n}
        = \VAR^{\pi}_h(s_h) + \E\brs*{\br*{\sum_{t=h+B_h}^H R_t - V_{h+B_h}^{\pi}(s_{h+B_h}) }^2\vert\pi,s_h, E_h^n}. \label{eq: ltv middle point}
    \end{align}
    Starting from the l.h.s., we have 
    {\small
    \begin{align*}
        \E&\brs*{\br*{\sum_{t=h}^H R_t - V_h^{\pi}(s_h) }^2\vert\pi,s_h, E_h^n}\\
        &\qquad = \E\brs*{\br*{\br*{\sum_{t=h}^{h+B_h-1} R_t + V_{h+B_h}^\pi(s_{h+B_h}) - V_h^{\pi}(s_h)} + \br*{\sum_{t=h+B_h}^H R_t - V_{h+B_h}^\pi(s_{h+B_h})}}^2\vert\pi,s_h, E_h^n} \\
        &\qquad = \E\brs*{\br*{\sum_{t=h}^{h+B_h-1} R_t + V_{h+B_h}^\pi(s_{h+B_h}) - V_h^{\pi}(s_h)}^2\vert\pi,s_h, E_h^n}  
        + \E\brs*{\br*{\sum_{t=h+B_h}^H R_t - V_{h+B_h}^\pi(s_{h+B_h})}^2\vert\pi,s_h, E_h^n}\\ 
        &\qquad\quad+ 2\E\brs*{\br*{\sum_{t=h}^{h+B_h-1} R_t + V_{h+B_h}^\pi(s_{h+B_h}) - V_h^{\pi}(s_h)}\br*{\sum_{t=h+B_h}^H R_t - V_{h+B_h}^\pi(s_{h+B_h})}\vert\pi,s_h, E_h^n} \\
        &\qquad = \VAR^{\pi}_h(s_h)
        + \E\brs*{\br*{\sum_{t=h+B_h}^H R_t - V_{h+B_h}^\pi(s_{h+B_h})}^2\vert\pi,s_h, E_h^n}\\ 
        &\qquad\quad+ 2\underbrace{\E\brs*{\br*{\sum_{t=h}^{h+B_h-1} R_t + V_{h+B_h}^\pi(s_{h+B_h}) - V_h^{\pi}(s_h)}\br*{\sum_{t=h+B_h}^H R_t - V_{h+B_h}^\pi(s_{h+B_h})}\vert\pi,s_h, E_h^n}}_{(*)} 
    \end{align*}}
    Therefore, to prove \cref{eq: ltv middle point}, we only need to prove that the cross-term $(*)$ equals zero. By the batching structure, given $E_h^n$ and a policy $\pi$, the starting point of the next point is deterministically determined by $s_h$ to be $h+B_h$. Therefore,
    \begin{align*}
        (*)
        &= \E\brs*{\br*{\sum_{t=h}^{h+B_h-1} R_t + V_{h+B_h}^\pi(s_{h+B_h}) - V_h^{\pi}(s_h)}\br*{\sum_{t=h+B_h}^H R_t - V_{h+B_h}^\pi(s_{h+B_h})}\vert\pi,s_h, E_h^n,E_{h+B_h}^n} \\
        &\overset{(1)}{=} \E\biggl\{\br*{\sum_{t=h}^{h+B_h-1} R_t + V_{h+B_h}^\pi(s_{h+B_h}) - V_h^{\pi}(s_h)}\biggr. \\
        & \hspace{4em}\cdot \biggl.\underbrace{\E\brs*{\sum_{t=h+B_h}^H R_t - V_{h+B_h}^\pi(s_{h+B_h})\vert\pi,s_h, E_h^n,E_{h+B_h}^n,s_{h+B_h},R_h,\dots,R_{h+B_h-1}}}_{=0}\vert\pi,s_h, E_h^n,E_{h+B_h}^n\biggr\} \\
        & \overset{(2)}= 0
    \end{align*}
    Relation $(1)$ is by the tower rule, and $(2)$ is by the definition of the value as the future cumulative reward given that a new batch starts. In particular, notice that given $B_h$ (indicating the time $h+B_h$), $E_{h+B_h}^n$ and $s_{h+B_h}$, the rewards collected by an ABP are independent of all past events due to the batching of the lookahead. This concludes the proof of \cref{eq: ltv middle point}.

    We now utilize this equality to prove the main result. In particular, if $h$ is the timestep in which the $i^{th}$ batch has started at ($\ts^{\pi}_i=h$), we have that $\ts^{\pi}_{i+1}=h+B_h$, and the following recursion
    {\small\begin{align*}
        \E&\brs*{\br*{\sum_{t=\ts^{\pi}_i}^H R_t - V_{\ts^{\pi}_i}^{\pi}(s_{\ts^{\pi}_i}) }^2\vert\pi,\ts^{\pi}_i, s_{\ts^{\pi}_i}} \\
        &= \VAR^{\pi}_{\ts^{\pi}_i}(s_{\ts^{\pi}_i}) + \E\brs*{\br*{\sum_{t=\ts^{\pi}_{i+1}}^H R_t - V_{\ts^{\pi}_{i+1}}^{\pi}(s_{\ts^{\pi}_{i+1}}) }^2\vert\pi, \ts^{\pi}_{i}, s_{\ts^{\pi}_{i}}}\tag{By \cref{eq: ltv middle point}, $\ts^{\pi}_{i+1}=\ts^{\pi}_{i}+B_{\ts^{\pi}_{i}}$}\\
        &=  \VAR^{\pi}_{\ts^{\pi}_i}(s_{\ts^{\pi}_i}) + \E\brs*{\E\brs*{\br*{\sum_{t=\ts^{\pi}_{i+1}}^H R_t - V_{\ts^{\pi}_{i+1}}^{\pi}(s_{\ts^{\pi}_{i+1}}) }^2\vert \pi, \ts^{\pi}_{i}, s_{\ts^{\pi}_{i}},\ts^{\pi}_{i+1}, s_{\ts^{\pi}_{i+1}}}\vert\pi,\ts^{\pi}_{i}, s_{\ts^{\pi}_{i}}}\tag{tower rule}\\
        & \overset{(1)}= \VAR^{\pi}_{\ts^{\pi}_i}(s_{\ts^{\pi}_i}) + \E\brs*{\E\brs*{\br*{\sum_{t=\ts^{\pi}_{i+1}}^H R_t - V_{\ts^{\pi}_{i+1}}^{\pi}(s_{\ts^{\pi}_{i+1}}) }^2\vert \pi,\ts^{\pi}_{i+1}, s_{\ts^{\pi}_{i+1}}}\vert\pi,\ts^{\pi}_{i}, s_{\ts^{\pi}_{i}}} \\
        & \overset{(2)}=\VAR^{\pi}_{\ts^{\pi}_i}(s_{\ts^{\pi}_i}) + \E\brs*{\sum_{j=i+1}^H \VAR^{\pi}_{\ts^{\pi}_j}(s_{\ts^{\pi}_j})\vert\pi,\ts^{\pi}_{i}, s_{\ts^{\pi}_{i}}}, 
    \end{align*}}
    where $(1)$ is since all events at $t\ge \ts^{\pi}_{i+1}$ are independent of events at previous times given $\ts^{\pi}_{i+1},s_{\ts^{\pi}_{i+1}}$ (due to the start of a new batch), and $(2)$ recursively applies the same relation up to the potentially $H^{th}$ batch (noting that when $\ts^{\pi}_{j}=H+1$, all terms are zero). Finally, recalling that $\ts^{\pi}_1=1$, we get the desired identity
    \begin{align*}
        \E\brs*{\sum_{i=1}^H \VAR^\pi_{\ts^{\pi}_i}(s_{\ts^{\pi}_i})\vert\pi,s_1=s} 
    = \E\brs*{\br*{\sum_{t=1}^H R_t - V_1^{\pi}(s_1) }^2\vert\pi,s_1=s}.
    \end{align*}
    The r.h.s. is trivially bounded by $H^2$, because $R_h\in[0,1]$ for all $h$ and $V_1^{\pi}(s_1)\in[0,H]$.
\end{proof}


\begin{lemma}
\label{lemma: sum value variance bound}
    Under the event $E^{\VAR}$ it holds that
    \begin{align*}
        \sum_{k=1}^K\sum_{i=1}^H \frac{\sqrt{\VAR^{\pi^k}_{\ts_i^k}(s_{\ts_i^k}^k)}}{\sqrt{n^{k}_{\ts_i^k}(s_{\ts_i^k}^k,B_{\ts_i^k}^k)\vee 1}}
        \leq 2\sqrt{SH^3K\ell L^K_{\delta}} + \sqrt{8S\ell}H^2 L^K_{\delta}.
    \end{align*}
\end{lemma}
\begin{proof}
    \begin{align*}
       \sum_{k=1}^K\sum_{i=1}^H \frac{\sqrt{\VAR^{\pi^k}_{\ts_i^k}(s_{\ts_i^k}^k)}}{\sqrt{n^{k}_{\ts_i^k}(s_{\ts_i^k}^k,B_{\ts_i^k}^k)\vee 1}}
        &\leq  \sqrt{\sum_{k=1}^K\sum_{i=1}^H \frac{1}{n^{k}_{\ts_i^k}(s_{\ts_i^k}^k,B_{\ts_i^k}^k)\vee 1}}
        \sqrt{\sum_{k=1}^K\sum_{i=1}^H \VAR^{\pi^k}_{\ts_i^k}(s_{\ts_i^k}^k)}\tag{Cauchy Schwarz}\\
        &\leq \sqrt{SH\ell \br*{2 + \ln(K)}}\sqrt{\sum_{k=1}^K\sum_{i=1}^H \VAR^{\pi^k}_{\ts_i^k}(s_{\ts_i^k}^k)}\tag{by \cref{lemma: count sum bounds}} \\
        & \leq \sqrt{2SH\ell L^K_{\delta}}\sqrt{2\sum_{k=1}^K \sum_{i=1}^H\E\brs*{\VAR^{\pi^k}_{\ts_i^k}(s_{\ts_i^k}^k) \vert s_1^k,\pi^k}  + 4H^3 \ln\frac{4K(K+1)}{\delta}}\tag{under $E^{\VAR}$} \\
        & \leq \sqrt{2SH\ell L^K_{\delta}}\sqrt{2KH^2  + 4H^3 \ln\frac{4K(K+1)}{\delta}} \tag{by \Cref{lemma: ltv} } \\
        & \leq 2\sqrt{SH^3K\ell L^K_{\delta}} + \sqrt{8S\ell}H^2 L^K_{\delta}.
    \end{align*}\vspace{-.5cm}
\end{proof}
\clearpage

\section{Computational Complexity}
\label{appendix:computational complexity}
In this short appendix, we analyze the computational complexity of our AL-UCB.

We separate the computational analysis into two components: value calculations and the execution of the policy $\pi^k$. For both, we rely on $\Jla^I_{h,b}(s,s')$ and $\Sla^I_{h,b}(s)$. Given $s$ and $I$, both can be calculated for all $s'$ and $b\leq B$ in $\Ocal(SAB)$ time using the dynamic programming equations in \cref{eq: efficient DP J and S}. We can also store the action that optimally leads to each $s'$ in $\Jla^I_{h,b}(s,s')$, using $SB$ memory: this permits reconstructing the optimal path from the initial state to any $s'$. Thus, the overall memory complexity given the initial $s$ and $I$ is $\Ocal(SB)$

\paragraph{Policy Execution.} Assume we pre-computed the optimistic values $\bar{V}^k_{h}$. For any step $h$ in which a batch starts, we compute $\Jla^k_h,\Sla^k_h$, the values of $\Jla^I_{h,b}(s,s')$ and $\Sla^I_{h,b}(s)$ starting from $s^k_h$ up to range $B=B^k_h$. Recalling that 
\begin{align*}
    \Qla^*_h(s,B;I,V)
    = \max_{s'\in \Sla_{h,B}^I(s)}\brc*{\Jla_{h,B}^I(s,s') + V(s')}
    =\max_{\phi\in\Pi_{det}}\brc*{\sum_{t=h}^{h+B-1}\Rla_{t\vert h}(s,\phi,I) +V(\sla_{h+B\vert h}(s,\phi,I))},
\end{align*}
we observe that 
the agent plans $\phi^{k,h}$ to reach $s'\in\Sla^k_h$ for which $\Jla^k_h (s^k_h,s')+\bar{V}^k_{h+B}(s')$ is maximized. 

Computing $\Jla^k_h,\Sla^k_h$ and $\phi^{k,h}$ (using the action trace) requires $\Ocal(SAB_h^k)$ computations and $\Ocal(SB_h^k)$ memory. At the end of the episode, we store $\Jla^k_h,\Sla^k_h$, so the per-episode computational complexity is $\Ocal(SAH)$, and we store a memory of at most $\Ocal(SH)$. 

\paragraph{Value Computations. }
For computing the optimistic values, we need to compute $\bar{Q}^k_h(s,B)$ for all $h,s,B$. To this end, we use the precomputed $\mathfrak{J}_h^k,\mathfrak{S}_h^k$ from previous episodes. Specifically, we need to compute empirical expectation w.r.t. past episodes. For each $(h,s)$, we have at most $k$ past observations (divided across all batching horizons), so each empirical expectation contains at most $k$ terms. Each term requires $\Ocal(S)$ computations, so the per-episode planning complexity is $\Ocal(S^2Hk)$.
Overall, the per-episode computational complexity is $\Ocal(SAH+S^2Hk)$, and the memory up to episode $k$ is $\Ocal(SHk)$.

\clearpage

\section{Concentration Results}

\begin{lemma} \label{lemma: uniform max concentration}
    Fix $K,n\in\N,C>0$ and let $X\in[0,C]^K, Y\subseteq[K]$ be two not necessarily independent random variables (with $Y\ne\emptyset$ w.p. 1). Also, for any $V,V_0\in\R^K$, define the functions
    \begin{align*}
        &f(X,Y;V) = \max_{k\in Y}\brc*{X(k)+V(k)}, 
        & \Delta f(X,Y;V,V_0) = f(X,Y;V) - f(X,Y;V_0).
    \end{align*}
    Finally, let $\brc*{X_i,Y_i}_{i\in[n]}$ be $n$ i.i.d. samples from the joint distribution of $(X,Y)$. Then, for any $\delta\in(0,1)$, $V_0\in\R^K$ and $\alpha\in\left(0,1\right]$, it holds w.p. at least $1-\delta$ that
    \begin{align}   
        \label{eq: uniform concentration multiple lookahead}
        \forall V\in\R^K\ s.t.\ V-V_0\in[0,C]^K:\quad
        \frac{1}{n}\sum_{i=1}^n\Delta f(X_i,Y_i;V,V_0) \leq \br*{1+\alpha}\E\brs*{\Delta f(X,Y;V,V_0)} + \frac{2CK\ln\frac{3n}{\delta}}{\alpha n}.
    \end{align}
\end{lemma}
\begin{proof}
    As a preliminary step, given a fixed $V\in\R^K$ s.t. $V-V_0\in[0,C]^K$, we define the random variables $Z_i^V=\Delta f(X_i,Y_i;V,V_0)$ and analyze their concentration properties. For ease of notation, we denote $m^V=\E\brs*{Z_i^V}$ and 
    $$V_0+[0,C]^K\triangleq\brc*{V\in\R^K: \forall k\in[K],V(k)\in[V_0(k),V_0(k)+C]}.$$
    Then, \cref{eq: uniform concentration multiple lookahead} can be rewritten as 
    \begin{align*}
        \forall V\in V_0+[0,C]^K,\qquad
        \frac{1}{n}\sum_{i=1}^n Z_i^V\leq \br*{1+\alpha}m^V + \frac{2CK\ln\frac{3n}{\delta}}{\alpha n}.
    \end{align*}
    \textbf{The support of $Z_i^V$.} Since $V(k)\ge V_0(k)$ for all $k\in[K]$, then for any $X\in\R^K,Y\subseteq[K]$
    \begin{align*}
        f(X,Y;V) = \max_{k\in Y}\brc*{X(k)+V(k)} \geq \max_{k\in Y}\brc*{X(k)+V_0(k)} =  f(X,Y;V_0)
    \end{align*}
    and so $Z_i^V\geq0$. Moreover, it holds that 
    \begin{align*}
        Z_i^V 
        &= \max_{k\in Y_i}\brc*{X_i(k)+V(k)} - \max_{k\in Y_i}\brc*{X_i(k)+V_0(k)}\\
        &\leq \max_{k\in Y_i}\brc*{\br*{X_i(k)+V(k)} - \br*{X_i(k)+V_0(k)}} \\
        & \leq \max_{k\in [K]}\brc*{V(k) - V_0(k)}\\
        &\leq C.
    \end{align*}
    In other words, it holds that $Z_i^V\in[0,C]$, and so does $m^V\in[0,C]$.

    Next, by \Cref{lemma: freedmans inequality} applied on $Z_i^V-m^V\in[-C,C]$, for any $\eta\in\left(0,\frac{1}{C}\right]$ and $\delta'>0$, it holds w.p. at least $1-\delta'$ that
    \begin{align*}
        \sum_{i=1}^n Z_i^V &\leq n\cdot m^V + \eta \sum_{i=1}^n \E\brs*{\br*{Z_i^V-m^V}^2} + \frac{\ln\frac{1}{\delta'}}{\eta} \\
        & \leq n\cdot m^V + \eta \sum_{i=1}^n \E\brs*{\br*{Z_i^V}^2} + \frac{\ln\frac{1}{\delta'}}{\eta} \\
        & \leq n\cdot m^V + C\eta \sum_{i=1}^n \E\brs*{Z_i^V} + \frac{\ln\frac{1}{\delta'}}{\eta}  \tag{$Z_i^V\in[0,C]$}\\
        & = n(1+C\eta)m^V + \frac{\ln\frac{1}{\delta'}}{\eta}. 
    \end{align*}
    Alternatively, denoting $\alpha=C\eta$, for any $\alpha\in(0,1]$, it holds w.p. $1-\delta'$ that
    \begin{align}
        \frac{1}{n}\sum_{i=1}^n Z_i^V 
        &\leq (1+\alpha)m^V + \frac{C\ln\frac{1}{\delta'}}{\alpha n}. \label{eq: uniform concentration before covering}
    \end{align}
    Next, for some $\epsilon>0$ let $V_1,V_2\in V_0+[0,C]^K$ s.t. $V_1(k)\leq V_2(k)\leq V_1(k)+ \epsilon$ for all $k\in[K]$. In particular, for all $i\in[n]$, it holds that 
    \begin{align*}
        Z_i^{V_1}&=\max_{k\in Y_i}\brc*{X_i(k)+V_1(k)} - \max_{k\in Y_i}\brc*{X_i(k)+V_0(k)}\\
        &\leq \max_{k\in Y_i}\brc*{X_i(k)+V_2(k)} - \max_{k\in Y_i}\brc*{X_i(k)+V_0(k)} \tag{$V_1(k)\leq V_2(k)$ for all $k$}\\
        & = Z_i^{V_2} \\
        &\leq \max_{k\in Y_i}\brc*{X_i(k)+V_1(k)+\epsilon} - \max_{k\in Y_i}\brc*{X_i(k)+V_0(k)} \tag{$V_2(k)\leq V_1(k)+ \epsilon$ for all $k$}\\
        & = \max_{k\in Y_i}\brc*{X_i(k)+V_1(k)} +\epsilon - \max_{k\in Y_i}\brc*{X_i(k)+V_0(k)}\\
        & = Z_i^{V_1}+\epsilon,
    \end{align*}
    i.e., $Z_i^{V_1} \leq Z_i^{V_2} \leq Z_i^{V_1}+\epsilon.$ Taking the expectation, this also implies that $ m^{V_1} \leq m^{V_2} \leq m^{V_1}+\epsilon$, and thus 
    \begin{align}
        \frac{1}{n}\sum_{i=1}^n Z_i^{V_2} - (1+\alpha)m^{V_2} 
        \leq \frac{1}{n}\sum_{i=1}^n Z_i^{V_1}+\epsilon - (1+\alpha)m^{V_1} ,
        \label{eq: lipschitz for concentration}
    \end{align}
    We now utilize this Lipschitzness to apply a covering argument. For any $V\in V_0+[0,C]^K$, let $V_\epsilon$ be the rounded version so that each component is rounded down to an element in $\brc*{V_0(k),V_0(k)+\epsilon,\dots, V_0(k)+\br*{\ceil*{C/\epsilon}-1}\epsilon}$. In particular, each component of $V_\epsilon$ lower bounds and is at most $\epsilon$ away from the respective component of $V$. Then, by \cref{eq: lipschitz for concentration} (with $V\to V_2$ and $V_1\to V_{\epsilon})$, for any $\delta'\in(0,1]$,
    \begin{align*}
        \Pr&\brc*{\exists V\in V_0+[0,C]^K: \frac{1}{n}\sum_{i=1}^n Z_i^{V} - (1+\alpha)m^{V} \geq \frac{C\ln\frac{1}{\delta'}}{\alpha n}+\epsilon} \\
        &\leq \Pr\brc*{\exists V\in V_0+[0,C]^K: \frac{1}{n}\sum_{i=1}^n Z_i^{V_\epsilon}+\epsilon - (1+\alpha)m^{V_\epsilon} \geq \frac{C\ln\frac{1}{\delta'}}{\alpha n}+\epsilon} \\
        & = \Pr\brc*{\exists V_\epsilon\in\R^K \ s.t. \ V_\epsilon-V_0\in \brc*{0,\epsilon,2\epsilon,\dots, \br*{\ceil*{C/\epsilon}-1}\epsilon}^K: \frac{1}{n}\sum_{i=1}^n Z_i^{V_\epsilon}- (1+\alpha)m^{V_\epsilon} \geq \frac{C\ln\frac{1}{\delta'}}{\alpha n}} \\
        & \leq \br*{\frac{C}{\epsilon}+1}^K\delta',
    \end{align*}
    where the last inequality takes the union bound over all possible $V_\epsilon$ and applies \cref{eq: uniform concentration before covering}. Finally, for any $\delta\in(0,1]$ and $\alpha\in(0,1]$, we fix $\epsilon=\frac{CK\ln\frac{3}{\delta}}{\alpha n}$ and $\delta' = \frac{\delta}{\br*{\frac{C}{\epsilon}+1}^K}$. Then, w.p., at least $1-\delta$, for all $V\in V_0+[0,C]^K$, it holds that 
    \begin{align*}
        \frac{1}{n}\sum_{i=1}^n Z_i^{V} - (1+\alpha)m^{V} 
        &\leq \frac{C\ln\frac{\br*{\frac{C}{\epsilon}+1}^K}{\delta}}{\alpha n}+\epsilon \\
        &= \frac{C\ln\frac{\br*{\frac{\alpha n }{K\ln\frac{3}{\delta}}+1}^K}{\delta}}{\alpha n}+\frac{CK\ln\frac{3}{\delta}}{\alpha n}\\
        &\leq \frac{C\ln\frac{(2n)^K}{\delta}}{\alpha n}+\frac{CK\ln\frac{3}{\delta}}{\alpha n}\\
        &\leq \frac{CK\ln\frac{2n}{\delta}}{\alpha n}+\frac{CK\ln\frac{3}{\delta}}{\alpha n}\\
        & \leq \frac{2CK\ln\frac{3n}{\delta}}{\alpha n}.
    \end{align*}
\end{proof}
\clearpage

\section{Existing Auxiliary Lemmas}

\begin{lemma}[Freedman's inequality, \citep{beygelzimer2011contextual}, Theorem 1]\label{lemma: freedmans inequality}
Let $\brc{X_t}_{t\geq 1}$ be a real-valued martingale difference sequence adapted to a filtration $\brc*{F_t}_{t\geq 0}$. If $|X_t|\leq R$ a.s. then for any $\eta\in (0,1/R], T\in \mathbb{N}$ it holds with probability greater than $1-\delta$,
\begin{align*}
    \sum_{t=1}^T X_t \leq \eta \sum_{t=1}^T \E[X_t^2| F_{t-1}] +\frac{\log(1/\delta)}{\eta}.
\end{align*}
\end{lemma}

We now present a variant of three lemmas from \citep{merlis2024reinforcement}. We remark that in the original statement in \citep{merlis2024reinforcement}, all variables were assumed to have finite support. We verified that the proofs, in fact, never rely on this assumption and remain completely identical for continuous random variables, which correspond to the following statements:
\begin{lemma}[\citep{merlis2024reinforcement}, Lemma 16, event $E^{pv2}$]
    \label{lemma: variance concentration}
    Let $X,X_1,\dots,X_n$ be i.i.d. random variables supported in $[0,C]$ for some $C>0$. Also define 
    \begin{align*}
        \widehat{\VAR}_n(X) = \frac{1}{n}\sum_{i=1}^n X_i^2 - \br*{\frac{1}{n}\sum_{i=1}^n X_i}^2,\quad \textrm{and,} \quad 
        \VAR(X) = \E\brs{X^2} - \br*{\E\brs*{X}}^2.
    \end{align*}
    Then, for any $\delta\in(0,1)$, w.p. at least $1-\delta$, it holds that :
    \begin{align*}
        \abs*{\sqrt{\widehat{\VAR}_n(X)} - \sqrt{\VAR(X)}} \leq 4C\sqrt{\frac{\ln\frac{2}{\delta}}{n\vee 1}}.
    \end{align*}
\end{lemma}
\begin{lemma}[\citep{merlis2024reinforcement}, Lemma 21]
\label{lemma: variance difference bound}
    Let $P$ be a distribution over $\X$. Also, let $V_1,V_2:\X\mapsto[0,C]$ be measurable functions for some $C>0$ such that $V_1(x)\leq V_2(x)$ for all $x\in\X$. Then, for any $\alpha,n>0$, it holds that 
    \begin{align*}
        \frac{\sqrt{\VAR_{P}(V_2(X)) }}{\sqrt{n}}
        \leq \frac{\sqrt{\VAR_{P}(V_1(X)) }}{\sqrt{n}} + \frac{1}{\alpha}\E_{X\sim P}\brs*{V_2(X)-V_1(X)} + \frac{C\alpha}{4n},
    \end{align*}
    where $\VAR_{P}(V(X)) = \E_{X\sim P}\brs*{V(X)^2} - \br*{\E_{X\sim P}\brs*{V(X)}}^2$
\end{lemma}

\begin{lemma}[\citep{merlis2024reinforcement}, Lemma 22]
\label{lemma: variance difference bound with different measures}
    Let $P,P'$ be two distributions over $\X$ (w.r.t. the same $\sigma$-algebra). Also, let $V_1,V_2,V_3:\X\mapsto[0,C]$ be measurable functions for some $C>0$ such that $V_1(x)\leq V_2(x)\leq V_3(x)$ for all $x\in\X$. Finally, assume that 
    \begin{align*}
        \abs*{\sqrt{\VAR_{P}(V_2(X))} - \sqrt{\VAR_{P'}(V_2(X)) }} \leq \beta
    \end{align*}
    for some $\beta>0$, where $\VAR_{P}(V(X)) = \E_{X\sim P}\brs*{V(X)^2} - \br*{\E_{X\sim P}\brs*{V(X)}}^2$. Then, for any $\alpha,n>0$, it holds that 
    \begin{align*}
        \frac{\sqrt{\VAR_{P'}(V_3(X)) }}{\sqrt{n}}
        &\leq \frac{\sqrt{\VAR_{P}(V_1(X)) }}{\sqrt{n}} + \frac{1}{\alpha}\E_{X \sim P'}\brs*{V_3(X)-V_2(X)}  
         + \frac{1}{\alpha}\E_{X \sim P}\brs*{V_2(X)-V_1(X)} 
         + \frac{C\alpha}{2n} + \frac{\beta}{\sqrt{n}}.
    \end{align*}
    \end{lemma}

\begin{lemma}[\citep{zhang2024settling}, Appendix C.1]
\label{lemma: bonus monotonicity specific constants}
For any $d\in\N$, $p\in\Delta^d$, $v\in\R_+^d$ s.t. $\norm*{v}_\infty\le H$, $\delta'\in(0,1)$ and positive integer $n$, define the function
\begin{align*}
    f(p,v,n) = p^Tv + \max\brc*{\frac{20}{3}\sqrt{\frac{\VAR_p(v)\ln\frac{1}{\delta'}}{n}},\frac{400}{9}\frac{H\ln\frac{1}{\delta'}}{n}},
\end{align*}
    where $\VAR_p(v)= \sum_{i=1}^d p_iv_i^2 - \br*{\sum_{i=1}^d p_iv_i}^2$. 
    Then, the function $f(p,v,n)$ is non-decreasing in each entry of $v$.
\end{lemma}
\begin{corollary}
    \label{lemma: bonus monotonicity}
    For any $d\in\N$, $p\in\Delta^d$, $v\in\R_+^d$ s.t. $\norm*{v}_\infty\le H$ and $\alpha\ge0$, define the function
\begin{align*}
    f(p,v,\alpha) = p^Tv + \max\brc*{\alpha\sqrt{\VAR_p(v)},\alpha^2 H},
\end{align*}
    where $\VAR_p(v)= \sum_{i=1}^d p_iv_i^2 - \br*{\sum_{i=1}^d p_iv_i}^2$. Then, the function $f(p,v,\alpha)$ is non-decreasing in each entry of $v$.
\end{corollary}
\begin{proof}
    The claim trivially holds for $\alpha=0$. For any $\alpha>0$, define $\delta' = \exp\brc*{-\frac{9}{400}\alpha^2 n}$; since $\alpha,n>0$, this is a valid choice for \Cref{lemma: bonus monotonicity specific constants}, as it yields $\delta'\in(0,1)$. Then, directly substituting this value leads to the $f(p,v,n)$ defined in \Cref{lemma: bonus monotonicity specific constants}: since it is non-decreasing in $v$, so is $f(p,v,\alpha)$ defined in this corollary.
\end{proof}

\begin{lemma}[\citep{efroni2021confidence}, Lemma 27]\label{lemma: consequences of freedman's inequality}
Let $\brc{Y_t}_{t\geq 1}$ be a real-valued sequence of random variables adapted to a filtration $\brc*{F_t}_{t\geq 0}$. Assume that for all $t\geq 1$ it holds that $0\leq Y_{t}\leq C$ a.s., and let $T\in \mathbb{N}$. Then each of the following inequalities holds with probability greater than $1-\delta$.
\begin{align*}
   &\sum_{t=1}^T \E[Y_t|F_{t-1}]\leq \br*{1+\frac{1}{2C}} \sum_{t=1}^T Y_t + 2(2C+1)^2 \ln\frac{1}{\delta},\\
   &\sum_{t=1}^T Y_t \leq 2\sum_{t=1}^T \E[Y_t|F_{t-1}] + 4C\ln\frac{1}{\delta}.
\end{align*}
\end{lemma}

\begin{lemma}[e.g., \citep{merlis2024reinforcement}, Lemma 20]
\label{lemma: count sum bounds}
    Under the convention that $n_h^k(s,B)=\infty$ if $h=H+1$, the following bounds hold:
    \begin{align*}
        & \sum_{k=1}^K\sum_{i=1}^H \frac{1}{n^{k}_{\ts_i^k}(s_{\ts_i^k}^k,B_{\ts_i^k}^k)\vee 1} \leq SH\ell \br*{2 + \ln(K)}.
    \end{align*}
\end{lemma}    
\begin{proof}
    The claim is standard and is provided due to its simplicity and for completeness. By definition, we can omit from the summation all terms where $\ts_i^k=H+1$. Since for all $h\in[H]$, there cannot be two batches starting at this step during the same episode, we can write
    \begin{align*}
        \sum_{k=1}^K\sum_{i=1}^H \frac{1}{n^{k}_{\ts_i^k}(s_{\ts_i^k}^k,B_{\ts_i^k}^k)\vee 1} 
        &= \sum_{k=1}^K\sum_{h=1}^H\sum_{s\in\Scal}\sum_{B=1}^{\ell_h} \frac{\Ind{\exists i\in[H]:\ts_i^k=h}\Ind{s_h^k=s,B_h^k=B}}{n^{k}_{h}(s,B)\vee 1} \\
        & \overset{(1)}{=}\sum_{h=1}^H\sum_{s\in\Scal}\sum_{B=1}^{\ell_h}\sum_{m=0}^{n_h^{K}(s,B)} \frac{1}{m\vee 1} \\
        & \overset{(2)}{\leq}\sum_{h=1}^H\sum_{s\in\Scal}\sum_{B=1}^{\ell_h}\br*{2 + \ln\br*{n_h^{K}(s,B) \vee 1}}\\
        & \overset{(3)}{\leq} \sum_{h=1}^H\sum_{s\in\Scal}\sum_{B=1}^{\ell_h}\br*{2 + \ln\br*{K}}\\
        & \leq SH\ell \br*{2 + \ln(K)}.
    \end{align*}
    Relation $(1)$ holds since the counts start at zero and increase by $1$ every time that a batch of size $B$ starts at step $h$ and state $s$. Relation $(2)$ bounds the harmonic sum. Finally, $(3)$ is since any tuple $(h,s,B)$ can be visited at most once per episode
\end{proof}
\clearpage

\section{Failure of Existing Approaches}
\subsection{Fixed Batching Policies}
\label{appendix: failure fixed batching}
\failureFixedBatch*
\begin{proof}
    Let $s_r$ be the root of a complete tree of depth $\ell$ with $A$ descendants at each node, where $\ell=1$ corresponds to a single root node. We fix the transitions to deterministically navigate agents through the tree and then lead to a terminal non-rewarding state. The only rewards are located at the $A^{\ell-1}$ leaves (for each of the $A$ actions) and are independent Bernoulli random variables of mean $A^{-\ell}$. Notably, a lookahead policy at $s_r$ can see whether a unit reward was realized at some leaf and execute a policy that collects it. This strategy obtains a reward of $1$ iff such a reward was realized in at least one of the leaves: w.p. at least $1-(1-A^{-\ell})^{A^{\ell}}\geq 1-1/e \ge 1/2$. Finally, add an initial state $s_1$ to be a non-rewarding state that deterministically leads to $s_r$; starting at this state clearly does not affect the optimal lookahead value. We now show that regardless of the batching horizon $B$, any fixed batching policy $\pi^B$ fails.
    \begin{itemize}
        \item If $B\le \ell/2$, then $\pi^B$ observes the rewards at a distance of at most $B$ steps, thus limiting to only $A^B$ reachable rewards to which the policy can actively navigate. By the union bound, the probability of a reward being realized in at least one of these leaves is at most $A^BA^{-\ell}\leq A^{-\ell/2}$, so this is the maximal reward that may be collected by such a policy.
        \item If $B> \ell/2$, then $\pi^B$ does not observe any information at the first batch and starts the second batch at some node of depth $B$: from which it can reach at most $A^{\ell-B+1}$ rewards. Thus, the value from this state is upper bounded by the probability of having a reward in at least one of these leaves: at most $A^{\ell-B+1}A^{-\ell}\leq A^{-\ell/2+1}$. 
    \end{itemize}
\end{proof}

\subsection{Model Predictive Control}
\label{appendix: failure MPC}
An $\ell$-step model predictive control is defined by the set of values $V_h(s)$ for all $h\in[H]$ and $s\in\Scal$. For any step $h$, being in a state $s$ and upon observing the $\ell$-step lookahead $I_{h,\ell_h}(s)$, the scheme solves for the deterministic policy
\begin{align}
    \phi\in\argmax_{\phi\in\Pi_{det}}\brc*{\sum_{t=h}^{h+\ell_h-1}\Rla_{t\vert h}(s,\phi,I) +V(\sla_{h+\ell_h\vert h}(s,\phi,I))},
\end{align}
and executes the first action in the plan $\phi_h(s)$. Then, after the environment transitions to $s_{h+1} = \sla_{h+1\vert h}(s,\phi,I)$ and the agent observes an additional step of lookahead, the agent recalculates the plan $\phi$ and again executes only $\phi_{h+1}(s_{h+1})$. This process repeats until the end of the interaction: planning using the lookahead for $\ell$-steps and executing the first action. 

We now claim that any backwards-induced value calculation scheme must be exponentially suboptimal. A reasonable definition of a backwards-induced scheme is to assume that the value $V_h(s)$ can only depend on events starting from step $h$ and state $s$. In fact, for the proof, we will only require a much milder assumption: that if a state $s$ in step $h$ deterministically leads to a non-rewarding terminal state, its value only depends on $h$ and on the reward distribution of the actions. Formally, we say that $s_T$ is a terminal state if for all $h,a$, $\Pr\brc*{s_{h+1}=s\vert s_h=s_T,a}=1$ and $R_h(s_T,a)=0$. Then, we say that a state $s$ is a semi-terminal state in step $h$ if $\Pr\brc*{s_{h+1}=s_T\vert s_h=s,a}=1$: it deterministically leads to a terminal state (with potentially random rewards). In particular, we say w.l.o.g. that episodes always reach a terminal state at step $h=H+1$; therefore, all states are semi-terminal in step $h=H$. 
We thus assume that the value calculation scheme obeys the following assumption, which is the formal version of \Cref{ass: backward informal}:
\begin{assumption}
    \label{assumption:mpc}
    Let $\D$ be the set of all (Borel-measurable) distributions over $[0,1]^A$. For any $h\in[H]$, there exists a function $f_h:\D\mapsto \R_+$ such that if $s$ is a semi-terminal state in step $h$ with a joint reward distribution $\Rcal_h(s)\in\D$ for its actions, it holds that $V_h(s) = f_h(\Rcal_h(s))$.
\end{assumption}
\begin{claim}
    \label{claim: mpc formal}
    Fix $\ell\ge2$, $A\ge2,H\ge \ell+1$ and let $\pi$ be an MPC agent whose values are computed by a procedure that follows \Cref{assumption:mpc}. Then, there exists an environment with $S\le A^{\ell}+\ell+1$  and deterministic transitions for which $\frac{V^\pi_1(s_1)}{V^*_1(s_1)}\leq A^{1-\ell/2}$, where $V^*(s_1)$ is the value of the optimal $\ell$-step lookahead policies. In particular, for the constructed environments, there exists an MPC agent that achieves the optimal value $V_1^*(s_1)$.
\end{claim}
\begin{proof}
    We create two 'state types' $s^B,s^D$ that will be semi-terminal states in step $\ell+1$. In states of type $s^D$, all actions yield a deterministic reward of $A^{-\ell/2}$ and transition to a terminal state $s_T$, while in states of type $s^B$, the rewards are independent Bernoulli $\Ber\br*{A^{-\ell}}$.  By \Cref{assumption:mpc}, there exist two values $V^B,V^D$, such that for all states $s$ of type $s^B$, it holds that $V_{\ell+1}(s)=V^B$, and for all states $s$ of type $s^D$, it holds that $V_{\ell+1}(s)=V^D$. 
    
    We now describe the environments in which we will analyze their value. The initial state is $s_1$, and it deterministically leads (via two arbitrary actions) to states $s^{Tr},s^L$ with respective deterministic rewards $R^{Tr},R^{L}$. From $s^{Tr}$, the environment is a complete $A$-ary tree of depth $\ell$; in particular, it requires less than $A^\ell-1$ states and has $A^{\ell-1}$ leaves (with $A$ actions each). There are no rewards in the middle of the tree, and the leaves will be either of type $s^B$ or $s^D$. On the other hand, from $s^L$, the environment forms a line of length $\ell$ where the only rewarding state is the last one (of type $s^B$ or $s^D$). After traversing either the tree or line, the environment leads to a terminal state $s^T$: this environment can be represented using at most $A^\ell+\ell+1$ states. We now fix the reward configuration depending on the following cases:
    \begin{itemize}
        \item If $V^B>V^D$: Set $R^{Tr}=R^L=0$, and put the state $s^B$ in the end of the line and $s^D$ in the leaves of the tree. Since there are no rewards until step $\ell+1$, the MPC agent will go to the line, which has a higher value $V_{\ell+1}(s^B)=V^B$, and will collect a reward of at most $V^{\pi}_1(s_1)\leq A\cdot A^{-\ell}=A^{-\ell+1}$. However, by going to the tree, any policy will earn a reward of $A^{-\ell/2}$; therefore, the optimal value is $V^*_1(s_1)=A^{-\ell/2}$ and $\frac{V^\pi_1(s_1)}{V^*_1(s_1)}\leq A^{1-\ell/2}$. 
        \item  If $V^D>V^B$: Set $R^{Tr}=R^L=0$, and put the state $s^D$ in the end of the line and $s^B$ in the leaves of the tree. The agent $\pi$ will thus navigate to the line and collect a reward of $A^{-\ell/2}$. However, upon reaching the root of the tree, all the rewards in all of the leaves and their actions are revealed. In particular, it is optimal to navigate to any leaf that has a realized unit reward, and the probability it happens is $1-(1-A^{-\ell})^{A^\ell}\ge 1-1/e$. In particular, this value is higher than $A^{-\ell/2}$ for $A,\ell\ge2$, which makes this choice optimal. 
        Therefore, it holds that $\frac{V^\pi_1(s_1)}{V^*_1(s_1)}\leq 2A^{-\ell/2}\leq A^{1-\ell/2}$. 
        \item If $V^D=V^B$: we now set $R^{Tr}=0$ and $R^L=0.25A^{-\ell/2}$, and again put the state $s^D$ in the end of the line and $s^B$ in the leaves of the tree.  Since the values are equal, the reward at the beginning will make the MPC agent choose the line, leading to a reward of $1.25A^{-\ell/2}$. Yet, traversing the tree will still lead to a reward of at least $1-1/e$, and so it is still optimal (for $A,\ell\ge2$) and $\frac{V^\pi_1(s_1)}{V^*_1(s_1)}\leq 2A^{-\ell/2}\leq A^{1-\ell/2}$. 
    \end{itemize}
    Combined, we get that regardless of the values $V^B,V^D$, there exists an environment for which $\frac{V^\pi_1(s_1)}{V^*_1(s_1)}\leq A^{1-\ell/2}$. Moreover, in all three cases, the optimal policy decides to traverse the tree and not the line: this behavior can be achieved using an MPC agent by putting a unit value at the leaves of the tree and a zero value at the end of the line: there exists an optimal MPC agent that collects the maximal value.
\end{proof}
\clearpage

\subsection{Numerical Illustrations}
\label{appendix:tables}
We end this appendix with two tables that include the numerical values for the examples in \Cref{claim: failure fixed batching} and \Cref{claim: failure MPC informal}, based on exact closed-form formulas for the analyzed environments.

\begin{itemize}
    \item Fixed batching counterexample (\Cref{claim: failure fixed batching}): For the tree construction,
$S=2+(A^\ell-1)/(A-1)$ (including the absorbing state), and one best fixed
batch length is $B^\star=\lfloor\ell/2\rfloor+1$. For this choice of $B^\star$, the exact fixed-batching and optimal values are given by the formulas
\begin{align*}
    V_{\mathrm{FB}}=1-(1-A^{-\ell})^{A^{\lceil\ell/2\rceil}} 
    \qquad \text{and}
    \qquad
    V_{\mathrm{opt}}=1-(1-A^{-\ell})^{A^\ell}.
\end{align*}
    The displayed values are rounded. The values for $A\in\brc*{2,5}$ and $\ell\in\brc*{2,5,10}$ are presented in \Cref{tab:fixed-batching-values}.
    \item MPC counterexample (\Cref{claim: mpc formal}): Taking into account both the tree and line (and a terminal state), it holds that $S=2+(A^\ell-1)/(A-1)+\ell$. For the illustration, we use the expected reward as the MPC terminal-value rule $V_h(s)=\max_a\mathbb{E}[R_h(s,a)]$: this is a natural choice for such states. 

    Since $V^B=A^{-\ell}<V^D=A^{-\ell/2}$, the tree has Bernoulli leaves, the line has
    deterministic rewards, and both first-step rewards are zero. MPC estimates
    the tree's value at $s_1$ as $\widehat V_{\mathrm{Tr}}=V^B=A^{-\ell}$ and takes
    the line whose value is $\widehat V_L=A^{-\ell/2}$. The actual values of the paths are exactly
    \begin{align*}
        V_{\mathrm{Tr}}=V_{\mathrm{opt}}=1-(1-A^{-\ell})^{A^\ell}
        \qquad \text{and}
        \qquad
        V_{\mathrm{L}}=V_{\mathrm{MPC}}=A^{-\ell/2};
    \end{align*}
    The displayed values are rounded. The values for $A\in\brc*{2,5}$ and $\ell\in\brc*{2,5,10}$ are presented in \Cref{tab:mpc-values}.
\end{itemize}

\begin{table}[h]
\caption{Exact-value illustration of \Cref{claim: failure fixed batching}.}
\label{tab:fixed-batching-values}
\centering
\setlength{\tabcolsep}{12pt}
\begin{tabular}{rrrrrrr}
\toprule
$A$ & $\ell$ & $S$ & $B^\star$ & $V_{\mathrm{FB}}$ & $V_{\mathrm{opt}}$ & $V_{\mathrm{FB}}/V_{\mathrm{opt}}$ \\
\midrule
2 & 2  & 5         & 2 & $0.4375$                 & $0.6836$ & $0.64$ \\
2 & 5  & 33        & 3 & $0.2243$                 & $0.6379$ & $0.3516$ \\
2 & 10 & 1{,}025   & 6 & $0.03078$                & $0.6323$ & $0.04868$ \\
5 & 2  & 8         & 2 & $0.1846$                 & $0.6396$ & $0.2887$ \\
5 & 5  & 783       & 3 & $0.03922$                & $0.6322$ & $0.06203$ \\
5 & 10 & 2{,}441{,}408 & 6 & $3.199\cdot 10^{-4}$ & $0.6321$ & $5.062\times 10^{-4}$ \\
\bottomrule
\end{tabular}
\end{table}

\begin{table}[h]
\caption{Exact-value illustration of \Cref{claim: mpc formal}.}
\label{tab:mpc-values}
\centering
\setlength{\tabcolsep}{8pt}
\begin{tabular}{rrrrrrr}
\toprule
$A$ & $\ell$ & $S$ & $\widehat V_{\mathrm{Tr}}$ & $\widehat V_{\mathrm{L}} = V_{\mathrm{L}}=V_{\mathrm{MPC}}$ & $V_{\mathrm{Tr}}=V_{\mathrm{opt}}$ & $V_{\mathrm{MPC}}/V_{\mathrm{opt}}$ \\
\midrule
2 & 2  & 7         & $0.25$                   & $0.5$      & $0.6836$               & $0.7314$ \\
2 & 5  & 38        & $0.03125$                & $0.1768$   & $0.6379$               & $0.2771$ \\
2 & 10 & 1{,}035   & $9.766\times10^{-4}$     & $0.03125$  & $0.6323$               & $0.04942$ \\
5 & 2  & 10        & $0.04$                   & $0.2$      & $0.6396$               & $0.3127$ \\
5 & 5  & 788       & $0.00032$                & $0.01789$  & $0.6322$               & $0.0283$ \\
5 & 10 & 2{,}441{,}418 & $1.024\times10^{-7}$ & $0.00032$  & $0.6321$               & $5.062\times10^{-4}$ \\
\bottomrule
\end{tabular}
\end{table}


\clearpage

\section{Table of Notations}

\begin{table}[H]
\centering
\small
\label{tab:notation}
\begin{tabularx}{1.1\textwidth}{@{}lX@{}}
\toprule
\textbf{Notation} & \textbf{Meaning} \\
\midrule
$\Mcal=(\Scal,\Acal,\Rcal,P,H)$ & Episodic tabular MDP. \\

$\Scal,\Acal$ & State space and action space. \\

$S,A$ & Cardinalities of the state and action spaces, respectively. \\

$\Rcal_h$ & Reward distribution at step $h$. The reward distribution of a state-action is  $\Rcal_h(s,a)$\\

$P_h$ & Transition kernel at step $h$. The probability  of moving to state $s'$ from $(s,a)$ is $P_h(s' \vert s,a)$. \\

$H$ & Episode horizon. \\

$R_h$ & Realized reward at step $h$, with $R_h \sim \Rcal_h(s_h,a_h)$. \\

$s_h,a_h$ & State and action at step $h$. \\

$\pi$ & Global policy used for the full interaction. \\

$\Pi_{det}$ & Set of deterministic Markov policies. \\

\midrule

$\ell$ & Maximum lookahead horizon. \\

$\ell_h=\min\{\ell, H-h+1\}$ & Effective lookahead horizon at step $h$ near the end of the episode. \\

$\phi$ & Local deterministic policy executed inside a lookahead batch. \\

$\tau^\phi_{h:t}(s)$ \eqref{eq: trajectory}& Realized trajectory from step $h$ to step $t$ starting at $s_h=s$ under local policy $\phi$. \\

$I_{h,\ell}(s)$ & Realized $\ell$-step lookahead information at step $h$ and state $s$; consists of all realized trajectories for all deterministic policies. \\

$\Ical_{h,\ell}(s)$ & Distribution of lookahead information sampled at step $h$ and state $s$. \\

$\Rla_{t\vert h}(s,\phi,I)$ & Realized reward at time $t$ induced by starting from $(h,s)$, following $\phi$, under lookahead realization $I$. \\

$\sla_{t\vert h}(s,\phi,I)$ & Realized state at time $t$ induced by starting from $(h,s)$, following $\phi$, under lookahead realization $I$. \\

\midrule

$B_h^\pi(s)$ & Batch length chosen by an ABP $\pi$ at step $h$ in state $s$. \\

$V^\pi_h(s)$ & The value of an ABP $\pi$ upon starting a new batch at step $h$ and state $s$. \\

$\ts_i^\pi$ & The (random) step in which a policy $\pi$ started its $i^{th}$ batch. \\

$E^n_h$ & The event in which some batch started at step $h$. \\

$\Qla_h^*(s,B;I,V)$ \eqref{eq: Q-val def} & Optimal cumulative reward over a batch of length $B$ starting from $(h,s)$ under lookahead $I$, plus terminal value $V$ at the end of the batch. \\

$B_h^*(s)$ \eqref{eq: DP range} & Optimal batch length at step $h$ in state $s$. \\

$V_h^*(s)$ \eqref{eq: DP val} & Optimal value among ABPs from a new batch at step $h$ and state $s$. \\

$Q_h^*(s,B)$ & Expected value of starting a fresh batch of length $B$ at $(h,s)$ and continuing optimally afterward:
$\;Q_h^*(s,B)=\E_{I\sim\Ical_{h,B}(s)}[\Qla_h^*(s,B;I,V^*_{h+B})]$. \\

$\phi^{*,h}$ \eqref{eq: DP policy} & Optimal local policy executed inside the batch started at step $h$. \\

$\Jla_{h,B}^I(s,s')$ \eqref{eq: efficient reward rep}& Maximum cumulative reward over $B$ steps, under lookahead $I$, among local policies that start at step $h$ in state $s$, and end in state $s'$ at step $h+B$. \\

$\Sla_{h,B}^I(s)$ \eqref{eq: reachable states}& Set of states reachable after $B$ steps from $(h,s)$ under lookahead realization $I$. \\

$\bR,\bs'$ & Lookahead information in a table format (as described in \Cref{appendix: data generation}). \\

$\MB$ & Augmented MDP used in the analysis/planning reduction for ABPs (see \Cref{appendix: augmented MDP}). \\

$V^\pi_{t}(s,\bR,\bs',h, B\vert\MB)$ & Value in the augmented MDP, with $V^\pi_{t}(s,h\vert\MB)=V^\pi_{t}(s,0^{\otimes SA\ell },\emptyset^{\otimes SA\ell},h, 0\vert\MB)$ 
        . \\
\midrule

$\Regret(K)$ & Regret after $K$ episodes:
$\;\sum_{k=1}^K \bigl(V_1^*(s_1^k)-V_1^{\pi^k}(s_1^k)\bigr)$. \\

$K$ & Number of episodes. \\

$\pi^k$ & Policy executed in episode $k$. \\

$s_1^k$ & Initial state of episode $k$. \\

$b_h^k(s,B)$ \eqref{eq: bonus}& Exploration bonus used by the learning algorithm. \\

$n_h^k(s,B)$ & Number of observations collected before episode $k$ for step $h$, state $s$, and batch length $B$. \\

$\widehat{\VAR}_h^k(s,B;V)$ \eqref{eq: empirical variance}& Empirical variance Q-value estimate assuming terminal value $V$. \\

$\bar{V}_{h}^k$ & Optimistic value estimate at episode $k$. \\

$L_\delta^k$ & Logarithmic confidence term:
$\;L_\delta^k=\ln\!\left(\frac{18SH\ell k^3(k+1)}{\delta}\right)$. \\

$\delta$ & Confidence parameter. \\

$\VAR_h^*(s,B)$ \eqref{eq: var opt value}& Variance of the optimal batch return when starting a fresh batch of length $B$ at $(h,s)$. \\

$\Vla_h^\pi(s,I)$ \eqref{eq: Q-value ABP}& Return of ABP $\pi$ when starting a fresh batch at $(h,s)$ under lookahead realization $I$. \\

$\VAR_h^\pi(s)$ \eqref {eq: Q-value variance ABP}& Variance of the random batch return $\Vla_h^\pi(s,I)$. \\

\midrule
$E^{opt}(k), E^{\widehat{\VAR}}(k), E^{Err}(k)$ & Good events defined in \Cref{appendix: concentration event}. \\

$E^{\mathrm{diff}},  E^{\VAR}$ & Good events defined in \Cref{appendix: second good event}. \\ 

\bottomrule
\end{tabularx}
\end{table}

\newpage
\section*{NeurIPS Paper Checklist}

\begin{enumerate}

\item {\bf Claims}
    \item[] Question: Do the main claims made in the abstract and introduction accurately reflect the paper's contributions and scope?
    \item[] Answer: \answerYes{} 
    \item[] Justification: In the paper, we describe the lookahead model in detail, show failure modes of existing common approaches and present in detail the policy space we study -- ABPs -- including planning and learning results.

\item {\bf Limitations}
    \item[] Question: Does the paper discuss the limitations of the work performed by the authors?
    \item[] Answer: \answerYes{} 
    \item[] Justification: All assumptions are clearly stated in the paper and future directions to mitigate them are discussed in the future work section.

\item {\bf Theory assumptions and proofs}
    \item[] Question: For each theoretical result, does the paper provide the full set of assumptions and a complete (and correct) proof?
    \item[] Answer: \answerYes{} 
    \item[] Justification: All assumptions are provided in the paper, except for Claim 2.2, for which an informal version of the assumption is given in the main paper and the formal assumptions are stated in the appendix. All proofs can be found in the appendix.

    \item {\bf Experimental result reproducibility}
    \item[] Question: Does the paper fully disclose all the information needed to reproduce the main experimental results of the paper to the extent that it affects the main claims and/or conclusions of the paper (regardless of whether the code and data are provided or not)?
    \item[] Answer: \answerNA{} 
    \item[] Justification: The paper is theoretical in nature and does not include experiments.

\item {\bf Open access to data and code}
    \item[] Question: Does the paper provide open access to the data and code, with sufficient instructions to faithfully reproduce the main experimental results, as described in supplemental material?
    \item[] Answer: \answerNA{} 
    \item[] Justification: The paper does not include experiments.

\item {\bf Experimental setting/details}
    \item[] Question: Does the paper specify all the training and test details (e.g., data splits, hyperparameters, how they were chosen, type of optimizer) necessary to understand the results?
    \item[] Answer: \answerNA{} 
    \item[] Justification: The paper does not include experiments.

\item {\bf Experiment statistical significance}
    \item[] Question: Does the paper report error bars suitably and correctly defined or other appropriate information about the statistical significance of the experiments?
    \item[] Answer: \answerNA{} 
    \item[] Justification: The paper does not include experiments.

\item {\bf Experiments compute resources}
    \item[] Question: For each experiment, does the paper provide sufficient information on the computer resources (type of compute workers, memory, time of execution) needed to reproduce the experiments?
    \item[] Answer: \answerNA{} 
    \item[] Justification: The paper does not include experiments.
    
\item {\bf Code of ethics}
    \item[] Question: Does the research conducted in the paper conform, in every respect, with the NeurIPS Code of Ethics \url{https://neurips.cc/public/EthicsGuidelines}?
    \item[] Answer: \answerYes{} 
    \item[] Justification: As a theoretical work, the research conducted in the paper conform with the NeurIPS Code of Ethics.

\item {\bf Broader impacts}
    \item[] Question: Does the paper discuss both potential positive societal impacts and negative societal impacts of the work performed?
    \item[] Answer: \answerNA{} 
    \item[] Justification: This paper theoretically studies the problem of planning and learning with multi-step lookahead information. While the applications for this setting vary, and some applications might have some societal consequences, our paper is theoretical in nature, and we do not feel that it has any direct societal impact that should be specifically highlighted.
    
\item {\bf Safeguards}
    \item[] Question: Does the paper describe safeguards that have been put in place for responsible release of data or models that have a high risk for misuse (e.g., pre-trained language models, image generators, or scraped datasets)?
    \item[] Answer: \answerNA{} 
    \item[] Justification: As a theoretical work, the paper poses no such risk.

\item {\bf Licenses for existing assets}
    \item[] Question: Are the creators or original owners of assets (e.g., code, data, models), used in the paper, properly credited and are the license and terms of use explicitly mentioned and properly respected?
    \item[] Answer: \answerNA{} 
    \item[] Justification: The paper does not use existing assets.

\item {\bf New assets}
    \item[] Question: Are new assets introduced in the paper well documented and is the documentation provided alongside the assets?
    \item[] Answer: \answerNA{} 
    \item[] Justification: The paper does not release new assets.

\item {\bf Crowdsourcing and research with human subjects}
    \item[] Question: For crowdsourcing experiments and research with human subjects, does the paper include the full text of instructions given to participants and screenshots, if applicable, as well as details about compensation (if any)? 
    \item[] Answer: \answerNA{} 
    \item[] Justification: The paper does not involve crowdsourcing nor research with human subjects.

\item {\bf Institutional review board (IRB) approvals or equivalent for research with human subjects}
    \item[] Question: Does the paper describe potential risks incurred by study participants, whether such risks were disclosed to the subjects, and whether Institutional Review Board (IRB) approvals (or an equivalent approval/review based on the requirements of your country or institution) were obtained?
    \item[] Answer: \answerNA{} 
    \item[] Justification: The paper does not involve crowdsourcing nor research with human subjects.

\item {\bf Declaration of LLM usage}
    \item[] Question: Does the paper describe the usage of LLMs if it is an important, original, or non-standard component of the core methods in this research? Note that if the LLM is used only for writing, editing, or formatting purposes and does \emph{not} impact the core methodology, scientific rigor, or originality of the research, declaration is not required.
    \item[] Answer: \answerNA{} 
    \item[] Justification: LLMs were only used for formatting, proofreading, and diagram generation.

\end{enumerate}

\end{document}